\documentclass[twoside]{article}
\usepackage{tikz}
\usetikzlibrary{arrows.meta}

\usepackage[preprint]{aistats2027}
\usepackage[round]{natbib}

\usepackage{mathrsfs}

\usepackage{amsmath,amsfonts,bm}

\def\eqref#1{equation~\ref{#1}}
\def\1{\bm{1}}

\DeclareMathAlphabet{\mathsfit}{\encodingdefault}{\sfdefault}{m}{sl}
\SetMathAlphabet{\mathsfit}{bold}{\encodingdefault}{\sfdefault}{bx}{n}

\usepackage{hyperref}
\usepackage{url}

\usepackage{graphicx}
\usepackage{subcaption}
\usepackage{booktabs}
\usepackage{multirow}
\usepackage{algorithmicx}
\usepackage{algpseudocode}
\usepackage[ruled,vlined,linesnumbered]{algorithm2e}
\usepackage{amsthm}
\usepackage{todonotes}
\usepackage{comment}
\usepackage{placeins}
\usepackage{float}

\usepackage{xcolor}
\usepackage[normalem]{ulem}

\newcommand{\del}[1]{}

\usepackage{colortbl}

\definecolor{cvxs1}{gray}{0.75}  % 0.571
\definecolor{cvxs2}{gray}{0.35}  % 0.717
\definecolor{cvxs3}{gray}{0.10}  % 0.730
\definecolor{cvxs4}{gray}{0.20}  % 0.725
\definecolor{cvxs5}{gray}{0.15}  % 0.728

\definecolor{cvxt1}{gray}{0.75}  % 0.907
\definecolor{cvxt2}{gray}{0.20}  % 0.951
\definecolor{cvxt3}{gray}{0.10}  % 0.954
\definecolor{cvxt4}{gray}{0.10}
\definecolor{cvxt5}{gray}{0.10}

\usepackage{amsmath,amssymb,amsfonts,mathtools}

\newcommand{\bK}{\mathbf{K}}

\newcommand{\ind}[1]{\mathbb{I}\!\left\{#1\right\}}

\theoremstyle{plain}
\newtheorem{theorem}{Theorem}[section]
\newtheorem{lemma}[theorem]{Lemma}
\newtheorem{proposition}[theorem]{Proposition}
\newtheorem{corollary}[theorem]{Corollary}
\newtheorem*{theorem*}{Theorem}

\theoremstyle{definition}

\theoremstyle{remark}
\newtheorem{remark}[theorem]{Remark}

\begin{document}

% =========================================================================
% Title and authors. AISTATS is two-column; the title block sits inside
% \twocolumn[...]. Keep the author block anonymous for submission.
% =========================================================================
\twocolumn[
% Match the title submitted with the abstract on OpenReview.
\aistatstitle{Convex Optimization of Spiking Neural Networks with Global and Approximation Guarantees}

% \aistatsauthor{ Anonymous Author(s) }
% \aistatsaddress{ Anonymous Institution(s) }
% Camera-ready:
% Authors in two rows; affiliations grouped with superscripts.
\runningauthor{Himanshu Udupi, Calvin Ang, Sungyoon Kim, Xiaocong Yang, ChengXiang Zhai}
% \And/\AND stretch names across the full page width, so lay the two rows out
% as one centered tabular with fixed gaps instead.
\aistatsauthor{ Himanshu Udupi$^{1}$ \hspace{2em} Calvin Ang$^{2}$ \hspace{2em} Sungyoon Kim$^{2}$ \\[0.6ex]
                \bfseries Xiaocong Yang$^{1}$ \hspace{2em} ChengXiang Zhai$^{1}$ }
% One centered affiliation line, so it does not read as belonging to the second-row authors.
\aistatsaddress{ \rule{0pt}{3.5ex}$^{1}$University of Illinois Urbana-Champaign \qquad $^{2}$Stanford University }
% Original per-author layout:
% \aistatsauthor{ Himanshu Udupi \And Calvin Ang \And Sungyoon Kim \And  Xiaocong Yang \And ChengXiang Zhai }
% \aistatsaddress{ University of Illinois Urbana-Champaign \And Stanford University \And Stanford University \And
%                  University of Illinois Urbana-Champaign \And
%                  University of Illinois Urbana-Champaign }

]

\begin{abstract}
Spiking neural networks (SNNs) have been proposed as energy-efficient alternatives to deep neural networks. Their training is non-convex and involves a non-differentiable spike function, so existing methods offer no guarantee of reaching a global optimum. We address this issue by formulating SNN optimization as an equivalent convex optimization problem. Specifically, we show that the weight-decay-regularized training problem for a class of threshold RNNs, including leaky integrate-and-fire SNNs, has zero duality gap and the same optimal value as a convex program. Since the resulting convex programs can have exponentially many variables in the amount of data, we introduce a practical randomized surrogate and establish a high-probability guarantee on its approximation quality. Empirically, the best convex variant outperforms surrogate-gradient training on addition and first-last-XOR, matches its performance on sequential MNIST, and trails a reservoir baseline on base-3 and base-5 addition.
\end{abstract}

% Main text: 8 pages. References, appendix and checklist do not count.
\section{Introduction}
% Placeholder macro for AISTATS drafting. Remove all \tbc{} before submission.
\providecommand{\tbc}[1]{\textcolor{orange}{[TBC: #1]}}
\label{sec:introduction}

Spiking Neural Networks (SNNs)~\citep{maass1997networks, gerstner2002spiking} are neural networks whose neurons communicate through discrete spikes instead of continuous activations, following how biological neurons process information. A spiking neuron integrates its inputs into a membrane potential that rises and decays over time, and fires a spike when the potential crosses a threshold. The most common model is the Leaky-Integrate-and-Fire (LIF) neuron~\citep{gerstner2002spiking}, governed by a differential equation, whose output over time is a binary sequence.

% An ANN neuron outputs a static scalar, while a spiking neuron outputs a sequence in time, so information is carried in the timing of spikes. This suits time-series and memory tasks~\citep{maass1997networks, bellec2018long}. Spikes are binary, so the main operations are accumulations rather than dense multiplications, which is often argued to save energy~\citep{roy2019towards}. Whether it saves energy in practice is debated~\citep{dampfhoffer2022, davies2021advancing, nature_neuromorphic_2024}. These properties make SNNs a candidate for energy-efficient and neuromorphic computing~\citep{nextgen1, nextgen2, nextgen3, nextgen4}.
A spiking neuron outputs a binary sequence in time, which suits time-series and memory tasks~\citep{maass1997networks, bellec2018long}, and its operations are accumulations rather than dense multiplications, which is often argued to save energy~\citep{roy2019towards}. Whether it saves energy in practice is debated~\citep{dampfhoffer2022, davies2021advancing, nature_neuromorphic_2024} and depends on the hardware. These properties make SNNs a candidate for energy-efficient and neuromorphic computing~\citep{nextgen1, nextgen2, nextgen3, nextgen4}.

The spike function is non-differentiable, so gradient backpropagation cannot be applied directly. Three families of methods address this. \emph{Surrogate gradients}~\citep{surrogate} replace the spike function with a smooth surrogate in the backward pass; the forward-backward mismatch gives no guarantee on the true loss and accumulates across layers, which makes deeper SNNs harder to train~\citep{surrogate_module}. \emph{ANN-to-SNN conversion} trains an ANN and maps it to an SNN, inheriting the ANN parameters but no spike-based training guarantee. \emph{Reservoir methods} fix the recurrent dynamics and train only a linear readout, as in the liquid-state machine~\citep{maass_real-time_2002, lukosevicius_reservoir_2009}; the readout is convex but the dynamics are fixed by a criterion independent of the training objective. Exact-gradient methods~\citep{mostafa_supervised_2017, wunderlich_event_2021, goeltz_fast_2021} avoid the mismatch in continuous-time models by differentiating spike times. The gradient is exact while the set of spikes is fixed, and it does not account for spikes appearing or disappearing. None of these methods characterizes the optimal value of the training objective, and therefore none can quantify how far a trained SNN is from it. When a trained network performs poorly, the failure may be due to the optimizer or to the architecture, and the two cannot be distinguished without knowledge of the optimal value. Characterizing the optimal value, and verifying that a solution attains it, separates the two.

We address this issue by formulating SNN training as an equivalent convex optimization problem. Specifically, we show that weight-decay-regularized training of parallel SNNs has zero duality gap and the same optimal value as a convex program over the spike patterns realizable at the last hidden layer (Theorems~\ref{thm:finite_convex_L_T} and~\ref{thm:finite_convex_lif_snn}). The regularizer is standard weight decay on all weights, augmented with a trainable neuron amplitude that does not enlarge the class of representable functions. Two properties of SNNs make this formulation non-trivial. Their weights are shared across time, so the hidden weights must be rescaled jointly over all timesteps and the spike patterns must be consistent along trajectories. Their neurons carry a real-valued membrane with a leak and reset fixed by the neuron model~\citep{gerstner2002spiking}, so the reduction must separate the scales that can be removed from those the model fixes. Our analysis resolves both and uses the convex duality of parallel threshold networks~\citep{ergen_globally_2023} as a tool. 

This formulation decomposes training into two stages, the selection of spike patterns and the solution of a convex program over them. The convex program can be solved exactly, and its duality gap verifies optimality over the given dictionary. The selection of patterns is the difficult stage, since for inputs in general position a network with two hidden layers of width at least $n$ realizes all $2^n$ binary patterns (Appendix~\ref{app:complete}). We therefore sample the patterns, and for shallow networks we establish a high-probability bound on the resulting loss of optimality (Theorem~\ref{thm:approx_template}). Our experiments then examine whether the best training fit over the sampled patterns generalizes on symbolic, temporal and sensory tasks.

Our main contributions are as follows.
\begin{itemize}
    \item \textbf{Convex characterization.} For parallel recurrent threshold networks and parallel LIF-SNNs under weight decay, the optimal training value equals the attained minimum of a reduced problem and the value of a finite convex program over the realizable spike patterns (Lemma~\ref{lem:attain}, Theorems~\ref{thm:finite_convex_L_T} and~\ref{thm:finite_convex_lif_snn}). The original problem is not attained in general. To our knowledge this is the first convex characterization of the optimal value of SNN training.

    \item \textbf{Reconstruction.} A solution of the convex program gives a minimizer of the reduced problem and, by scaling its inner weights, an $\varepsilon$-optimal network for every $\varepsilon>0$. Since the realizable patterns are complete for inputs in general position (Appendix~\ref{app:complete}), we give a two-step algorithm that samples patterns and solves the convex program over them.

    \item \textbf{Sampling guarantees.} For Gaussian and criticality-tuned reservoir witnesses, we bound the sampled optimum against the complete one with high probability at width $\mathrm{poly}(n,1/\lambda_{\min}(\bK),\log 1/\delta)$, for shallow networks, namely two-layer threshold and LIF networks and parallel LIF subnetworks (Theorem~\ref{thm:approx_template}). To our knowledge this is the first such bound for spiking networks with fixed hidden features, and it shows that tuning a reservoir affects the bound only through its kernel (Theorem~\ref{thm:ap_ratio}).

    \item \textbf{Experiments.} On small tasks the convex optimum is sparse and the reconstruction reproduces it. On symbolic, temporal and sensory tasks the convex program is solved exactly and fits the training data in most settings, while test accuracy depends on the source of the patterns.
\end{itemize}

\section{Related Works}
\subsection{Training Algorithms for Spiking Neural Networks}
\label{sec:related_work}

We organize SNN training methods by which components they fix and which they optimize, since our formulation decomposes training along the same lines. One way to handle the non-differentiable spike function is through surrogate gradients, where gradients are computed via continuous approximations such as arctangent or sigmoid functions to facilitate backpropagation \citep{surrogate, Boht2000SpikePropBF, Wu_2018_STBP}. However, to the best of our knowledge, there is no analytical or theoretical study that establishes comparable optimality conditions for model weights after surrogate-gradient SNN training. Exact-gradient methods~\citep{mostafa_supervised_2017, wunderlich_event_2021, goeltz_fast_2021} instead differentiate spike times in continuous-time models, which avoids the surrogate mismatch. The gradient is exact while the set of spikes is fixed, so it does not capture changes in which spikes occur, and these methods do not apply to the discrete-time model studied here, whose spikes are binary at fixed time steps. Both classes of gradient methods optimize all weights jointly and do not characterize the optimal value. Within our formulation, a network trained with surrogate gradients can instead serve as a source of spike patterns for the convex program (\textsc{SG-CVX}).

ANN-SNN conversion trains an ANN and maps it to an SNN \citep{conversion1, conversion2}. It does not train the temporal dynamics directly, typically requires many timesteps \citep{bu2023optimalannsnnconversionhighaccuracy}, and degrades under low-latency constraints \citep{DIET}, so we do not include it among our baselines.

A third category fixes the recurrent dynamics and trains only a linear readout. This is the principle of reservoir computing and the liquid-state machine~\citep{maass_real-time_2002, jaeger_echo_2001, lukosevicius_reservoir_2009}, and of random-feature methods~\citep{rahimi_random_2007}: a fixed, random or criticality-tuned recurrent layer projects the input into a high-dimensional state, and a convex readout is trained over it. The dynamics are fixed by a criterion independent of the training objective, and the readout carries no guarantee relative to that objective. Our practical algorithm shares this fix-then-solve structure, and we include a criticality-tuned spiking reservoir among our baselines. The distinction is that our decomposition follows from the convex characterization of SNN training, so the readout is optimal over the patterns of the reservoir and its distance to the global optimum can be bounded.

\subsection{Convex Reconstruction of Neural Networks}
\label{reconstruction}

A line of work establishes globally optimal training of 
overparameterized parallel neural networks via convex 
reformulation, including feedforward networks~\citep{pilanci_neural_2020, ergen_convex_2021, 
ergen_convex_2025, wang_convex_2021}, CNNs~\citep{ergen_implicit_2021}, 
and Transformers~\citep{ergen_convexifying_2022}. We use the duality argument of \citet{ergen_globally_2023}, which establishes global optimality for feedforward parallel threshold networks under weight-decay regularization, as one ingredient. The parts of our analysis specific to SNNs are the treatment of weights tied across time and of the fixed leak and reset of the neuron model. These reformulations are exact over the complete set of activation patterns, and in practice the patterns are sampled. \citet{kim_convex_2024} show that a relaxation of the convex program, solved over randomly sampled patterns, approximates the global optimum of two-layer ReLU training to within a logarithmic factor. Sampling patterns with random weights is a random-feature approximation \citep{rahimi_random_2007}, and the width needed to approximate a kernel is governed by its leverage scores \citep{bach_equivalence_2017,avron_random_2017}. \citet{drago_rank_2026} give widths at which a random sign feature matrix has full row rank. \citet{becchetti_dimensionfree_2026} make this width dimension-free for homogeneous activations, including the threshold, and lower-bound the smallest eigenvalue of the population kernel. Their work was developed independently of ours, and we became aware of it after our analysis was complete. Our sampling bound is of the same type and applies to threshold and LIF witnesses, including tuned reservoirs.

\section{Convex Analysis of Spiking Neural Network Training}
Weight-decay-regularized training of a parallel SNN is non-convex in all of its weights. We reformulate it in three steps. Scale invariance of the threshold removes the hidden weights from the penalty (\S\ref{sec:notation_prelim}). Convex duality then yields a convex program over the realizable spike patterns, for recurrent threshold networks and for LIF-SNNs (\S\ref{sec:rnn_section}, \S\ref{sec:snn_section}). Since the set of patterns is exponentially large, we sample it and bound the resulting loss of optimality (\S\ref{sec:practical_impl}, \S\ref{sec:approx}).

\subsection{Scale Invariance and the Reduction to the Readout Norm}
\label{sec:notation_prelim}
For $m\in\mathbb{N}$, let $[m]=\{1,\ldots,m\}$. We write $\ind{\cdot}$ for the indicator function, $\|\cdot\|_p$ for the $\ell_p$ norm, and $\|\cdot\|_{p^*}$ for the dual norm ($1/p+1/p^*=1$). The training set has $n$ samples with $d$ features. Feedforward inputs are $X\in\mathbb{R}^{n\times d}$, recurrent and spiking inputs are $X^{1:T}=(X^1,\ldots,X^T)$ with $X^t\in\mathbb{R}^{n\times d}$, labels are $Y$, and every loss $\mathcal L(\cdot,Y)$ is convex in its first argument.

The scaled threshold activation is $\sigma_s(x)=s\ind{x\ge 0}$, with $s$ a trainable amplitude. The amplitude is balanced against the outgoing weights in the reduction (Theorem~\ref{thm:reduction_nonparallel}, Appendix~\ref{app:reduction_single}), and we write $\sigma(x)=\ind{x\ge 0}$ for the unit-amplitude gate. The key property used throughout is \textit{positive scaling invariance}: $\sigma(cx)=\sigma(x)$ for every $c>0$, so rescaling the pre-activation of any threshold neuron leaves its output unchanged.

We represent a network as a weighted Directed Acyclic Graph (DAG) $G(V,E,w)$, with input nodes $V_{\mathrm{in}}$, output nodes $V_{\mathrm{out}}$, and edge weights $w(e)$ for $e=(u,v)\in E$. Nodes are individual neurons. Hidden neurons carry amplitudes $s_v$. The neurons of the last hidden layer enter the readout with unit amplitude, and the scale of the readout is carried by the output amplitude $s_{V_{\mathrm{out}}}$, which is grouped with the incoming weights $W(\cdot,V_{\mathrm{out}})$ of the output node in $\Omega_{p,2}$ below. For a fully connected feedforward threshold network of depth $L$, 
$
f^{L,\Theta}(X)
=
\sigma\!\left(\cdots\sigma\!\left(\sigma(XW_1)W_2\right)\cdots W_{L-1}\right)W_L 
$, with parameters $\Theta=(W_1,\ldots,W_L)$, the 
group-norm regularizer of \citet{neyshabur2015path} with inner norm $p$ is 

%% Calvin - I think this is a bit confusing because in the earlier paragraph we say \sigma(x) is the unit-amplitude gate and we write f in terms of this and then reintroduce learnable amplitudes after the fact. It would be simpler, I think to just get rid of the s_v term in the regularization

$$
\Omega_{p,q}(\Theta)
=
\left(
\sum_{v\in V\setminus V_{\mathrm{in}}}
\Big(
|s_v|^p+\sum_{(u,v)\in E}|w(u,v)|^p
\Big)^{q/p}
\right)^{1/q},
$$
the $\ell_q$ aggregation across nodes of each node's incoming block (weight column and amplitude $s_v$, with no amplitude term for the neurons of the last hidden layer) measured in $\ell_p$. For the output node this block is $(s_{V_{\mathrm{out}}},W(\cdot,V_{\mathrm{out}}))$. We use $q=2$,
$$
\Omega_{p,2}(\Theta)^2
=
\sum_{v\in V\setminus V_{\mathrm{in}}}
\Big(
|s_v|^p+\sum_{(u,v)\in E}|w(u,v)|^p
\Big)^{2/p},
$$
an additive sum over nodes of squared incoming $\ell_p$ blocks; for $p=2$ this is ordinary squared weight decay. The amplitudes are the only departure from standard weight decay. They do not enlarge the class of representable functions, since a hidden amplitude can be absorbed into the weights of the next layer and the readout amplitude into the readout. The reduction of Lemma~\ref{lem:attain} uses only this additivity across nodes together with positive scaling invariance.

A $K$-parallel threshold network extends such a feedforward network to one consisting of subnetworks $\{G_k(V_k,E_k,w_k)\}_{k=1}^K$ that share input nodes but have disjoint hidden parameters. Its output and weight-decay regularizer are written as
$
f^{L,K,\Theta}(X)=\sum_{k=1}^K f^{L,k,\Theta_k}(X)
$
and
$
\Omega_{p,2}(\Theta)^2=\sum_{k=1}^K\Omega_{p,2}(\Theta_k)^2
$
respectively. We restrict attention to such parallel constructions, since non-parallel (single-branch) networks do not admit convex reformulations \citep{DBLP:journals/corr/abs-2110-06482}. The regularized training objective is
$$
P^{L,K}
=
\min_\Theta\ \mathcal L\big(f^{L,K,\Theta}(X),Y\big)
+
\tfrac{\beta_{\mathrm{reg}}}{2}\,\Omega_{p,2}(\Theta)^2,
\qquad
\beta_{\mathrm{reg}}>0.
$$

Since the threshold is invariant to positive scaling of its input, the hidden weights of a threshold network can be made arbitrarily small without changing the network function. This removes them from the penalty, but it also implies that the original problem has no minimizer in general. The following lemma formalizes both observations.

\begin{lemma}[Reduction and attainment]\label{lem:attain}
Let $P$ be the optimal value of the weight-decay regularized problem and let
$Q := \min_{\Theta:\,|s_{V_{\mathrm{out}}}|=1}\ \mathcal{L}(f^{L,\Theta}(X),Y)+\beta_{\mathrm{reg}}\|W(\cdot,V_{\mathrm{out}})\|_p$
be the reduced problem. Then: (i) $P=Q$, and the minimum defining $Q$ is attained. (ii) If $\Theta^\star$ minimizes $Q$ and $\Theta_\alpha$ scales every hidden incoming block and inner amplitude of $\Theta^\star$ by $\alpha\in(0,1]$, then $f^{L,\Theta_\alpha}=f^{L,\Theta^\star}$ and the objective at $\Theta_\alpha$ tends to $P$ as $\alpha\to0$, so for every $\varepsilon>0$ some $\Theta_\alpha$ is $\varepsilon$-optimal for the original problem. (iii) $P$ is attained if and only if $Q$ is attained by a network whose hidden nodes all have zero incoming magnitude and hence, under $\mathbb{I}\{0\ge0\}=1$, constant output $s_v$. In general it is not.
\end{lemma}
The proof is given in Appendix~\ref{app:attain_proof}.
The subnetworks of a $K$-parallel network share no hidden parameters and the penalty is additive over them, so the argument of Lemma~\ref{lem:attain} applies to each subnetwork. Each keeps only the norm of its readout block $W_{L,k}\in\mathbb{R}^{m_{L-1}\times d_{\mathrm{out}}}$, with $|s_k|=1$ on its readout amplitude (Lemma~\ref{thm:reduction_parallel}, Appendix~\ref{app:reduction_parallel}),
$$
\inf_\Theta\ \mathcal L + \tfrac{\beta_{\mathrm{reg}}}{2}\Omega_{p,2}(\Theta)^2
=
\min_{\Theta:\,|s_k|=1}\mathcal L + \beta_{\mathrm{reg}}\sum_{k=1}^K\|W_{L,k}\|_p,
$$
where $\|\cdot\|_p$ is the entrywise norm of the block, which for $K=1$ is $\|W(\cdot,V_{\mathrm{out}})\|_p$ of Lemma~\ref{lem:attain}. We write $P^{L,K}$ also for this reduced objective. It is still non-convex, because the readout multiplies features that depend on the hidden weights. Its optimal value, however, can be characterized through the convex duality of \citet{ergen_globally_2023}. Write $\phi_k(X;\theta_k)$ for the last-hidden-layer output of subnetwork $k$ with hidden parameters $\theta_k$, $\Lambda^{L,K}$ for the value of the dual, and $P_\infty^{L,K}$ for the value of the bidual, a convex program over measures on these feature maps. The dual and bidual have zero duality gap, and the bidual equals the primal for $K\ge K^*$ with $K^*\le n+1$ (Proposition~\ref{thm:feedforward_bi_dual}, Appendix~\ref{app:feedforward_bidual}). The reduction of Lemma~\ref{thm:reduction_parallel} is what makes this characterization possible. Once only the readout is penalized, the problem is a regularized linear model in the readout for fixed hidden features, and the non-convexity of training is confined to the features. The rest of this section establishes the same characterization for recurrent networks and LIF-SNNs, where weight tying and the neuron model must be accounted for, and makes the convex program finite.

\subsection{Recurrent Threshold Networks and the Finite Convex Program}
\label{sec:rnn_section}
A recurrent network unrolled over $T$ steps is a feedforward DAG whose weights are tied across time, so the reduction of Lemma~\ref{thm:reduction_parallel} and the duality result of Proposition~\ref{thm:feedforward_bi_dual} apply once the tying is accounted for. We consider $K$-parallel recurrent threshold networks of depth $L$. For subnetwork $k$, layer $l\in[L-1]$ and timestep $t\in[T]$,
$$
H_{k,l}^t
=
\sigma\!\left(
H_{k,l-1}^tP_{k,l,\mathrm{in}}
+
H_{k,l}^{t-1}P_{k,l,\mathrm{rec}}
\right),
\qquad
H_{k,0}^t=X^t .
$$
Unless stated otherwise, $H_{k,l}^0=0$. The final-time readout is
$$
f^{L,T,K,\Theta}(X^{1:T})
=
\sum_{k=1}^K
H_{k,L-1}^T P_{k,\mathrm{out}},
\label{eq:recurrent_f}
$$
where $P_{k,l,\mathrm{in}}\in\mathbb{R}^{m_{l-1}\times m_{l}}$,
$P_{k,l,\mathrm{rec}}\in\mathbb{R}^{m_{l}\times m_{l}}$, and
$P_{k,\mathrm{out}}\in\mathbb{R}^{m_{L-1}\times d_{\mathrm{out}}}$. All subnetworks have equal widths $m_l$, so we write $m_l$ for the width of layer $l$ in every subnetwork. For recurrent and LIF networks the readout block $W_{L,k}$ of the unrolled DAG is $P_{k,\mathrm{out}}$, and we use $P_{k,\mathrm{out}}$ from here on.

Tying changes two things. Each hidden neuron $i$ of layer $l$ has two tied incoming blocks, $P_{k,l,\mathrm{in}}[:,i]$ and $P_{k,l,\mathrm{rec}}[:,i]$. Scaling both by a common $\alpha>0$ scales the whole pre-activation by $\alpha$, which the threshold absorbs, so the joint magnitude
$$
a_{k,l,i}=\Big(\|P_{k,l,\mathrm{in}}[:,i]\|_p^p
+
\|P_{k,l,\mathrm{rec}}[:,i]\|_p^p\Big)^{1/p}
$$
plays the role of the incoming weight column and is driven to zero as in Lemma~\ref{thm:reduction_parallel} (Theorem~\ref{thm:par-rec-thresh-reg}, Appendix~\ref{app:recurrent_theory}). And a column at layer $l$ and timestep $t$ cannot be chosen independently of the earlier columns, since all are produced by the same shared weights, so the dictionary must be trajectory-consistent (Appendix~\ref{app:recurrent_arrangements}). With these two changes, the reduction and the duality result of Proposition~\ref{thm:feedforward_bi_dual} applied to the unrolled DAG give the following zero-duality-gap result, with the unrolled feature map $\phi_{T,k}(X^{1:T};\theta_k):=H^T_{k,L-1}$ and dual and bidual values $\Lambda^{L,T,K}$ and $P_\infty^{L,T,K}$.

\begin{theorem}[Bi-dual of parallel recurrent threshold networks]
\label{thm:recurrent_bi_dual}
Let $P^{L,T,K}$ denote the reduced weight-decay regularized recurrent training problem above, with convex loss $\mathcal L$ and regularization parameter $\beta_{\mathrm{reg}}>0$. Proposition~\ref{thm:feedforward_bi_dual} holds with $\phi_k(X;\theta_k)$ replaced by the unrolled feature map $\phi_{T,k}(X^{1:T};\theta_k)$. Since $\beta_{\mathrm{reg}}>0$, strong duality holds:
$$
\Lambda^{L,T,K}=P_{\infty}^{L,T,K}.
$$
Moreover, for scalar outputs, if $K\ge K^*$ for some $K^*\le n+1$, then
$$
P^{L,T,K}=\Lambda^{L,T,K}=P_{\infty}^{L,T,K}.
$$
For $d_{\mathrm{out}}$ output coordinates, the same Carath\'eodory argument gives
$K^*\le nd_{\mathrm{out}}+1$ after vectorizing the predictions.
\end{theorem}
The proof is given in Appendix~\ref{app:recurrent_bidual}.

The bidual of Theorem~\ref{thm:recurrent_bi_dual} ranges over an infinite family of feature maps. It becomes finite because on a fixed dataset each threshold unit can only output finitely many binary patterns. Following \citet{ergen_globally_2023}, we collect the distinct patterns at layer $L-1$ and timestep $T$ as the columns of the \textit{spike dictionary} $D^{L-1,T}\in\{0,1\}^{n\times M^{L-1,T}}$. Parameters that produce a column are a \textit{witness} of that column, and producing columns from witnesses is \textit{witness generation} (Appendix~\ref{app:recurrent_arrangements}).

Each column of the dictionary is the output of one subnetwork and its coefficient is that subnetwork's readout, so the penalty on the readout norms becomes an $\ell_1$ penalty on the coefficients, and Carath\'eodory's theorem gives an optimal solution with at most $n+1$ non-zero coefficients, one per subnetwork. The factor $\sqrt{m_{L-1}}$ arises because each atom is carried by the $m_{L-1}$ last-layer neurons of one subnetwork.

\begin{theorem}[Finite convex formulation for recurrent threshold networks]
\label{thm:finite_convex_L_T}
For weight decay with $p=2$, let
$D^{L-1,T}\in\{0,1\}^{n\times M^{L-1,T}}$
be the complete witnessed dictionary at the final hidden layer and final timestep. Then the reduced recurrent training problem admits the finite convex bi-dual objective
$$
\widetilde P^{L,T,K}
=
\min_{\widetilde w\in\mathbb{R}^{M^{L-1,T}}}
\mathcal L(D^{L-1,T}\widetilde w,Y)
+
\frac{\beta_{\mathrm{reg}}}{\sqrt{m_{L-1}}}\|\widetilde w\|_1 .
$$
Moreover, for scalar outputs,
$$
\widetilde P^{L,T,K}=P^{L,T,K}
$$
whenever $K\ge K^*$ for some $K^*\le n+1$. For $d_{\mathrm{out}}$ output coordinates, the corresponding sufficient bound is
$K^*\le nd_{\mathrm{out}}+1$.
\end{theorem}
The proof is given in Appendix~\ref{app:finite_convex_LT}.
This is the program we solve in practice. It depends on the network only through the dictionary. Training therefore decomposes into the selection of patterns, which is difficult, and the solution of a convex program over them, which can be carried out exactly and verified through the duality gap.

\subsection{Spiking Neural Networks as Structured Recurrent Threshold Networks}
\label{sec:snn_section}
An LIF neuron is a threshold neuron whose recurrent input is its own leaky membrane and reset. An LIF-SNN is therefore a recurrent threshold network with a constrained recurrence, and it remains to verify that the reduction of Lemma~\ref{thm:reduction_parallel} holds under this constraint. If it does, the same decomposition applies to LIF-SNN training. We use the discrete-time LIF model standard in machine-learning implementations~\citep{surrogate, Wu_2018_STBP}, not the underlying ODE, in the threshold-centered convention $S_l^t=\sigma(U_l^t)=\ind{U_l^t\ge 0}$, and drop the subnetwork index $k$. For layer $l$ and timestep $t$, the LIF recurrence is
$$
U_l^t
=
S_{l-1}^tP_{l,\mathrm{in}}
+
U_l^{t-1}B_l
-
S_l^{t-1}R_l,
\qquad
S_l^t=\sigma(U_l^t),
$$
where $B_l=\operatorname{Diag}(\beta_l)$ with $\beta_l\in[0,1)$ the leak coefficient,
$R_l=\operatorname{Diag}(U_{\mathrm{thr}}^l)$ the soft-reset matrix, and $U_l^t$ the membrane potential at layer $l$, timestep $t$. The recurrent input is thus $[U_l^{t-1}\ S_l^{t-1}]$ times the block matrix $[B_l;\,-R_l]$, a leak-reset form rather than an arbitrary dense matrix. The recurrent state includes the real-valued membrane $U_l^{t-1}$, and the leak and reset are fixed by the neuron model rather than free weights, so an LIF-SNN is not a direct instance of \S\ref{sec:rnn_section}. We type the edges of the unrolled graph, separating those that carry a removable positive scale from those fixed by the dynamics (Appendix~\ref{app:snn_dag}). Scaling a neuron's input weights and reset by the same factor scales every term of its membrane, as long as the leak acts only on its own membrane.

\begin{lemma}[LIF membrane rescaling]
\label{lem:lif_membrane_rescaling}
Let $A_l=\operatorname{Diag}(a_{l,1},\ldots,a_{l,m_l})$ with $a_{l,i}>0$. Define
$$
\bar U_l^t=U_l^tA_l,
\quad
\bar P_{l,\mathrm{in}}=P_{l,\mathrm{in}}A_l,
\quad
\bar R_l=R_lA_l,
\quad
\bar B_l=A_l^{-1}B_lA_l.
$$
Then the rescaled variables satisfy
$$
\bar U_l^t
=
S_{l-1}^t\bar P_{l,\mathrm{in}}
+
\bar U_l^{t-1}\bar B_l
-
S_l^{t-1}\bar R_l.
$$
Moreover, if $B_l$ is diagonal, then $\bar B_l=B_l$, and the spike outputs are unchanged:
$$
\sigma(\bar U_l^t)=\sigma(U_l^tA_l)=\sigma(U_l^t).
$$
\end{lemma}
The proof is given in Appendix~\ref{app:lif_membrane_rescaling}.
The trainable parameters $P_{l,\mathrm{in}}[:,i]$ and $U_{\mathrm{thr},i}^l$ feed only $U_l^t$, so their magnitudes and the inner amplitudes can be driven to zero as in Lemma~\ref{thm:reduction_parallel}, and the reduction carries over (Appendix~\ref{app:lif_path_reduction}),
$$
\inf_\Theta \mathcal L + \tfrac{\beta_{\mathrm{reg}}}{2}\Omega_{p,2}(\Theta)^2
=
\min_{\Theta:\,|s_k|=1}\mathcal L + \beta_{\mathrm{reg}}\sum_{k=1}^K\|P_{k,\mathrm{out}}\|_p,
$$
The reduction holds in the setting of Appendix~\ref{app:snn_dag}, in which $U^l_{\mathrm{thr}}$ acts as a bias that is rescaled together with $P_{l,\mathrm{in}}$ (trainable, or tied to the input scale). If $U^l_{\mathrm{thr}}$ is held at a constant that is not rescaled, the orbit of Lemma~\ref{lem:lif_membrane_rescaling} leaves the model class and the inner penalty is not removable. With the reduction in place, the zero-duality-gap result of Theorem~\ref{thm:recurrent_bi_dual} (Appendix~\ref{app:recurrent_bidual}) applies with the feature map $S_{k,L-1}^T$, and the dictionary stays finite (Lemma~\ref{lem:snn_finiteness}, Appendix~\ref{app:snn_finiteness}).

\begin{theorem}[Finite convex formulation for LIF-SNNs]
\label{thm:finite_convex_lif_snn}
For weight decay with $p=2$, let
$D_{\mathrm{SNN}}^{L-1,T}\in\{0,1\}^{n\times M_{\mathrm{SNN}}^{L-1,T}}$
be the complete SNN witnessed dictionary at the final hidden layer and final timestep. Then the reduced SNN training problem with weight-decay regularization admits the same finite convex program as in Theorem~\ref{thm:finite_convex_L_T}:
$$
\widetilde P_{\mathrm{SNN}}^{L,T}
=
\min_{\widetilde w\in\mathbb{R}^{M_{\mathrm{SNN}}^{L-1,T}}}
\mathcal L(D_{\mathrm{SNN}}^{L-1,T}\widetilde w,Y)
+
\frac{\beta_{\mathrm{reg}}}{\sqrt{m_{L-1}}}\|\widetilde w\|_1 .
$$
For scalar outputs,
$$
\widetilde P_{\mathrm{SNN}}^{L,T}
=
P_{\mathrm{SNN}}^{L,T,K}
$$
whenever $K\ge K^*$ for some $K^*\le n+1$. For $d_{\mathrm{out}}$ output coordinates, the sufficient bound is
$K^*\le nd_{\mathrm{out}}+1$ after vectorizing the predictions.
\end{theorem}
The proof is given in Appendix~\ref{app:snn_finiteness}.

\textbf{The program depends only on the realizable patterns.} The programs of Theorems~\ref{thm:finite_convex_L_T} and~\ref{thm:finite_convex_lif_snn} are determined by the labels and the binary patterns the network can realize at the final hidden layer and timestep, so the leak, the reset, earlier spikes, spike rates and membrane potentials affect the optimum only through which patterns are realizable. For inputs in general position (conditions (G1), (R1) and (L1) of Appendix~\ref{app:complete}), the realizable set is all of $\{0,1\}^n$ for feedforward threshold networks from the second hidden layer on, for a threshold RNN with one hidden layer of width $n+1$ at every $T\ge2$, and for an LIF-SNN at the second hidden layer when the first has width at least $n$ (Theorems~\ref{thm:ap_full}, \ref{thm:ap_rnn_complete} and~\ref{thm:ap_lif_complete}). For an LIF-SNN with one hidden layer it can be a strict, input-dependent subset (Remark~\ref{rem:ap_lif_notcomplete}). In the complete regimes the program is exact but cannot be enumerated, which the next two subsections address.

\subsection{From the Convex Program to a Network}
\label{sec:practical_impl}
A solution of the program is a set of weighted patterns, and each pattern used needs parameters that produce it. By Theorem~\ref{thm:recurrent_bi_dual}, the program has an optimal solution $\widetilde w^*$ with at most $n+1$ non-zero coefficients. We assign each active column $D_{\mathrm{SNN}}^{L-1,T}[:,i]$ to a subnetwork $k$, whose weights are initialized by a witness $\omega_i$ of that column.

Over the complete witnessed dictionary, the reconstruction returns a minimizer of the reduced problem, and by Lemma~\ref{lem:attain}(ii) scaling its inner weights by a small $\alpha$ gives parameters within any $\varepsilon$ of the optimal value $P$, which itself is not attained in general. The complete dictionary, however, cannot be enumerated, since under the conditions of Theorems~\ref{thm:ap_full}--\ref{thm:ap_lif_complete} it contains all $2^n$ patterns. In practice we fix the witnesses. Fixed witnesses sample all parameters except the last layer from a Gaussian distribution, following \citet{ergen_globally_2023}, and pretrained witnesses take the parameters of an SNN trained on the same task. The resulting dictionary is tractable in size, and the convex program is optimal over this dictionary and no further. Surrogate-gradient pretraining and reservoir tuning are also ways of fixing witnesses, so the convex program provides each with an optimal readout over its patterns. Without a bound on the gap to the complete optimum, however, such a solution carries no guarantee for the original problem.

% Placeholder macro for AISTATS drafting. Remove all \tbc{} before submission.
\providecommand{\tbc}[1]{\textcolor{orange}{[TBC: #1]}}
\providecommand{\R}{\mathbb{R}}
\providecommand{\E}{\mathbb{E}}
% =========================================================================
% §3.5  Approximation guarantees for sampled dictionaries
% Proofs: Appendix~\ref{app:approx} (Sections/Appendix_Approximation.tex).
% =========================================================================
\subsection{How Far the Sampled Program Is from the Optimum}
\label{sec:approx}

% The program of \S\ref{sec:practical_impl} is solved over a sampled dictionary, and it is optimal only over that fixed dictionary. This section bounds its value against the complete dictionary when the witnesses are Gaussian. For scalar outputs we write $y\in\mathbb{R}^n$ for the labels. When the fit term enforces exact interpolation of $y$, the program over a dictionary is $\beta_{\mathrm{reg}}/\sqrt{m_{L-1}}$ times the smallest $\ell_1$ norm of an interpolating weight vector, and this factor cancels in the ratio below. Let $\mathcal D$ be the set of realizable patterns, let $D$ be the matrix of its distinct elements, and let $A\in\{0,1\}^{n\times m}$ collect the columns of $m$ sampled neurons. Define
The sampled program is optimal over its own patterns. We now bound its value relative to the optimum over the complete dictionary. For scalar outputs we write $y\in\mathbb{R}^n$ for the labels. When exact interpolation of $y$ is enforced, the program over a dictionary is $\beta_{\mathrm{reg}}/\sqrt{m_{L-1}}$ times the smallest $\ell_1$ norm of an interpolating weight vector, and this factor cancels in the ratio below. , $D$ the matrix of its distinct elements (each pattern is a different column), and $A\in\{0,1\}^{n\times m}$ a matrix formed by sampling $m$ columns from $D$. Define
$$
p^\star:=\min\{\|u\|_1:\ Du=y\},
\quad
p^\star_m:=\min\{\|v\|_1:\ Av=y\},
$$
with $p^\star_m:=+\infty$ if $Av=y$ is infeasible. The lowercase $p^\star$, $p^\star_m$ are these minimum-$\ell_1$ interpolation values and are distinct from the optimal values $P$ of \S\ref{sec:notation_prelim}--\S\ref{sec:snn_section}. We write $\bK$ for a kernel, to keep $K$ for the number of subnetworks.

The sampled program can only be worse than the complete one. It is worse by at most a factor $C$, since its minimum-norm solution is controlled by the sampled Gram matrix, which concentrates around $m\bK$ once $m$ exceeds the leverage-score width. Concretely $C=(r/\sqrt q)(1-\eta)^{-1/2}$ with $q=y^\top\bK^{-1}y$ and $r=\max_{a\in\mathcal D}|a^\top\bK^{-1}y|$, and we write $C_0:=r/\sqrt q$.

\begin{theorem}[Sampled dictionary, Gaussian witnesses]
\label{thm:approx_template}
Let the columns $a_j$ be i.i.d.\ in $\mathcal D$ with $\mathbb{E}[aa^\top]=\bK\succ0$. Let $\mu=\max_{a\in\mathcal D}a^\top\bK^{-1}a$, $q=y^\top\bK^{-1}y$, and $r=\max_{a\in\mathcal D}|a^\top\bK^{-1}y|$. If
$$
m\ \ge\ \frac{2\,\mu\log(n/\delta)}{\eta^2},
$$
then with probability at least $1-\delta$,
$$
p^\star\ \le\ p^\star_m\ \le\ C\,p^\star,
\qquad
C=\frac{r}{\sqrt q}\,(1-\eta)^{-1/2},
$$
and $n\le\mu\le n/\lambda_{\min}(\bK)$, so the width is $\mathrm{poly}(n,1/\lambda_{\min}(\bK),\log 1/\delta)$.
\end{theorem}
The proof (Appendix~\ref{app:approx_two_layer}) combines the minimum-norm solution of the sampled problem, matrix concentration \citep{tropp_userfriendly_2012,tropp_introduction_2015} and the dual of the complete problem, and $m\gtrsim\mu\log n$ is the leverage-score width of random features \citep{bach_equivalence_2017,avron_random_2017}.

Theorem~\ref{thm:approx_template} covers the following cases. For a two-layer threshold network, $\mathcal D$ is the set of arrangement cells of the data and $\bK$ is the arc-cosine kernel \citep{cho_kernel_2009} (Appendix~\ref{app:approx_two_layer}). For an LIF network with one hidden layer, $\bK$ is the kernel of the LIF trajectory map and $p^\star$ is taken over $\{0,1\}^n$ (Theorem~\ref{thm:ap_lif_one}). The relaxation is needed because one LIF layer is not complete (Remark~\ref{rem:ap_lif_notcomplete}), and the bound stays valid over the realizable patterns. For parallel LIF subnetworks, each with one hidden layer of width at least $n$ and a one-neuron readout layer, every pattern in $\{0,1\}^n$ has positive probability, so the dictionary is complete and $p^\star$ is the true optimum (Theorem~\ref{thm:ap_parallel}). In each case $K$ parallel subnetworks give $m$ i.i.d.\ columns in total.

The same bound covers the reservoir of \textsc{R} and \textsc{R-CVX}, whose Gaussian input weights are combined with a tuned leak $\beta$, threshold $\theta$ and input scale $s$. Without lateral weights its columns are i.i.d.\ given the knobs, also when the knobs are selected on the same draw from a fixed grid (Theorem~\ref{thm:ap_res}, Corollary~\ref{cor:ap_grid}), so tuning affects the guarantee only through the kernel (Theorem~\ref{thm:ap_ratio}). The branching-ratio criterion drives $\theta/s\to0$, where the kernel converges to the arc-cosine kernel of the leaky-integrated inputs (Theorem~\ref{thm:ap_gain}).

% Section for the high-probability approximation bound, if kept separate:
% \section{Approximation Guarantees for Sampled Dictionaries}
% \input{Sections/Approximation}

\section{Experiments}
% Placeholder macro for AISTATS drafting. Remove all \tbc{} before submission.
\providecommand{\tbc}[1]{\textcolor{orange}{[TBC: #1]}}
\label{sec:experiments}
% Two-column floats placed early so they land near their first reference.
\begin{table}[t]
\centering
\small
\caption{Support of the convex solution on small instances ($n_{\mathrm{train}}=23$, bound $n+1=24$, three seeds per task), over the distinct patterns of $15$M witness draws. No dictionary saturated, so each support is that of an optimum over the enumerated patterns. Primal--dual gap below $3\cdot10^{-8}$.}
  \label{tab:support}
\setlength{\tabcolsep}{4pt}
\renewcommand{\arraystretch}{1.15}
\begin{tabular}{@{}lcc@{}}
\toprule
Task & support (per seed) & mean $\pm$ std \\
\midrule
\textsc{Addition}       & $2,2,2$      & $2.0\pm0.0$ \\
\textsc{first-last-XOR} & $17,14,2$  & $11.0\pm6.5$ \\
\textsc{MNIST-Seq}      & $22,6,17$   & $15.0\pm6.7$ \\
\bottomrule
\end{tabular}
\end{table}

\begin{table*}[t]
  \centering
  \small
  \caption{Test accuracy (mean $\pm$ std over 3 seeds; mean only for \textsc{Addition}) of the best variant in each class, selected on the validation split, with the selected variant in brackets. The convex class is \{\textsc{CVX}, \textsc{SG-CVX}, \textsc{R-CVX}\} and their analog-readout (\textsc{-an}) controls, the SG class is \{\textsc{SG}, \textsc{SG-SG}\}, and \textsc{R} is the reservoir ridge baseline. \textsc{Addition}: in-distribution token-wise accuracy at $n_{\mathrm{digits}}=5$. \textsc{first-last-XOR}: $L=5$. Depth $L=10$: Table~\ref{tab:main_L10}.}
  \label{tab:main}
  \setlength{\tabcolsep}{8pt}
  \renewcommand{\arraystretch}{1.15}
  \begin{tabular}{@{}llccc@{}}
    \toprule
    \textbf{Task} & \textbf{Timestep} & \textbf{Best CVX-variant} & \textbf{SG} & \textbf{R} \\
    \midrule
        \textsc{Addition}\textsubscript{Base-2}  & $T=6$  & \textbf{0.954} (SG-CVX) & 0.542 (SG-SG) & 0.904 \\
    \textsc{Addition}\textsubscript{Base-3}  & $T=6$  & 0.366 (CVX)             & 0.238 (SG-SG) & \textbf{0.407} \\
    \textsc{Addition}\textsubscript{Base-5}  & $T=6$  & 0.232 (R-CVX)           & 0.178 (SG-SG) & \textbf{0.357} \\
    \textsc{first-last-XOR}          & $T=6$  & \textbf{1.000$\pm$.000} (R-CVX) & 0.737$\pm$.022 & \textbf{1.000$\pm$.000} \\
    \textsc{first-last-XOR}          & $T=14$ & \textbf{0.544$\pm$.044} (SG-CVX) & 0.516$\pm$.011 & 0.524$\pm$.003 \\
    \textsc{MNIST-Seq}               & $T=4$  & 0.851$\pm$.013 (SG-CVX-an) & \textbf{0.852$\pm$.015} & 0.508$\pm$.053 \\
    \bottomrule
  \end{tabular}
\end{table*}

Since the convex program is solved exactly over any dictionary, the experiments ask two questions: whether the theory describes practice (sparsity of the optimum, reconstruction, and the sampling bound), and whether the best training fit over the sampled patterns generalizes, depending on the witnesses and the task. They are diagnostics rather than a claim of state-of-the-art accuracy. Ablations and a loss-landscape comparison of the readouts are in Appendix~\ref{sec:analysis}.

\textbf{Tasks.} The three tasks require different kinds of patterns. \textsc{Addition} is symbolic, base-$b$ digits with carry, evaluated in and out of distribution over sequence length~\citep{lee2023teaching}. \textsc{first-last-XOR} is temporal credit assignment over the horizon $T$. \textsc{MNIST-Seq} \citep{lecun2010mnist} is a sensory sequence, pixels streamed in time (Appendix~\ref{app:moved_results}).

\textbf{Method.} All convex variants solve the same program over a fixed dictionary and differ only in how the witnesses are generated (\textsc{CVX}: Gaussian, \textsc{SG-CVX}: surrogate-pretrained, \textsc{R-CVX}: criticality-tuned reservoir without lateral weights). The baselines are \textsc{SG}, \textsc{SG-SG} and a reservoir with ridge readout \textsc{R} (Appendix~\ref{app:training_algorithms}).

\textbf{The theory in practice.} On small instances we draw $15$M witnesses, solve the program over the distinct patterns they produce, and read off the support of the solution (Table~\ref{tab:support}). The solution uses two subnetworks in every \textsc{Addition} instance and between $2$ and $22$ on the temporal and visual tasks, below the Carath\'eodory bound $n+1=24$ of Theorem~\ref{thm:finite_convex_L_T} in every case. Rebuilding the network from the support reproduces the convex optimum, with a primal--dual gap below $3\cdot10^{-8}$. At first hidden width $8<n$ the completeness condition of Theorem~\ref{thm:ap_lif_complete} (Appendix~\ref{app:complete}) does not hold, so each optimum is over the enumerated patterns (caption).

\textbf{Sampled dictionaries against the bound.} Table~\ref{tab:sampled_width} reports the ratio $p^\star_m/p^\star$ of Theorem~\ref{thm:approx_template} against $m/n$ together with $C_0=r/\sqrt q$. At fixed $m/n$ both the ratio and $C_0$ grow like $\sqrt n$, and the ratio lies below $C_0$ in every draw from $m=2n$ (one exception in eight at $n=600$), so the bound is loose by a constant factor but consistent with $\sqrt n$ growth over $n=200$--$600$. Tuned reservoir witnesses match untuned Gaussian witnesses in both the ratio and $C_0$, consistent with Theorem~\ref{thm:ap_ratio}. Parallel LIF subnetworks match the single layer from $m=4n$ on, but their kernel is worse conditioned at large $n$ ($C_0=67$ against $42$ at $n=600$), so their guarantee (Theorem~\ref{thm:ap_parallel}) is weaker although their dictionary is complete.

\begin{table}[t]
\centering
\small
\caption{Measured $p^\star_m/p^\star$ and $C_0=r/\sqrt q$ on Gaussian data ($d=20$, balanced labels, $p^\star=2$). Means over $8$ draws, except the subnetwork block ($N=n$), which uses one or two draws with $m\le6000$. At $\eta=\frac12$ the concentration event of Theorem~\ref{thm:approx_template} is observed from $m\approx16n$ and certified ($\delta=0.05$) from $976$--$1252n$ (two-layer), $1276$--$1708n$ (LIF), $1005$--$1275n$ (reservoir) and $1490$--$8170n$ (subnetworks).}
  \label{tab:sampled_width}
\setlength{\tabcolsep}{3pt}
\renewcommand{\arraystretch}{1.1}
\begin{tabular}{@{}llrrrrrr@{}}
\toprule
 & & \multicolumn{5}{c}{$m/n$} & \\
\cmidrule(lr){3-7}
Witnesses & $n$ & $1.5$ & $2$ & $4$ & $8$ & $16$ & $C_0$ \\
\midrule
\multirow{3}{*}{Two-layer} & 200 & 28.9 & 21.1 & 15.2 & 12.3 & 10.4 & 23.0 \\
 & 400 & 40.8 & 30.9 & 21.7 & 17.9 & 15.3 & 33.5 \\
 & 600 & 51.2 & 40.0 & 27.5 & 22.1 & 19.1 & 41.7 \\
\midrule
\multirow{3}{*}{LIF, 1 hidden} & 200 & 28.7 & 22.8 & 15.6 & 12.7 & 10.8 & 24.2 \\
 & 400 & 42.9 & 33.1 & 23.0 & 19.1 & 16.1 & 35.7 \\
 & 600 & 54.4 & 42.2 & 29.1 & 23.9 & 20.4 & 44.8 \\
\midrule
\multirow{3}{*}{Tuned reservoir} & 200 & 28.5 & 21.8 & 14.8 & 12.2 & 10.4 & 22.8 \\
 & 400 & 41.6 & 31.8 & 21.7 & 17.7 & 15.2 & 33.1 \\
 & 600 & 52.0 & 39.2 & 27.3 & 22.1 & 19.0 & 41.6 \\
\midrule
\multirow{3}{*}{LIF subnetworks} & 200 & 35.0 & 23.9 & 15.2 & 11.9 & 10.0 & 24.1 \\
 & 400 & 48.1 & 35.0 & 22.1 & 17.2 & -- & 37.7 \\
 & 600 & 58.8 & 42.7 & 27.3 & 21.9 & -- & 67.0 \\
\bottomrule
\end{tabular}
\end{table}

\textbf{Does the best fit generalize?} Table~\ref{tab:main} compares the best convex variant, selected on the validation split, with the better of \textsc{SG} and \textsc{SG-SG} and with the reservoir baseline \textsc{R}. No witness source dominates. Surrogate-pretrained witnesses are best on base-$2$ \textsc{Addition} ($0.954$ against $0.542$ and $0.904$). On bases $3$ and $5$ the best convex variants exceed surrogate-gradient training but remain below \textsc{R}. On short-horizon \textsc{first-last-XOR} the reservoir and its convex readout reach $1.000$ against $0.737$ for surrogate-gradient training, and at $T=14$ all methods are close to chance. On \textsc{MNIST-Seq} the convex family matches surrogate-gradient training ($0.851$ against $0.852$), and both are well above \textsc{R}. The convex program is solved to a primal--dual gap below $10^{-6}$ in every run (Table~\ref{tab:duality_gap}) and fits the training data in most settings (Table~\ref{tab:train_summary}), so differences between the binary-readout variants are attributable to their witnesses.

% Moved to the appendix for AISTATS (now input from Sections/Appendix.tex):
% \section{Analysis and Ablation Study}
% \input{Sections/Analysis_and_Ablation_Study}

\section{Conclusion}
\label{sec:conclusion}
We showed that weight-decay-regularized training of parallel recurrent threshold networks and LIF-SNNs has the same optimal value as a finite convex program over the realizable spike patterns, and bounded the suboptimality of sampled dictionaries with high probability. Empirically, test accuracy depends on the source of the patterns. Limitations and broader impact are discussed in Appendix~\ref{app:limitations}.

% AISTATS 2027: required, placed immediately before the references. Does not count toward the page limit.
\section*{AI Use Statement}
We used a large language model assistant (Claude, Anthropic) during this work. It was used to edit and restructure the text, to format the LaTeX, and to search for related literature. For the approximation guarantees (\S\ref{sec:approx} and Appendix~\ref{app:approx}), once the authors had specified the method of proof, the assistant was used to derive the explicit constants and numerical values in those proofs. For the experiments, the authors wrote the algorithms as pseudocode, and the assistant was used to implement them and to set up the test benches and library code. The authors formulated the theoretical results and the proofs, designed and ran the experiments, verified every mathematical claim and every reported number, and take full responsibility for the content of the paper.

\bibliographystyle{plainnat}
\bibliography{Whole_thing}

% AISTATS reproducibility checklist (does not count toward the page limit).
% Commented out for the arXiv preprint; restore for the AISTATS submission/camera-ready.
% \clearpage
% \onecolumn
% \input{Sections/Checklist}

\clearpage
\appendix
\onecolumn
\newpage
\appendix
\part*{Appendix}

\section{Feedforward Threshold Networks}
\label{app:feedforward_duality}
We write $\mathcal L^*$ for the Fenchel conjugate of the loss in its first argument, $\mathcal L^*(Z)=\sup_{U}\big\{\langle Z,U\rangle-\mathcal L(U,Y)\big\}$; it satisfies $\inf_{U}\{\mathcal L(U,Y)+\lambda^\top U\}=-\mathcal L^*(-\lambda)$, the identity used to dualize the fit term.

We first prove the reduction of the weight-decay regularizer defined in \S\ref{sec:notation_prelim} to its last-layer norms for a single network before proving Lemma~\ref{thm:reduction_parallel}.

\subsection{Proof of Reduction to outer Norms}
\label{app:reduction_single}

\emph{Intuition.} The reduction rests on the positive scaling invariance of the threshold, $\sigma(cx)=\sigma(x)$ for $c>0$. Because the pre-activation of any hidden neuron may be rescaled without altering the network output, the incoming weight magnitude of every hidden neuron can be sent to zero along a minimizing sequence, collapsing the inner contribution to the weight-decay penalty. Only the output-layer weights retain a scale that affects the function, so the penalty reduces to their norm alone. The equality below is therefore between an infimum on the left, approached as the inner magnitudes vanish, and a minimum on the right; the left-hand value is not in general attained under the convention $\ind{0\ge0}=1$.

\begin{theorem}[Reduction to outer norms]\label{thm:reduction_nonparallel}
Let $G(V,E,w)$ be a fully connected threshold network with amplitudes $\{s_v\}$ and weight-decay regularizer
\[
\Omega_{p,2}(\Theta)^2
=
\sum_{v\in V\setminus V_{\mathrm{in}}}
\Big(|s_v|^p+\sum_{(u,v)\in E}|w(u,v)|^p\Big)^{2/p} .
\]
Then
\[
\inf_\Theta\ \mathcal L(f^{L,\Theta}(X),Y)+\frac{\beta_{\mathrm{reg}}}{2}\Omega_{p,2}(\Theta)^2
=
\min_{\substack{\Theta\\ |s_{V_{\mathrm{out}}}|=1}} \mathcal L(f^{L,\Theta}(X),Y)+\beta_{\mathrm{reg}}\|W(\cdot,V_{\mathrm{out}})\|_p .
\]
\end{theorem}
\begin{proof}
For a hidden node $v$, let $a_v=\big(\sum_{(u,v)\in E}|w(u,v)|^p\big)^{1/p}$ and $z_v=\sum_{(u,v)\in E}w(u,v)o(u)$, so $o(v)=\sigma_{s_v}(z_v)=s_v\ind{z_v\ge0}$. For $\alpha>0$,
\[
\sigma_{s_v}(\alpha z_v)=s_v\ind{\alpha z_v\ge0}=s_v\ind{z_v\ge0}=\sigma_{s_v}(z_v),
\]
hence $f^{L,\Theta}(X)$ is invariant to $a_v$ for every hidden $v$ and to $s_v$ for every hidden $v$ that carries an amplitude. The output retains only $s_{V_{\mathrm{out}}}W(\cdot,V_{\mathrm{out}})$. By this invariance, driving the inner magnitudes to zero along a minimizing sequence,
\[
a_v\to 0\ \ (v\ \text{hidden}),\qquad |s_v|\to 0\ \ (v\ \text{hidden with an amplitude}),
\]
leaves
\[
\inf_\Theta\ \mathcal L(f^{L,\Theta}(X),Y)+\frac{\beta_{\mathrm{reg}}}{2}\Big(|s_{V_{\mathrm{out}}}|^2+\|W(\cdot,V_{\mathrm{out}})\|_p^2\Big).
\]
With $s_{V_{\mathrm{out}}}\mapsto\alpha s_{V_{\mathrm{out}}}$, $W(\cdot,V_{\mathrm{out}})\mapsto\alpha^{-1}W(\cdot,V_{\mathrm{out}})$ fixing the output,
\begin{equation}
\frac{\beta_{\mathrm{reg}}}{2}\Big(\alpha^2|s_{V_{\mathrm{out}}}|^2+\alpha^{-2}\|W(\cdot,V_{\mathrm{out}})\|_p^2\Big)
\ge
\beta_{\mathrm{reg}}\,|s_{V_{\mathrm{out}}}|\,\|W(\cdot,V_{\mathrm{out}})\|_p,
\label{eq:balance}
\end{equation}
with equality at $\alpha^2=\|W(\cdot,V_{\mathrm{out}})\|_p/|s_{V_{\mathrm{out}}}|$. Setting $|s_{V_{\mathrm{out}}}|=1$,
\[
\min_{\substack{\Theta\\ |s_{V_{\mathrm{out}}}|=1}} \mathcal L(f^{L,\Theta}(X),Y)+\beta_{\mathrm{reg}}\|W(\cdot,V_{\mathrm{out}})\|_p. \qedhere
\]
\end{proof}

\subsection{Proof of Lemma~\ref{lem:attain}}\label{app:attain_proof}

\begin{proof}
Write $J(\Theta):=\mathcal L(f^{L,\Theta}(X),Y)+\frac{\beta_{\mathrm{reg}}}2\Omega_{p,2}(\Theta)^2$ for the original objective and $J_Q(\Theta):=\mathcal L(f^{L,\Theta}(X),Y)+\beta_{\mathrm{reg}}\|W(\cdot,V_{\mathrm{out}})\|_p$ for the reduced one, so that $P=\inf_\Theta J$ and $Q=\min_{|s_{V_{\mathrm{out}}}|=1}J_Q$.

(i) \emph{$P\ge Q$.} Fix $\Theta$. Every inner term of $\Omega_{p,2}(\Theta)^2$ is non-negative, so dropping them gives $J(\Theta)\ge\mathcal L(f^{L,\Theta}(X),Y)+\frac{\beta_{\mathrm{reg}}}2\big(|s_{V_{\mathrm{out}}}|^2+\|W(\cdot,V_{\mathrm{out}})\|_p^2\big)$. By \eqref{eq:balance} this is at least $\mathcal L(f^{L,\Theta}(X),Y)+\beta_{\mathrm{reg}}|s_{V_{\mathrm{out}}}|\,\|W(\cdot,V_{\mathrm{out}})\|_p$. The network output depends on $(s_{V_{\mathrm{out}}},W(\cdot,V_{\mathrm{out}}))$ only through the product $s_{V_{\mathrm{out}}}W(\cdot,V_{\mathrm{out}})$. Replacing the pair by $(\pm1,|s_{V_{\mathrm{out}}}|\,W(\cdot,V_{\mathrm{out}}))$ therefore keeps the function and gives a feasible point of $Q$ with value exactly this lower bound. Hence $J(\Theta)\ge Q$ for every $\Theta$, and $P\ge Q$.

\emph{$P\le Q$.} Let $\Theta^\star$ attain $Q$ (attainment is shown below). Rescale the output pair to $(s,W)=(\sqrt{\|W^\star\|_p},\,W^\star/\sqrt{\|W^\star\|_p})$, with $W^\star:=W(\cdot,V_{\mathrm{out}})$ at $\Theta^\star$. This keeps the product, hence the function, and makes $\frac{\beta_{\mathrm{reg}}}2(|s|^2+\|W\|_p^2)=\beta_{\mathrm{reg}}\|W^\star\|_p$. Now multiply every hidden incoming block and every inner amplitude by $\alpha\in(0,1]$, layer by layer. Each layer's output is multiplied by $\alpha$, so the next layer's pre-activation is multiplied by $\alpha$, which the threshold absorbs. The function is therefore unchanged, and the inner part of the penalty is multiplied by $\alpha^2$. So $J(\Theta_\alpha)=Q+O(\alpha^2)\to Q$, and $P\le Q$.

\emph{$Q$ is attained.} On the fixed data $X$ the hidden layers realize finitely many activation patterns, because every threshold layer maps finitely many inputs to finitely many sign patterns. For a fixed pattern of the last hidden layer, $J_Q$ is a convex function of the output weights, and it is coercive because $\beta_{\mathrm{reg}}>0$ and the loss is bounded below. So its minimum over the output weights is attained. $Q$ is the minimum of these finitely many attained minima, hence attained.

(ii) This is the second half of the proof of $P\le Q$ above, started from any minimizer $\Theta^\star$ of $Q$.

(iii) Suppose $\Theta$ attains $P$ and some hidden node $v$ has incoming magnitude $a_v>0$. Multiplying the incoming weights of $v$ by $\alpha<1$ multiplies its pre-activation by $\alpha$, which the threshold absorbs, so the function is unchanged, and the penalty strictly decreases. This contradicts optimality. So every hidden node has $a_v=0$, its pre-activation is $0$, and under $\ind{0\ge0}=1$ its output is the constant $s_v$. Conversely, a network whose hidden nodes all have zero incoming magnitude has zero inner weight penalty, so if it attains $Q$ it attains $P=Q$.
\end{proof}

\subsection{Proof of Lemma~\ref{thm:reduction_parallel}}\label{app:reduction_parallel}

\emph{Intuition.} Parallel subnetworks share no hidden parameters and the penalty adds up over them, so the single-network reduction applies to each subnetwork on its own.

\begin{lemma}[Reduction to outer norms for $K$-parallel networks]
\label{thm:reduction_parallel}
Let $G(V,E,w)$ be a fully connected $K$-parallel threshold network with weight-decay regularizer $\Omega_{p,2}(\Theta)^2=\sum_{k=1}^K\Omega_{p,2}(\Theta_k)^2$, where the hidden neurons of layers $\ell\le L-2$ carry amplitudes and the readout block $W_{L,k}\in\mathbb{R}^{m_{L-1}\times d_{\mathrm{out}}}$ of each subnetwork carries one readout amplitude $s_k$. Then the weight-decay regularized training problem has the same optimal value, under the constraint $|s_k|=1$ on the readout amplitude of each subnetwork, as the problem in which the regularizer is the outer-layer norm:
\begin{enumerate}
    \item $f^{L,K,\Theta}(X)=f^{L,K,\bar\Theta}(X)$ for all inputs $X$, where $\bar\Theta$ is $\Theta$ with its inner weights and inner amplitudes rescaled as in Lemma~\ref{lem:attain}(ii);
    \item the inner weight magnitudes $\{\|W_{l,k}\|_p\}_{l\le L-1}$ and inner amplitudes $\{s^{(l)}_k\}_{l\le L-2}$ are driven to zero along a minimizing sequence, giving
    \[
    \inf_\Theta\ \mathcal L + \tfrac{\beta_{\mathrm{reg}}}{2}\Omega_{p,2}(\Theta)^2
    =
    \min_{\Theta:\,|s_k|=1}\mathcal L + \beta_{\mathrm{reg}}\sum_{k=1}^K\|W_{L,k}\|_p,
    \]
    where the equality is between an infimum and an attained minimum in the sense of Lemma~\ref{lem:attain}. Here $\|\cdot\|_p$ is the entrywise $p$-norm of the block (Frobenius for $p=2$). For $K=1$ the block $W_{L,1}$ is $W(\cdot,V_{\mathrm{out}})$, $s_1$ is $s_{V_{\mathrm{out}}}$, and the penalty is $\|W(\cdot,V_{\mathrm{out}})\|_p$ of Lemma~\ref{lem:attain}.
\end{enumerate}
\end{lemma}

Because the parallel subnetworks share input nodes but no hidden parameters, and the weight-decay regularizer is additive across them, the single-network reduction applies independently to each subnetwork. Each retains only the norm of its own output weights, and the reduced penalties sum over $k$. As in Lemma~\ref{lem:attain}, the left-hand side below is an infimum that is approached but not in general attained.

\begin{lemma}
Since $\Omega_{p,2}(\Theta)^2=\sum_{k=1}^K\Omega_{p,2}(\Theta_k)^2$ and the $G_k$ have disjoint parameters, the previous theorem applies per $k$: driving $\{\|W_{l,k}\|_p\}_{l\le L-1}\to0$ and $\{|s^{(l)}_k|\}_{l\le L-2}\to0$ along a minimizing sequence leaves
\[
\frac{\beta_{\mathrm{reg}}}{2}\left(|s_k|^2+\|W_{L,k}\|_p^2\right)
\ge
\beta_{\mathrm{reg}}\,|s_k|\,\|W_{L,k}\|_p,
\]
where $s_k$ is the readout amplitude of subnetwork $k$, and $|s_k|=1$ gives $\beta_{\mathrm{reg}}\|W_{L,k}\|_p$. With $f^{L,K,\Theta}=\sum_{k=1}^K f^{L,k,\Theta_k}$,
\[
\inf_\Theta\ \mathcal L(f^{L,K,\Theta}(X),Y)+\frac{\beta_{\mathrm{reg}}}{2}\Omega_{p,2}(\Theta)^2
=
\min_{\substack{\Theta\\ |s_k|=1}}\ \mathcal L(f^{L,K,\Theta}(X),Y)+\beta_{\mathrm{reg}}\sum_{k=1}^K\|W_{L,k}\|_p.
\]
\end{lemma}

\section{Proof of Proposition~\ref{thm:feedforward_bi_dual}}
\label{app:feedforward_bidual}

\emph{Intuition.} The reduced problem is a sum over subnetworks of a fixed feature map composed with a norm-penalized output weight. Dualizing the fit term through its Fenchel conjugate and the output weights through the conjugacy of the $\ell_p$ norm yields a dual that constrains the correlation between the multiplier and every realizable feature map. The bidual is then the lifting of this constraint to a measure over feature maps, and strong duality holds because the constraint is strictly feasible at $\lambda=0$. Carath\'eodory's theorem reduces the supporting measure to finitely many atoms, one per subnetwork, giving the finite-width equivalence.

\begin{proposition}[Bi-dual of parallel feedforward threshold networks]
\label{thm:feedforward_bi_dual}
Let $\phi_k(X;\theta_k)$ denote the final hidden-layer feature map of subnetwork $k$:
$
\phi_k(X;\theta_k)
=
\sigma\!\left(
\cdots
\sigma\!\left(\sigma(XW_{1,k})W_{2,k}\right)
\cdots W_{L-1,k}
\right)
$, and $\phi(X;\theta)=\{\phi_k(X;\theta_k)\}_{k=1}^K$ by applying Lemma~\ref{thm:reduction_parallel}.
Let $\Lambda^{L,K}$ denote the dual of the reduced weight-decay regularized training problem for a $K$-parallel feedforward threshold network with convex loss $\mathcal L$ and regularization parameter $\beta_{\mathrm{reg}}>0$. Then the dual problem is
$$
\Lambda^{L,K}
=
\max_{\lambda}\ -\mathcal L^*(-\lambda)
\quad
\mathrm{s.t.}
\quad
\max_{\theta_k\in\bar\Theta_k}
\|\lambda^\top\phi_k(X;\theta_k)\|_{p^*}
\le \beta_{\mathrm{reg}},
\qquad k\in[K].
$$
The corresponding infinite bi-dual problem is
$$
p_{\infty}^{L,K}
=
\min_{\nu}
\left\{
\mathcal L\!\left(
\int_{\theta\in\bar\Theta}
\phi(X;\theta)\,d\nu(\theta),
Y
\right)
+
\beta_{\mathrm{reg}}\|\nu\|_{\mathrm{TV}}
\right\},
$$
where $\nu$ is a finite signed measure over the reduced feature maps $\phi(X;\theta)$, indexed by $\theta\in\bar\Theta$; the atomic construction is given in Appendix~\ref{app:feedforward_bidual}. Since $\beta_{\mathrm{reg}}>0$, strong duality holds:
$$
\Lambda^{L,K}=p_{\infty}^{L,K}.
$$
Moreover, for scalar outputs, if $K\ge K^*$ for some $K^*\le n+1$, then
$$
P^{L,K}=\Lambda^{L,K}=P_{\infty}^{L,K}.
$$
For $d_{\mathrm{out}}$ output coordinates, the same Carath\'eodory argument gives
$K^*\le nd_{\mathrm{out}}+1$ after vectorizing the predictions.
\end{proposition}

\begin{proof}
After applying Lemma~\ref{thm:reduction_parallel}, the reduced training problem is
\[
P^{L,K}
=
\min_{\{\theta_k,w_k\}_{k=1}^K}
\mathcal L\!\left(
\sum_{k=1}^K \phi_k(X;\theta_k)w_k,
Y
\right)
+
\beta_{\mathrm{reg}}\sum_{k=1}^K\|w_k\|_p,
\]
where $\theta_k\in\bar\Theta_k$ denotes the reduced hidden parameters of subnetwork $k$, and $w_k$ denotes its output-layer weights.

Introduce an auxiliary variable
\[
\hat Y
=
\sum_{k=1}^K \phi_k(X;\theta_k)w_k.
\]
The equivalent constrained problem is
\[
\min_{\hat Y,\{\theta_k,w_k\}}
\mathcal L(\hat Y,Y)
+
\beta_{\mathrm{reg}}\sum_{k=1}^K\|w_k\|_p
\quad
\mathrm{s.t.}
\quad
\hat Y
=
\sum_{k=1}^K \phi_k(X;\theta_k)w_k.
\]
Let $\lambda$ be the Lagrange multiplier for this equality constraint. The Lagrangian is
\[
\mathcal L(\hat Y,Y)
+
\lambda^\top \hat Y
-
\lambda^\top
\sum_{k=1}^K \phi_k(X;\theta_k)w_k
+
\beta_{\mathrm{reg}}\sum_{k=1}^K\|w_k\|_p.
\]
Taking the infimum over $\hat Y$ gives the Fenchel conjugate defined in \S\ref{sec:notation_prelim},
\[
\inf_{\hat Y}
\left\{
\mathcal L(\hat Y,Y)+\lambda^\top \hat Y
\right\}
=
-\mathcal L^*(-\lambda).
\]
For each $k$, the infimum over $w_k$ is
\[
\inf_{w_k}
\left\{
-\lambda^\top\phi_k(X;\theta_k)w_k
+
\beta_{\mathrm{reg}}\|w_k\|_p
\right\}.
\]
By the conjugacy of the $\ell_p$ norm, this infimum is finite and equal to zero if and only if
\[
\|\lambda^\top\phi_k(X;\theta_k)\|_{p^*}\le \beta_{\mathrm{reg}};
\]
otherwise it is $-\infty$. Since each subnetwork has its own hidden parameters, the dual must satisfy this constraint uniformly over all admissible hidden parameters. Therefore, the dual is
\[
\Lambda^{L,K}
=
\max_{\lambda}\ -\mathcal L^*(-\lambda)
\quad
\mathrm{s.t.}
\quad
\max_{\theta_k\in\bar\Theta_k}
\|\lambda^\top\phi_k(X;\theta_k)\|_{p^*}
\le \beta_{\mathrm{reg}},
\qquad k\in[K].
\]
If all subnetworks share the same admissible hidden-parameter class, this is equivalently a single uniform constraint over that class.

We next derive the bidual. Define the reduced atom set
\[
\mathcal A
=
\left\{
\phi(X;\theta)c
:
\theta\in\bar\Theta,\ \|c\|_p\le 1
\right\}
\subseteq \mathbb{R}^n
\]
for scalar outputs. For vector-valued outputs, we regard the output matrix as a vector in $\mathbb{R}^{nd_{\mathrm{out}}}$ and define $\mathcal A$ analogously after vectorization.

The dual constraint can be written as
\[
\sup_{a\in\mathcal A}\lambda^\top a\le \beta_{\mathrm{reg}}.
\]
The associated atomic gauge is
\[
\gamma_{\mathcal A}(z)
=
\inf_{\nu}
\left\{
\|\nu\|_{\mathrm{TV}}
:
z=\int_{a\in\mathcal A} a\,d\nu(a)
\right\}.
\]
The bidual is therefore
\[
P_{\infty}^{L,K}
=
\min_z
\mathcal L(z,Y)+\beta_{\mathrm{reg}}\gamma_{\mathcal A}(z),
\]
or equivalently,
\[
P_{\infty}^{L,K}
=
\min_{\nu}
\left\{
\mathcal L\!\left(
\int_{a\in\mathcal A} a\,d\nu(a),
Y
\right)
+
\beta_{\mathrm{reg}}\|\nu\|_{\mathrm{TV}}
\right\}.
\]
Reparameterizing the atoms $a=\phi(X;\theta)c$ gives the stated form over reduced hidden parameters.

Strong duality follows from Slater's condition. At $\lambda=0$,
\[
\max_{\theta\in\bar\Theta}
\|0^\top\phi(X;\theta)\|_{p^*}
=
0
<
\beta_{\mathrm{reg}},
\]
since $\beta_{\mathrm{reg}}>0$. Hence the dual feasible set has nonempty relative interior, and
\[
\Lambda^{L,K}=P_{\infty}^{L,K}.
\]

Finally, the finite-width equivalence follows from Carath\'eodory's theorem. For scalar outputs, the atom set lies in $\mathbb{R}^n$, so any point in the convex hull of the atom set can be represented using at most $n+1$ atoms. Each atom corresponds to one parallel subnetwork. Therefore, if
\[
K\ge K^*
\qquad
\text{for some}
\qquad
K^*\le n+1,
\]
the finite $K$-parallel primal can realize the same prediction vector as the infinite bidual. Hence
\[
P^{L,K}=P_{\infty}^{L,K}.
\]
Combining this with strong duality gives
\[
P^{L,K}
=
\Lambda^{L,K}
=
P_{\infty}^{L,K}.
\]
For vector-valued outputs with $d_{\mathrm{out}}$ output dimensions, the same argument applies in $\mathbb{R}^{nd_{\mathrm{out}}}$, giving the bound $K^*\le nd_{\mathrm{out}}+1$.
\end{proof}
\section{Recurrent Threshold Network Theory}
\label{app:recurrent_theory}
\subsection{Unrolled DAG representation and recurrent reduction}

Recall that for a $K$-parallel Recurrent Threshold Network with fixed time horizon $T$ and $L$ layers, its activations are defined via the relation

\begin{align*}
    H^t_{k,\ell} &= \sigma\left(H_{k,\ell-1}^t P_{k,\ell,\textrm{in}} + H_{k,\ell}^{t-1} P_{k,\ell,\textrm{rec}}\right)
\end{align*}

for $t \in [T]$, $k \in [K]$, and $\ell \in [L-1]$ and

\begin{align*}
    H_{k,0}^t = X^t, &\quad H_{k,\ell}^0 = 0.
\end{align*}

We can visualize this as a special case of a $(L + T - 1)$-layer feedforward threshold network, we visualize this for a single subnetwork in Figure~\ref{fig:rnn-unrolling} .

\begin{figure*}[t]
    \centering
    \resizebox{0.9\textwidth}{!}{\begin{tikzpicture}[
    >=Latex,
    cell/.style={
        draw,
        rounded corners=2pt,
        minimum width=1.45cm,
        minimum height=0.72cm,
        align=center
    },
    input/.style={
        minimum width=0.8cm,
        minimum height=0.6cm,
        align=center
    },
    label/.style={
        font=\scriptsize,
        inner sep=2.5pt,
        outer sep=2pt,
        fill=white
    },
    layerlabel/.style={
        font=\scriptsize\bfseries,
        anchor=east
    },
    guide/.style={
        dashed,
        gray!35,
        thin
    }
]

% ============================================================
% LEFT: ORIGINAL RNN
% ============================================================

\node[font=\large\bfseries] at (0,0.9)
    {Depth-$L$ RNN};

\node[input] (x) at (0,-0.2) {$X^t$};

\node[cell] (r1) at (0,-1.35) {$H_t^{(1)}$};
\node[cell] (r2) at (0,-2.75) {$H_t^{(2)}$};
\node[cell] (r3) at (0,-4.15) {$\vdots$};
\node[cell] (r4) at (0,-5.55) {$H_t^{(L)}$};

\node[input] (y) at (0,-6.7) {$\hat y_t$};

% ------------------------------------------------------------
% Within-timestep connections
% ------------------------------------------------------------

\draw[->,thick,draw=blue!75!black]
    (x) -- (r1)
    node[midway,right=5pt,label,text=blue!75!black]
    {$P_{1,\mathrm{in}}$};

\draw[->,thick,draw=green!60!black]
    (r1) -- (r2)
    node[midway,right=5pt,label,text=green!60!black]
    {$P_{2,\mathrm{in}}$};

\draw[->,thick]
    (r2) -- (r3);

\draw[->,thick,draw=purple!75!black]
    (r3) -- (r4)
    node[midway,right=5pt,label,text=purple!75!black]
    {$P_{L,\mathrm{in}}$};

\draw[->,thick]
    (r4) -- (y);

% ------------------------------------------------------------
% Recurrent self-loops
% ------------------------------------------------------------

\draw[
    ->,
    thick,
    draw=orange!85!black
]
    (r1.east)
    to[out=20,in=70,looseness=7]
    (r1.north east)
    node[pos=0.52,right=7pt,
         text=orange!85!black,
         font=\scriptsize,
         inner sep=2.5pt,
         fill=white]
    {$P_{1,\mathrm{rec}}$};

\draw[
    ->,
    thick,
    draw=red!70!black
]
    (r2.east)
    to[out=20,in=70,looseness=7]
    (r2.north east)
    node[pos=0.52,right=7pt,
         text=red!70!black,
         font=\scriptsize,
         inner sep=2.5pt,
         fill=white]
    {$P_{2,\mathrm{rec}}$};

\draw[
    ->,
    thick,
    draw=brown!75!black
]
    (r4.east)
    to[out=20,in=70,looseness=7]
    (r4.north east)
    node[pos=0.52,right=7pt,
         text=brown!75!black,
         font=\scriptsize,
         inner sep=2.5pt,
         fill=white]
    {$P_{L,\mathrm{rec}}$};

% ============================================================
% RIGHT: UNROLLED NETWORK
% ============================================================

\begin{scope}[xshift=5.0cm]

% ------------------------------------------------------------
% Title
% ------------------------------------------------------------

\node[font=\large\bfseries] at (5.0,0.9)
    {Depth-$(L+T-1)$ FNN};

% ------------------------------------------------------------
% Layer labels
% ------------------------------------------------------------

\node[layerlabel] at (-0.75,-1.25) {Layer 1};
\node[layerlabel] at (-0.75,-2.55) {Layer 2};
\node[layerlabel] at (-0.75,-3.85) {$\vdots$};
\node[layerlabel] at (-0.75,-5.15) {Layer $L$};

% ------------------------------------------------------------
% Horizontal guides to emphasize layers
% ------------------------------------------------------------

\draw[guide] (-0.55,-1.25) -- (11.8,-1.25);
\draw[guide] (-0.55,-2.55) -- (11.8,-2.55);
\draw[guide] (-0.55,-3.85) -- (11.8,-3.85);
\draw[guide] (-0.55,-5.15) -- (11.8,-5.15);

% ------------------------------------------------------------
% TIMESTEP LABELS
% ------------------------------------------------------------

\node[font=\scriptsize\bfseries] at (0,0.25) {$t=1$};
\node[font=\scriptsize\bfseries] at (3.0,0.25) {$t=2$};
\node[font=\scriptsize\bfseries] at (7.0,0.25) {$\cdots$};
\node[font=\scriptsize\bfseries] at (11.0,0.25) {$t=T$};

% ------------------------------------------------------------
% Inputs
% ------------------------------------------------------------

\node[input] (x1) at (0,-0.15) {$X^1$};
\node[input] (x2) at (3.0,-0.15) {$X^2$};
\node[input] (xdots) at (7.0,-0.15) {$\vdots$};
\node[input] (xT) at (11.0,-0.15) {$X^T$};

% ------------------------------------------------------------
% Hidden states: t = 1
% ------------------------------------------------------------

\node[cell] (h11) at (0,-1.25) {$H_1^{(1)}$};
\node[cell] (h12) at (0,-2.55) {$H_1^{(2)}$};
\node[cell] (h13) at (0,-3.85) {$\vdots$};
\node[cell] (h14) at (0,-5.15) {$H_1^{(L)}$};

% ------------------------------------------------------------
% Hidden states: t = 2
% ------------------------------------------------------------

\node[cell] (h21) at (3.0,-1.25) {$H_2^{(1)}$};
\node[cell] (h22) at (3.0,-2.55) {$H_2^{(2)}$};
\node[cell] (h23) at (3.0,-3.85) {$\vdots$};
\node[cell] (h24) at (3.0,-5.15) {$H_2^{(L)}$};

% ------------------------------------------------------------
% Ellipsis representing t = 3,...,T-1
% ------------------------------------------------------------

\node[font=\large] (hd1) at (7.0,-1.25) {$\cdots$};
\node[font=\large] (hd2) at (7.0,-2.55) {$\cdots$};
\node[font=\large] (hd3) at (7.0,-3.85) {$\cdots$};
\node[font=\large] (hd4) at (7.0,-5.15) {$\cdots$};

% ------------------------------------------------------------
% Hidden states: t = T
% ------------------------------------------------------------

\node[cell] (hT1) at (11.0,-1.25) {$H_T^{(1)}$};
\node[cell] (hT2) at (11.0,-2.55) {$H_T^{(2)}$};
\node[cell] (hT3) at (11.0,-3.85) {$\vdots$};
\node[cell] (hT4) at (11.0,-5.15) {$H_T^{(L)}$};

% ------------------------------------------------------------
% Input connections
% ------------------------------------------------------------

\draw[->,thick,draw=blue!75!black]
    (x1) -- (h11);

\draw[->,thick,draw=blue!75!black]
    (x2) -- (h21);

\draw[->,thick,draw=blue!75!black]
    (xT) -- (hT1);

% ------------------------------------------------------------
% Within-timestep connections
% ------------------------------------------------------------

% t = 1
\draw[->,thick,draw=green!60!black]
    (h11) -- (h12)
    node[midway,left=8pt,label,text=green!60!black]
    {$P_{2,\mathrm{in}}$};

\draw[->,thick]
    (h12) -- (h13);

\draw[->,thick,draw=purple!75!black]
    (h13) -- (h14)
    node[midway,left=8pt,label,text=purple!75!black]
    {$P_{L,\mathrm{in}}$};

% t = 2
\draw[->,thick,draw=green!60!black]
    (h21) -- (h22);

\draw[->,thick]
    (h22) -- (h23);

\draw[->,thick,draw=purple!75!black]
    (h23) -- (h24);

% t = T
\draw[->,thick,draw=green!60!black]
    (hT1) -- (hT2);

\draw[->,thick]
    (hT2) -- (hT3);

\draw[->,thick,draw=purple!75!black]
    (hT3) -- (hT4);

% ------------------------------------------------------------
% Recurrent connections
% ------------------------------------------------------------

% P_{1,rec}: t=1 -> t=2
\draw[->,thick,draw=orange!85!black]
    (h11) -- (h21)
    node[midway,above=7pt,label,text=orange!85!black]
    {$P_{1,\mathrm{rec}}$};

% P_{2,rec}: t=1 -> t=2
\draw[->,thick,draw=red!70!black]
    (h12) -- (h22)
    node[midway,above=7pt,label,text=red!70!black]
    {$P_{2,\mathrm{rec}}$};

% P_{L,rec}: t=1 -> t=2
\draw[->,thick,draw=brown!75!black]
    (h14) -- (h24)
    node[midway,above=7pt,label,text=brown!75!black]
    {$P_{L,\mathrm{rec}}$};

% ------------------------------------------------------------
% Recurrent connections across omitted timesteps
% ------------------------------------------------------------

% Layer 1
\draw[->,thick,draw=orange!85!black]
    (h21) -- (hT1)
    node[midway,above=7pt,label,text=orange!85!black]
    {$P_{1,\mathrm{rec}}$};

% Layer 2
\draw[->,thick,draw=red!70!black]
    (h22) -- (hT2)
    node[midway,above=7pt,label,text=red!70!black]
    {$P_{2,\mathrm{rec}}$};

% Layer L
\draw[->,thick,draw=brown!75!black]
    (h24) -- (hT4)
    node[midway,above=7pt,label,text=brown!75!black]
    {$P_{L,\mathrm{rec}}$};

% ------------------------------------------------------------
% Outputs
% ------------------------------------------------------------

\node[input] (y1) at (0,-6.25) {$\hat y_1$};
\node[input] (y2) at (3.0,-6.25) {$\hat y_2$};
\node[input] (yT) at (11.0,-6.25) {$\hat y_T$};

\draw[->,thick]
    (h14) -- (y1);

\draw[->,thick]
    (h24) -- (y2);

\draw[->,thick]
    (hT4) -- (yT);

\end{scope}

\end{tikzpicture}}
    \caption{%
        A depth-$L$ recurrent neural network (left) and its
        unrolling over $T$ time steps as a depth-$(L+T-1)$
        feedforward computational graph (right)d.
        Colors indicate shared parameter classes.
    }
    \label{fig:rnn-unrolling}
\end{figure*}

Mathematically, we can write this as a feedforward network with input

\begin{align*}
    X^{1:T} &= \begin{bmatrix}
    | & | & ... & |\\
    X^1 & X^2 & ... & X^T\\
    | & | & ... & |
    \end{bmatrix}
\end{align*}

and weights $\{W_{\ell'}\}_{\ell' = 1}^{L + T - 1}$ with $W_{\ell'} \in \R^{m'_{\ell'-1} \times m'_{\ell'}}$.

For $\ell' \leq L$ we can write the weights in the form:

\begin{align*}
W_{\ell'} &= \begin{bmatrix}
P_{\ell’,\mathrm{in}} & P_{\ell’-1,\mathrm{rec}} & 0 & \cdots & 0 & 0\\
0 & P_{\ell’-1,\mathrm{in}} & P_{\ell’-2,\mathrm{rec}} & \ddots & \vdots & \vdots\\
\vdots & \ddots & \ddots & \ddots & 0 & \vdots\\
0 & \cdots & 0 & P_{2,\mathrm{in}} & P_{1,\mathrm{rec}} & 0\\
0 & \cdots & \cdots & 0 & P_{1,\mathrm{in}} & 0\\
0 & \cdots & \cdots & \cdots & 0 & I_d\\
\vdots & & & & & \ddots
\end{bmatrix}.
\label{eq:rnn-fnn-schematic1}
\end{align*}

For $L < \ell' \leq L + T - 1$ we can write

\begin{align*}
W_{\ell’} &=\begin{bmatrix}
I_d & 0 & 0 & \cdots & 0 & 0 & 0\\
0 & I_d & 0 & \cdots & 0 & 0 & 0\\
0 & 0 & I_d & \ddots & \vdots & \vdots & \vdots\\
\vdots & \ddots & \ddots & \ddots & 0 & 0 & 0\\
0 & \cdots & 0 & I_d & 0 & 0 & 0\\
0 & \cdots & 0 & P_{\ell’-L,\mathrm{in}}
& P_{\ell’-L-1,\mathrm{rec}} & 0 & 0\\
0 & \cdots & 0 & 0
& P_{\ell’-L-1,\mathrm{in}}
& P_{\ell’-L-2,\mathrm{rec}} & 0\\
\vdots & & & \vdots & \ddots & \ddots & \ddots\\
0 & \cdots & 0 & 0 & \cdots & 0
& P_{1,\mathrm{in}}
\end{bmatrix}.
\end{align*}

Now, at this point we have shown that the threshold RNNs can be viewed as FNNs with threshold activations and a restricted weight structure, we further note that due to the threshold activation being invariant under positive scalings that we can preserve the output of the model via adding in multiplicative constant $\alpha > 0$ and scaling

\begin{align*}
    H_{k,\ell}^t &= \sigma \left(H^t_{k,\ell-1} (\alpha P_{k,\ell,\textrm{in}}) + H_{k,\ell}^{t-1} (\alpha P_{k,\ell,\textrm{rec}} )\right)
\end{align*}

Let

\begin{align*}
    R_{p,q}(P_{\textrm{in}}, P_{\textrm{rec}}) &= \left( \left(|s|^p + \sum_{i} \left(|\lVert P_{\textrm{in}}[:,i]\rVert_p^p + \lVert P_{\textrm{rec}}[:,i]\rVert_p^p\right)\right)^{q/p}\right)^{1/q}\\
    R'_{p,q}(P_{\textrm{out}}) &= \left(\sum_{i} \left(\lVert P_{\textrm{out}}[:,i]\rVert_p\right)^{q}\right)^{1/q}
\end{align*}

where $s$ is the trainable activation for the given hidden layer.

Recall the recurrent threshold network training problem with $q= 2$

\begin{align*}
    P^{L,T,K} &= \inf_{\Theta}\left\{ \mathcal{L} (f^{L,T,K,\Theta}(X^{1:T}), Y) + \frac{\beta_{\mathrm{reg}}}{2} \sum_{k=1}^{K} \left[\sum_{\ell =1}^{L-1}  R_{p,2}(P_{k,\ell,\text{in}},P_{k,\ell,\text{rec}})^2  + R'_{p,2}( P_{k,\textrm{out}})^2\right]\right\}\\
\end{align*}

where the infimum is taken due to the threshold activation being discontinuous at $0$ which makes the minimum unattainable in this problem as described in the below theorem.

\emph{Intuition.} A recurrent neuron reads the layer below and its own past through two weight blocks. Scaling both blocks by the same factor scales the whole pre-activation, so the feedforward argument goes through when the two blocks are shrunk together.

\begin{theorem}[Outer Norm Regularization Reduction for $K$-Parallel Recurrent Threshold Networks]
\label{thm:par-rec-thresh-reg}
    The weight-decay regularized training problem for $K$-parallel recurrent threshold networks has optimal value attained by the outer-norm regularized optimization problem
    \begin{align*}
        \min_\Theta \left\{ \mathcal{L}(f^{L,T,K,\Theta}(X^{1:T}), Y) + \beta_{\mathrm{reg}} \sum_{k=1}^K R'_{p,2}(P_{k,\textrm{out}}) \right\}.
    \end{align*}
\end{theorem}

\begin{proof}
    First, we note that for any given $\Theta$, we can obtain the $\Theta'$ by mapping all $(P_{k,\ell,\textrm{in}}, P_{k,\ell,\textrm{rec}}) \mapsto (\alpha P_{k,\ell,\textrm{in}}, \alpha P_{k,\ell,\textrm{rec}})$ with $\alpha \in (0,1)$ such that

    \begin{align*}
        f^{L,T,K,\Theta} \equiv f^{L,T,K,\Theta'}
    \end{align*}

    by the aforementioned invariance of $\sigma$ under positive scalings. 

    % \begin{align*}
    %     \sum_{k=1}^{K} \left[\sum_{\ell =1}^{L-1} (\lVert \alpha P_{k,\ell,\text{in}} \rVert_F^2 + \lVert \alpha P_{k,\ell,\text{rec}} \rVert_F^2)  + \lVert P_{k,\textrm{out}} \rVert_F^2\right] &\leq \sum_{k=1}^{K} \left[\sum_{\ell =1}^{L-1} (\lVert P_{k,\ell,\text{in}} \rVert_F^2 + \lVert P_{k,\ell,\text{rec}} \rVert_F^2)  + \lVert P_{k,\textrm{out}} \rVert_F^2\right]
    % \end{align*}

    \begin{align*}
       \sum_{k=1}^{K} \left[\sum_{\ell =1}^{L-1}  R_{p,2}(\alpha P_{k,\ell,\text{in}},\alpha P_{k,\ell,\text{rec}})^2  + R'_{p,2}( P_{k,\textrm{out}})^2\right] &= \sum_{k=1}^{K} \left[\sum_{\ell =1}^{L-1}  \alpha^2 R_{p,2}( P_{k,\ell,\text{in}}, P_{k,\ell,\text{rec}})^2  + R'_{p,2}( P_{k,\textrm{out}})^2\right]\\
       &\leq \sum_{k=1}^{K} \left[\sum_{\ell =1}^{L-1}  R_{p,2}(P_{k,\ell,\text{in}},P_{k,\ell,\text{rec}})^2  + R'_{p,2}( P_{k,\textrm{out}})^2\right]
    \end{align*}

    and as $\alpha, s_{k,\ell} \downarrow 0$ for all $\ell < L$ we note that the LHS monotonically decreases to $\sum_{k=1}^K |s_{k,L}|^2 + R'_{p,2}( P_{k,\textrm{out}})^2$ and, by an identical scaling argument to that of Lemma~\ref{thm:reduction_parallel} we can see that the optimization problem is equivalent to

    \begin{align*}
        P^{L,T,K} &\geq \inf_\Theta \left\{ \mathcal{L}(f^{L,T,K,\Theta}(X^{1:T}), Y) + \beta_{\mathrm{reg}} \sum_{k=1}^K R'_{p,2}( P_{k,\textrm{out}}) \right\}
    \end{align*}

    Now, we note that there are only a finite number of binary activations which can be encoded in the hidden layer weights of $\Theta$ and that via fixing a particular pattern, the resulting problem is a convex optimization problem in $P_{k,\textrm{out}}$ and thus admits an attainable minimum. Taking the best value over all finite possible hidden activation patterns yields that we can replace the above $\inf$ with a $\min$.

    We also note that if

    \begin{align*}
        P^{L,T,K} &> \min_\Theta \left\{ \mathcal{L}(f^{L,T,K,\Theta}(X^{1:T}), Y) + \beta_{\mathrm{reg}} \sum_{k=1}^K R'_{p,2}( P_{k,\textrm{out}}) \right\} =: Q^{L,T,K}
    \end{align*}

    and let $\Theta^*$ be a set of parameters that attains the minimum, then we have that

    \begin{align*}
        \lim_{\alpha \downarrow 0}\mathcal{L}(f^{L,T,K,\Theta}(X^{1:T}), Y) + \beta_{\mathrm{reg}} \sum_{k=1}^{K} \left[\sum_{\ell =1}^{L-1}  R_{p,2}(\alpha P_{k,\ell,\text{in}},\alpha P_{k,\ell,\text{rec}})^2  + R'_{p,2}( P_{k,\textrm{out}})\right] &= Q^{L,T,K}
    \end{align*}

    which implies $P^{L,T,K} = Q^{L,T,K}$.
\end{proof}

\subsection{Tied-weight reduction: a worked example}
\label{app:recurrent_example}

We illustrate the joint rescaling on the smallest nontrivial case: one subnetwork ($K=1$), one hidden layer ($L=2$), across two timesteps ($T=2$). We drop the subnetwork index and write the single hidden neuron $i$ with input block $P_{\mathrm{in}}[:,i]$ and recurrent block $P_{\mathrm{rec}}[:,i]$.

Unrolling the recurrence across the two timesteps, with $H^0=0$, gives the hidden pre-activations
$$
z^1 = X^1 P_{\mathrm{in}}[:,i],
\qquad
H^1=\sigma(z^1),
$$
$$
z^2 = X^2 P_{\mathrm{in}}[:,i] + H^1 P_{\mathrm{rec}}[:,i],
\qquad
H^2=\sigma(z^2).
$$
In the unrolled DAG the input block $P_{\mathrm{in}}[:,i]$ labels \emph{two} edges, one into $z^1$ at the first timestep and one into $z^2$ at the second, while the recurrent block $P_{\mathrm{rec}}[:,i]$ labels the single edge carrying $H^1$ from the first timestep into $z^2$ at the second. The same parameters appear on edges at different timesteps, which is the tying: they cannot be varied independently across time because they are shared.

Rescale both blocks by a common $\alpha>0$,
$$
P_{\mathrm{in}}[:,i]\mapsto \alpha P_{\mathrm{in}}[:,i],
\qquad
P_{\mathrm{rec}}[:,i]\mapsto \alpha P_{\mathrm{rec}}[:,i].
$$
At the first timestep $z^1\mapsto \alpha z^1$, so $H^1=\sigma(z^1)$ is unchanged. At the second timestep,
$$
z^2 \mapsto \alpha X^2 P_{\mathrm{in}}[:,i] + \alpha H^1 P_{\mathrm{rec}}[:,i] = \alpha z^2,
$$
so $H^2$ is unchanged as well. Every spike across both timesteps is preserved, hence the network output is invariant to $\alpha$. This is why the two blocks must be scaled \emph{together}: scaling only $P_{\mathrm{in}}[:,i]$ would rescale the input term of $z^2$ but not the recurrent term $H^1 P_{\mathrm{rec}}[:,i]$, changing $H^2$ at the second timestep.

Because the output is invariant, the joint magnitude
$$
a_i=\big(\|P_{\mathrm{in}}[:,i]\|_p^p+\|P_{\mathrm{rec}}[:,i]\|_p^p\big)^{1/p}
$$
may be driven to zero along a minimizing sequence, exactly as the single incoming column is in the feedforward reduction of Theorem~\ref{thm:reduction_nonparallel}. The general case follows by applying this argument to every inner neuron of every subnetwork in the unrolled DAG.

\subsection{Recurrent bidual theorem}
\label{app:recurrent_bidual}

We prove Theorem~\ref{thm:recurrent_bi_dual} through the following explicit form of the dual and bidual.

\emph{Intuition.} After the recurrent reduction, the training problem has the same structure as the feedforward reduced problem, with the final hidden-layer feature map replaced by its unrolled counterpart at timestep $T$. The dual and bidual therefore follow from the feedforward bidual of Proposition~\ref{thm:feedforward_bi_dual}, with the same Slater and Carath\'eodory arguments. The prediction is read out only at the final timestep, so the atom set lies in $\mathbb{R}^n$ and the finite-width bound is unchanged.

\begin{theorem}[Bidual of recurrent parallel threshold networks]
\label{thm:recurrent_bi_dual_appendix}
Let $P^{L,T,K}$ denote the reduced weight-decay regularized recurrent training problem with convex loss $\mathcal L$ and regularization parameter $\beta_{\mathrm{reg}}>0$. Let
$\phi_{T,k}(X^{1:T};\theta_k)$ denote the final hidden-layer feature map of subnetwork $k$ at timestep $T$ after the recurrent reduction. Then the dual problem is
$$
\Lambda^{L,T,K}
=
\max_{\lambda}\ -\mathcal L^*(-\lambda)
\quad
\mathrm{s.t.}
\quad
\max_{\theta_k\in\bar\Theta_k}
\|\lambda^\top\phi_{T,k}(X^{1:T};\theta_k)\|_{p^*}
\le \beta_{\mathrm{reg}},
\qquad k\in[K].
$$
The corresponding infinite bidual is
$$
P_{\infty}^{L,T,K}
=
\min_{\nu}
\left\{
\mathcal L\!\left(
\int_{\theta\in\bar\Theta}
\phi_T(X^{1:T};\theta)\,d\nu(\theta),
Y
\right)
+
\beta_{\mathrm{reg}}\|\nu\|_{\mathrm{TV}}
\right\}.
$$
Since $\beta_{\mathrm{reg}}>0$, strong duality holds:
$$
\Lambda^{L,T,K}=P_{\infty}^{L,T,K}.
$$
Moreover, for scalar outputs, if $K\ge K^*$ for some $K^*\le n+1$, then
$$
P^{L,T,K}=\Lambda^{L,T,K}=P_{\infty}^{L,T,K}.
$$
For $d_{\mathrm{out}}$ output coordinates, the same argument gives
$K^*\le nd_{\mathrm{out}}+1$ after vectorizing the predictions.
\end{theorem}

\begin{proof}
Using Appendix~\ref{app:recurrent_theory}, the reduced recurrent training problem is
$$
P^{L,T,K}
=
\min_{\{\theta_k,w_k\}_{k=1}^K}
\mathcal L\!\left(
\sum_{k=1}^K
\phi_{T,k}(X^{1:T};\theta_k)w_k,
Y
\right)
+
\beta_{\mathrm{reg}}\sum_{k=1}^K\|w_k\|_p,
$$
where $\theta_k\in\bar\Theta_k$ denotes the reduced hidden recurrent parameters and $w_k=P_{k,\mathrm{out}}$ denotes the output weights.
Introduce an auxiliary prediction variable
$$
\hat Y
=
\sum_{k=1}^K
\phi_{T,k}(X^{1:T};\theta_k)w_k.
$$
The constrained form is
$$
\min_{\hat Y,\{\theta_k,w_k\}}
\mathcal L(\hat Y,Y)
+
\beta_{\mathrm{reg}}\sum_{k=1}^K\|w_k\|_p
\quad
\mathrm{s.t.}
\quad
\hat Y
=
\sum_{k=1}^K
\phi_{T,k}(X^{1:T};\theta_k)w_k.
$$
Let $\lambda$ be the Lagrange multiplier for this equality constraint. The Lagrangian is
$$
\mathcal L(\hat Y,Y)
+
\lambda^\top \hat Y
-
\lambda^\top
\sum_{k=1}^K
\phi_{T,k}(X^{1:T};\theta_k)w_k
+
\beta_{\mathrm{reg}}\sum_{k=1}^K\|w_k\|_p.
$$
Taking the infimum over $\hat Y$ gives the Fenchel conjugate defined in \S\ref{sec:notation_prelim},
$$
\inf_{\hat Y}
\left\{
\mathcal L(\hat Y,Y)+\lambda^\top\hat Y
\right\}
=
-\mathcal L^*(-\lambda).
$$
For each subnetwork $k$, the infimum over $w_k$ is
$$
\inf_{w_k}
\left\{
-\lambda^\top\phi_{T,k}(X^{1:T};\theta_k)w_k
+
\beta_{\mathrm{reg}}\|w_k\|_p
\right\}.
$$
By the conjugacy of the $\ell_p$ norm, this infimum is finite and equal to zero if and only if
$$
\|\lambda^\top\phi_{T,k}(X^{1:T};\theta_k)\|_{p^*}
\le \beta_{\mathrm{reg}}.
$$
Otherwise, it is $-\infty$. Since the hidden parameters are optimized, the constraint must hold uniformly over the reduced recurrent parameter set. Therefore, the dual is
$$
\Lambda^{L,T,K}
=
\max_{\lambda}\ -\mathcal L^*(-\lambda)
\quad
\mathrm{s.t.}
\quad
\max_{\theta_k\in\bar\Theta_k}
\|\lambda^\top\phi_{T,k}(X^{1:T};\theta_k)\|_{p^*}
\le \beta_{\mathrm{reg}},
\qquad k\in[K].
$$
Define the recurrent atom set
$$
\mathcal A_{\mathrm{RNN}}
=
\left\{
\phi_T(X^{1:T};\theta)c
:
\theta\in\bar\Theta,\ \|c\|_p\le 1
\right\}
\subseteq \mathbb{R}^n
$$
for scalar outputs. For vector-valued outputs, predictions are vectorized into $\mathbb{R}^{nd_{\mathrm{out}}}$. The dual constraint is equivalently
$$
\sup_{a\in\mathcal A_{\mathrm{RNN}}}\lambda^\top a\le \beta_{\mathrm{reg}}.
$$
The associated atomic gauge is
$$
\gamma_{\mathcal A_{\mathrm{RNN}}}(z)
=
\inf_{\nu}
\left\{
\|\nu\|_{\mathrm{TV}}
:
z=\int_{a\in\mathcal A_{\mathrm{RNN}}}a\,d\nu(a)
\right\}.
$$
Thus, the bidual is
$$
P_{\infty}^{L,T,K}
=
\min_z
\mathcal L(z,Y)
+
\beta_{\mathrm{reg}}\gamma_{\mathcal A_{\mathrm{RNN}}}(z),
$$
or equivalently,
$$
P_{\infty}^{L,T,K}
=
\min_{\nu}
\left\{
\mathcal L\!\left(
\int_{a\in\mathcal A_{\mathrm{RNN}}}a\,d\nu(a),
Y
\right)
+
\beta_{\mathrm{reg}}\|\nu\|_{\mathrm{TV}}
\right\}.
$$
Reparameterizing atoms as $a=\phi_T(X^{1:T};\theta)c$ gives the stated measure-valued form over recurrent hidden parameters.
Strong duality follows from Slater's condition. At $\lambda=0$,
$$
\max_{\theta\in\bar\Theta}
\|0^\top\phi_T(X^{1:T};\theta)\|_{p^*}
=
0
<
\beta_{\mathrm{reg}},
$$
because $\beta_{\mathrm{reg}}>0$. Hence the dual feasible set has nonempty relative interior, and
$$
\Lambda^{L,T,K}
=
P_{\infty}^{L,T,K}.
$$

Finally, since $L$ and $T$ are fixed and the training set has $n$ samples, every recurrent threshold activation pattern lies in $\{0,1\}^n$. Therefore, the recurrent feature generator is finite, and hence bounded, on the training set. For scalar outputs, the atom set lies in $\mathbb{R}^n$, so Carath\'eodory's theorem implies that any point in the convex hull of the atom set can be represented using at most $n+1$ atoms. Each atom is realizable by one parallel recurrent subnetwork. Hence, if $K\ge K^*$ for some $K^*\le n+1$, the finite $K$-parallel recurrent network can realize the same prediction vector as the infinite bidual. Therefore,
$$
P^{L,T,K}
=
P_{\infty}^{L,T,K}.
$$
Combining with strong duality gives
$$
P^{L,T,K}
=
\Lambda^{L,T,K}
=
P_{\infty}^{L,T,K}.
$$
For $d_{\mathrm{out}}$ output coordinates, the atom set lies in $\mathbb{R}^{nd_{\mathrm{out}}}$ after vectorization, giving the bound $K^*\le nd_{\mathrm{out}}+1$.
\end{proof}

\subsection{Witnessed recurrent arrangement construction}
\label{app:recurrent_arrangements}

This section details the enumeration that produces the finite spike dictionary of Theorem~\ref{thm:finite_convex_L_T}. The construction differs from the feedforward arrangement because recurrent weights are shared across timesteps: the activation column at a given layer and timestep cannot be chosen independently of the columns at earlier timesteps, since all are generated by the same recurrent parameters. We track this dependence by attaching to each column the recurrent weights that generated its trajectory.

Let $D^{l,t}$ denote the arrangement dictionary at layer $l$ and timestep $t$. In the recurrent setting, this dictionary is witnessed: each activation column is stored together with the shared recurrent weights that generated the full trajectory. For layer $l$, the witness is
$$
\omega_l=(P_{l,\mathrm{in}},P_{l,\mathrm{rec}}),
$$
the recurrent parameters, which are chosen once and reused across all timesteps. Thus, $D^{l,t}$ is not an unconstrained marginal set of activation patterns at $(l,t)$; it is a trajectory-consistent dictionary generated by shared recurrent weights.
The recurrent pre-activation can be written as a feedforward-style hyperplane over a stacked feature matrix:
$$
H_l^tP_{l+1,\mathrm{in}}
+
H_{l+1}^{t-1}P_{l+1,\mathrm{rec}}
=
\begin{bmatrix}
H_l^t & H_{l+1}^{t-1}
\end{bmatrix}
\begin{bmatrix}
P_{l+1,\mathrm{in}}\\
P_{l+1,\mathrm{rec}}
\end{bmatrix}.
$$
For the first recurrent layer, the stacked feature matrix is
$$
\begin{bmatrix}
X^t & H_1^{t-1}
\end{bmatrix}.
$$
Therefore, the feedforward arrangement operator $A(\cdot)$ extends to recurrent networks by applying it to these stacked feature matrices, while restricting recurrent columns to those generated by the same witness.
At $t=1$, the witness constraint is vacuous because no previous recurrent trajectory has been generated. The base construction is
$$
D^{1,1}
=
A\!\left(
\begin{bmatrix}
X^1 & H_1^0
\end{bmatrix}
\right).
$$
For $l\ge 2$ at the first timestep,
$$
D^{l,1}
=
\bigsqcup_{|\mathcal S_F|=m_{l-1}}
A\!\left(
\begin{bmatrix}
D^{l-1,1}_{\mathcal S_F} & H_l^0
\end{bmatrix}
\right).
$$
Here $\mathcal S_F$ ranges over column subsets of size $m_{l-1}$, and $\bigsqcup$ denotes union with duplicate activation columns removed.
For $t>1$, the recurrent input cannot be sampled freely from the marginal set $D^{l,t-1}$. It must come from the same shared-weight trajectory that generated the current forward column. Thus, for the first recurrent layer,
$$
D^{1,t}
=
\bigsqcup_{\omega}
A\!\left(
\begin{bmatrix}
X^t & D_{\omega}^{1,t-1}
\end{bmatrix}
\right),
$$
where $D_{\omega}^{1,t-1}$ denotes the recurrent activation columns at timestep $t-1$ generated by the same witness $\omega=(P_{1,\mathrm{in}},P_{1,\mathrm{rec}})$. For deeper layers,
$$
D^{l,t}
=
\bigsqcup_{\omega}
A\!\left(
\begin{bmatrix}
D_{\omega}^{l-1,t} & D_{\omega}^{l,t-1}
\end{bmatrix}
\right).
$$
The notation means that both the forward and recurrent columns are selected from dictionaries generated by the same shared-weight witness.
This witnessed construction enforces recurrent parameter sharing at the dictionary level. If recurrent columns were sampled freely from the marginal set $D^{l,t-1}$, the construction could combine activation patterns generated by incompatible recurrent parameters. Such combinations may not be realizable by any single recurrent network.
At every layer--timestep pair $(l,t)$, each activation column is a binary vector over the $n$ training samples. Hence
$$
D^{l,t}\subseteq \{0,1\}^n,
\qquad
|D^{l,t}|\le 2^n.
$$
Since $L$ and $T$ are fixed, the full witnessed collection
$$
\{D^{l,t}\}_{l\in[L-1],\,t\in[T]}
$$
is finite. The witness constraint can only reduce the admissible set of columns; it cannot increase the cardinality beyond $2^n$ at any fixed $(l,t)$.

\subsection{Proof of Theorem~\ref{thm:finite_convex_L_T}}
\label{app:finite_convex_LT}

The proof specializes the recurrent dual of Theorem~\ref{thm:recurrent_bi_dual_appendix} to $p=2$ and rewrites the dual constraint over the finite witnessed dictionary. The maximization over hidden parameters becomes a maximization over the finitely many dictionary columns, and dualizing the resulting finite constraint set through the Fenchel identity yields an $\ell_1$-penalized convex program. The $\sqrt{m_{L-1}}$ factor enters through the final change of variables that absorbs the layer width into the coefficients.

\begin{proof}
By Theorem~\ref{thm:recurrent_bi_dual_appendix}, the recurrent dual is
$$
\Lambda^{L,T}
=
\max_{\lambda}\ -\mathcal L^*(-\lambda)
\quad
\mathrm{s.t.}
\quad
\max_{\Theta\in\bar\Theta_{\mathrm{RNN}}}
\left\|
\lambda^\top
\sigma\!\left(
H_{L-2}^T P_{L-1,\mathrm{in}}
+
H_{L-1}^{T-1}P_{L-1,\mathrm{rec}}
\right)
\right\|_2
\le \beta_{\mathrm{reg}}.
$$
The witnessed arrangement construction in Appendix~\ref{app:recurrent_arrangements} ensures that every column of the final hidden-layer activation at $(L-1,T)$ is a column of
$$
D^{L-1,T}
=
[d_1,\ldots,d_{M^{L-1,T}}].
$$
Thus, for a width-$m_{L-1}$ final hidden layer,
$$
\left\|
\lambda^\top
\sigma\!\left(
H_{L-2}^T P_{L-1,\mathrm{in}}
+
H_{L-1}^{T-1}P_{L-1,\mathrm{rec}}
\right)
\right\|_2
=
\left(
\sum_{j=1}^{m_{L-1}}
(\lambda^\top d_j)^2
\right)^{1/2},
$$
where each $d_j$ is a column of $D^{L-1,T}$. Since the final-layer hidden neurons are independently parameterized, maximizing over the hidden parameters gives
$$
\max_{d_i\in D^{L-1,T}}
\sqrt{m_{L-1}}|\lambda^\top d_i|
\le \beta_{\mathrm{reg}}.
$$
This constraint is equivalent to the two-sided inequalities
$$
\sqrt{m_{L-1}}\lambda^\top d_i-\beta_{\mathrm{reg}}\le 0,
\qquad
-\sqrt{m_{L-1}}\lambda^\top d_i-\beta_{\mathrm{reg}}\le 0,
\qquad
i\in[M^{L-1,T}].
$$
At $\lambda=0$, all inequalities are strict because $\beta_{\mathrm{reg}}>0$, so Slater's condition holds.
Introduce nonnegative Lagrange multipliers $\gamma_i^+$ and $\gamma_i^-$ for the two inequalities. The Lagrangian of the dual problem is
$$
-\mathcal L^*(-\lambda)
-
\sum_i\gamma_i^+
(\sqrt{m_{L-1}}\lambda^\top d_i-\beta_{\mathrm{reg}})
-
\sum_i\gamma_i^-
(-\sqrt{m_{L-1}}\lambda^\top d_i-\beta_{\mathrm{reg}}).
$$
Collecting terms gives
$$
-\mathcal L^*(-\lambda)
-
\lambda^\top
\sqrt{m_{L-1}}
\sum_i(\gamma_i^+-\gamma_i^-)d_i
+
\beta_{\mathrm{reg}}\sum_i(\gamma_i^++\gamma_i^-).
$$
Using the Fenchel identity of Appendix~\ref{app:feedforward_duality} with
$$
v
=
-\sqrt{m_{L-1}}
\sum_i(\gamma_i^+-\gamma_i^-)d_i,
$$
we obtain
$$
\min_{\gamma^+,\gamma^-\ge 0}
\mathcal L\!\left(
-\sqrt{m_{L-1}}
\sum_i(\gamma_i^+-\gamma_i^-)d_i,
Y
\right)
+
\beta_{\mathrm{reg}}\sum_i(\gamma_i^++\gamma_i^-).
$$
Let $w_i=\gamma_i^+-\gamma_i^-$. Minimizing over nonnegative $\gamma_i^+,\gamma_i^-$ with fixed difference $w_i$ gives
$$
\gamma_i^++\gamma_i^-=|w_i|.
$$
Therefore,
$$
\Lambda^{L,T}
=
\min_{w\in\mathbb{R}^{M^{L-1,T}}}
\mathcal L(-\sqrt{m_{L-1}}D^{L-1,T}w,Y)
+
\beta_{\mathrm{reg}}\|w\|_1.
$$
Finally, set $\widetilde w=-\sqrt{m_{L-1}}w$. Then
$$
-\sqrt{m_{L-1}}D^{L-1,T}w
=
D^{L-1,T}\widetilde w,
\qquad
\|w\|_1
=
\frac{1}{\sqrt{m_{L-1}}}\|\widetilde w\|_1.
$$
Thus,
$$
\widetilde P^{L,T}
=
\min_{\widetilde w\in\mathbb{R}^{M^{L-1,T}}}
\mathcal L(D^{L-1,T}\widetilde w,Y)
+
\frac{\beta_{\mathrm{reg}}}{\sqrt{m_{L-1}}}\|\widetilde w\|_1.
$$
The equivalence to the original finite-width recurrent training problem follows from Theorem~\ref{thm:recurrent_bi_dual_appendix} whenever $K\ge K^*$.
\end{proof}

\section{LIF-SNN Theory}
\label{app:snn_theory}
\subsection{Typed LIF DAG and trainability settings}
\label{app:snn_dag}

The LIF-SNN is the recurrent threshold network of Appendix~\ref{app:recurrent_theory} with a structured recurrent matrix. To apply the recurrent reduction, we must identify which edges carry the removable positive scale and which are fixed by the LIF dynamics. We make this explicit by typing the edges of the unrolled LIF graph.

The fixed-horizon LIF-SNN can be represented as an unrolled DAG by introducing, for every layer $l$ and timestep $t$, a membrane node $U_l^t$ and a spike node $S_l^t$. The membrane node receives the three inputs appearing in the LIF recurrence:
$$
U_l^t
=
S_{l-1}^tP_{l,\mathrm{in}}
+
U_l^{t-1}B_l
-
S_l^{t-1}R_l,
\qquad
S_l^t=\sigma(U_l^t),
$$
where $B_l=\operatorname{Diag}(\beta_l)$, $\beta_l\in[0,1)^{m_l}$, and
$R_l=\operatorname{Diag}(U_{\mathrm{thr}}^l)$. Thus, the corresponding edges are
$$
S_{l-1}^t\to U_l^t
\quad\text{with weight }P_{l,\mathrm{in}},
\qquad
U_l^{t-1}\to U_l^t
\quad\text{with weight }B_l,
\qquad
S_l^{t-1}\to U_l^t
\quad\text{with weight }-R_l.
$$
The spike node is obtained by applying the threshold activation to the membrane node. Equivalently, the edge $U_l^t\to S_l^t$ has fixed weight $1$, and the nonlinearity at $S_l^t$ is $\sigma(\cdot)$.
We use two edge labels. The first label, $\tau:E\to\{0,1\}$, indicates whether an edge weight is optimized. Thus, $\tau(e)=1$ denotes a trainable edge, while $\tau(e)=0$ denotes a fixed structural edge. The second label, $\delta:E\to\{0,1\}$, indicates whether the edge is scale-bearing under the LIF membrane-rescaling symmetry. The input-to-membrane edges $P_{l,\mathrm{in}}$ and reset-threshold edges $-R_l$ are scale-bearing, while the leak edges $B_l$ are structural. The leak coefficients may be trainable or fixed depending on the configuration, but they do not carry the removable positive scale used in the weight-decay reduction.
In the structured-witness setting, at least one of $\beta_l$ or $U_{\mathrm{thr}}^l$ may be optimized while generating the LIF witness. If $U_{\mathrm{thr}}^l$ is trainable, it is a scale-bearing parameter scaled jointly with the input block. If $\beta_l$ is trainable, it remains a bounded structural recurrent coefficient and is not a scale-bearing trainable parameter. When $\beta_l$ is trainable it is excluded from the weight-decay regularizer. Otherwise its penalty is not removable and would remain in the reduced objective. In the fixed-structural setting, both $\beta_l$ and $U_{\mathrm{thr}}^l$ are fixed hyperparameters. They restrict the admissible LIF feature generator but are not optimized in the convex reconstruction stage.

\subsection{Proof of Lemma~\ref{lem:lif_membrane_rescaling}}
\label{app:lif_membrane_rescaling}

The lemma shows that positively rescaling the membrane coordinates of a layer, while correspondingly rescaling the input and reset weights, leaves every spike unchanged. This is the LIF form of the positive scaling invariance used in the feedforward reduction, and it is what allows the input and threshold magnitudes to be driven to zero while the leak, which is conjugated rather than scaled, is preserved.

\begin{proof}
Let
$$
A_l=\operatorname{Diag}(a_{l,1},\ldots,a_{l,m_l}),
\qquad
a_{l,i}>0.
$$
Define
$$
\bar U_l^t=U_l^tA_l,
\qquad
\bar P_{l,\mathrm{in}}=P_{l,\mathrm{in}}A_l,
\qquad
\bar R_l=R_lA_l,
\qquad
\bar B_l=A_l^{-1}B_lA_l.
$$
Starting from the LIF recurrence,
$$
U_l^t
=
S_{l-1}^tP_{l,\mathrm{in}}
+
U_l^{t-1}B_l
-
S_l^{t-1}R_l,
$$
right-multiply both sides by $A_l$:
$$
U_l^tA_l
=
S_{l-1}^tP_{l,\mathrm{in}}A_l
+
U_l^{t-1}B_lA_l
-
S_l^{t-1}R_lA_l.
$$
Using the definitions above,
$$
U_l^tA_l=\bar U_l^t,
\qquad
P_{l,\mathrm{in}}A_l=\bar P_{l,\mathrm{in}},
\qquad
R_lA_l=\bar R_l.
$$
For the recurrent membrane term,
$$
\bar U_l^{t-1}\bar B_l
=
U_l^{t-1}A_l(A_l^{-1}B_lA_l)
=
U_l^{t-1}B_lA_l.
$$
Therefore,
$$
\bar U_l^t
=
S_{l-1}^t\bar P_{l,\mathrm{in}}
+
\bar U_l^{t-1}\bar B_l
-
S_l^{t-1}\bar R_l.
$$
If $B_l$ is diagonal, then $A_l$ and $B_l$ commute, so
$$
\bar B_l=A_l^{-1}B_lA_l=B_l.
$$
Finally, since $A_l$ has strictly positive diagonal entries, each membrane coordinate is positively rescaled. The threshold activation is invariant under positive coordinate-wise scaling, so
$$
\sigma(\bar U_l^t)
=
\sigma(U_l^tA_l)
=
\sigma(U_l^t).
$$
Thus, positive diagonal rescaling of the membrane coordinates leaves all spike outputs unchanged.
\end{proof}

\subsection{Reduction of the LIF weight-decay regularizer}
\label{app:lif_path_reduction}

The LIF reduction is the recurrent reduction applied to the structured LIF graph. The trainable input and threshold parameters act only through the membrane and are removable by the rescaling of Lemma~\ref{lem:lif_membrane_rescaling} when $U^l_{\mathrm{thr}}$ is scale-bearing, while the leak is structural and is preserved.

The trainable parameters of a hidden LIF neuron $i$ in layer $l$ are $P_{l,\mathrm{in}}[:,i]$ and, when trainable, $U_{\mathrm{thr},i}^l$; both act only through $U_l^t$, which feeds $\sigma$. The leak $B_l$ is structural (Lemma~\ref{lem:lif_membrane_rescaling}). For $A_l\succ0$,
\[
\sigma(U_l^tA_l)=\sigma(U_l^t),
\]
so these magnitudes and the inner amplitudes are invisible to $f$, and we may drive
\[
\|P_{l,\mathrm{in}}[:,i]\|_p\to0\ \ (l\le L-1),
\qquad
|s^{(l)}_k|\to0\ \ (l\le L-2),
\]
along a minimizing sequence, which sends their $\Omega_{p,2}$ terms to $0$ while leaving $B_l$ untouched. The pair $(s_k,P_{k,\mathrm{out}})$, with $s_k$ the readout amplitude of subnetwork $k$, merges under the optimal rescaling with $|s_k|=1$ as in Theorem~\ref{thm:reduction_nonparallel}, giving
\[
\inf_{\Theta}\ \mathcal L+\frac{\beta_{\mathrm{reg}}}{2}\Omega_{p,2}(\Theta)^2
=
\min_{\substack{\Theta\\ |s_k|=1}}\mathcal L+\beta_{\mathrm{reg}}\sum_{k=1}^K\|P_{k,\mathrm{out}}\|_p.
\]

\subsection{Witnessed SNN arrangement construction and finiteness}
\label{app:snn_finiteness}

\paragraph{Proof of Theorem~\ref{thm:finite_convex_lif_snn}.} The proof is that of Theorem~\ref{thm:finite_convex_L_T} with the LIF witnessed dictionary, using the reduction of Appendix~\ref{app:lif_path_reduction} and the finiteness shown below.

Unlike the pure threshold case, the LIF recurrence carries a real-valued membrane state between timesteps. We show that, for a fixed witness, this membrane state still takes only finitely many values on the training set, so the SNN dictionary remains finite.

The LIF-SNN reduction introduces a real-valued membrane state in addition to the binary spike state. Therefore, the SNN arrangement dictionary must track both spike patterns and the membrane values that generated them.
For a fixed LIF witness
$$
\omega_l=(B_l,R_l,P_{l,\mathrm{in}},U_l^0),
$$
the LIF recurrence unrolls as
$$
U_l^t
=
U_l^0B_l^t
+
\sum_{\tau=1}^{t}
\left(
S_{l-1}^{\tau}P_{l,\mathrm{in}}
-
S_l^{\tau-1}R_l
\right)
B_l^{t-\tau}.
$$
Thus, under a fixed witness, $U_l^t$ is a deterministic function of the previous-layer spike trajectory
$$
(S_{l-1}^1,\ldots,S_{l-1}^t)
$$
and the same-layer recurrent spike trajectory
$$
(S_l^0,\ldots,S_l^{t-1}).
$$
\emph{Intuition.} The membrane is real-valued, but for a fixed witness it is a deterministic function of finitely many binary spike trajectories, so it takes only finitely many values.

\begin{lemma}[Finiteness of realizable membrane states for a fixed witness]
\label{lem:snn_finiteness}
Fix a finite training set of $n$ samples, horizon $T$, and LIF parameter witness
$\omega_l=(B_l,R_l,P_{l,\mathrm{in}},U_l^0)$. Then the set of membrane states realizable at layer--timestep pair $(l,t)$ from a finite collection of spike trajectories is finite. More precisely,
$$
\left|
\{U_l^t \text{ realizable under } \omega_l\}
\right|
\le
\prod_{\tau=1}^{t}|D_{\mathrm{SNN}}^{l-1,\tau}|
\cdot
\prod_{\tau=0}^{t-1}|D_{\mathrm{SNN}}^{l,\tau}|.
$$
\end{lemma}
\begin{proof}
The tuple
$$
(S_{l-1}^1,\ldots,S_{l-1}^t,S_l^0,\ldots,S_l^{t-1})
$$
ranges over a finite Cartesian product of spike-arrangement sets. For a fixed witness $\omega_l$, the unrolled recurrence maps each such tuple deterministically to a membrane state $U_l^t$. Therefore, the number of realizable membrane states is bounded by the product above.
\end{proof}

Across different LIF witnesses, the real-valued membrane values may vary. However, the projected spike dictionary remains finite after deduplication, because every spike column lies in $\{0,1\}^n$. Hence
$$
D_{\mathrm{SNN}}^{l,t}\subseteq\{0,1\}^n,
\qquad
|D_{\mathrm{SNN}}^{l,t}|\le 2^n.
$$
We now define the witnessed SNN arrangement operator. At layer--timestep pair $(l,t)$, the SNN pre-activation can be written in stacked form:
$$
U_l^t
=
\begin{bmatrix}
S_{l-1}^t & U_l^{t-1} & S_l^{t-1}
\end{bmatrix}
\begin{bmatrix}
P_{l,\mathrm{in}}\\
B_l\\
-R_l
\end{bmatrix}.
$$
The structured arrangement operator $A_{\mathrm{SNN}}(\cdot)$ applies the usual hyperplane-arrangement construction to the stacked state
$$
\begin{bmatrix}
S_{l-1}^t & U_l^{t-1} & S_l^{t-1}
\end{bmatrix},
$$
but only over LIF-structured parameter blocks
$$
\begin{bmatrix}
P_{l,\mathrm{in}}\\
B_l\\
-R_l
\end{bmatrix},
$$
where $B_l$ and $R_l$ are diagonal and $B_l$ has entries in $[0,1]$.
At each cell $(l,t)$, the forward input is drawn from the spike dictionary $D_{\mathrm{SNN}}^{l-1,t}$. The recurrent input is drawn from the witnessed joint dictionary at $(l,t-1)$, whose entries are membrane--spike pairs
$$
(U_l^{t-1},S_l^{t-1})
$$
generated by the same trajectory witness. Therefore, the recurrent subset is not an arbitrary Cartesian product of membrane and spike columns; choosing a recurrent spike column automatically chooses the membrane column associated with it under the same witness.
With this convention, the recursive construction is
$$
D_{\mathrm{SNN}}^{l,t}
:=
\bigsqcup_{\substack{
|\mathcal S_F|=m_{l-1},\ |\mathcal S_R|=m_l\\
\text{witness-consistent}
}}
A_{\mathrm{SNN}}\!\left(
\begin{bmatrix}
D_{\mathrm{SNN},\mathcal S_F}^{l-1,t}
&
U_{l,\mathcal S_R}^{t-1}
&
D_{\mathrm{SNN},\mathcal S_R}^{l,t-1}
\end{bmatrix}
\right),
$$
where $U_{l,\mathcal S_R}^{t-1}$ denotes the membrane columns associated with the selected recurrent spike columns under the same trajectory witness. The union $\bigsqcup$ removes duplicate spike columns after projection.
The construction is well-founded by diagonal induction on $l+t$, since each $D_{\mathrm{SNN}}^{l,t}$ depends only on the previous-layer dictionary $D_{\mathrm{SNN}}^{l-1,t}$ and the same-layer witnessed recurrent dictionary $D_{\mathrm{SNN}}^{l,t-1}$. Because every projected spike column lies in $\{0,1\}^n$, the complete projected SNN dictionary is finite for fixed $L$ and $T$.

\subsection{Proof of Per-Timestep Loss}
\label{app:snn_per_timestep_dual}

The main theorems supervise only the final timestep, so the prediction lies in $\mathbb{R}^n$ and the Carath\'eodory bound is $n+1$. Here we treat per-timestep supervision with a shared output layer, where the prediction is stacked across all $T$ timesteps. The derivation follows Theorem~\ref{thm:finite_convex_L_T}, with a dual block per timestep and a single shared output coefficient coupling the timestep losses; the stacked prediction raises the finite-width bound to $nT+1$.

We prove the finite convex formulation for per-timestep supervision with a shared output layer. The reduced training problem is
$$
P_{\mathrm{all}\text{-}t}^{L,T,K}
=
\min_{\Theta\in\bar\Theta_{\mathrm{SNN}}}
\sum_{t=1}^T
\mathcal L\!\left(
\sum_{k=1}^K
\sigma(U_{L-1,k}^t)P_{k,\mathrm{out}},
Y^t
\right)
+
\beta_{\mathrm{reg}}\sum_{k=1}^K\|P_{k,\mathrm{out}}\|_2 .
$$
The proof follows the finite-convex derivation of Theorem~\ref{thm:finite_convex_lif_snn} (Appendix~\ref{app:finite_convex_LT}), with two changes. First, the dual variable has one block $\lambda_t\in\mathbb{R}^n$ per timestep. Second, the shared output layer couples all timestep losses through a single coefficient vector.
Introduce auxiliary variables
$$
\widehat Y^t
=
\sum_{k=1}^K
\sigma(U_{L-1,k}^t)P_{k,\mathrm{out}},
\qquad
t\in[T],
$$
with multipliers $\lambda_t\in\mathbb{R}^n$. For each subnetwork $k$, the terms involving $P_{k,\mathrm{out}}$ are
$$
-
\left(
\sum_{t=1}^T
\lambda_t^\top\sigma(U_{L-1,k}^t)
\right)P_{k,\mathrm{out}}
+
\beta_{\mathrm{reg}}\|P_{k,\mathrm{out}}\|_2.
$$
Taking the infimum over $P_{k,\mathrm{out}}$ gives zero if
$$
\left\|
\sum_{t=1}^T
\lambda_t^\top\sigma(U_{L-1,k}^t)
\right\|_2
\le \beta_{\mathrm{reg}},
$$
and gives $-\infty$ otherwise. Hence the dual constraint is
$$
\max_{\Theta\in\bar\Theta_{\mathrm{SNN}}}
\left\|
\sum_{t=1}^T
\lambda_t^\top\sigma(U_{L-1,k}^t)
\right\|_2
\le \beta_{\mathrm{reg}}.
$$
Under shared recurrent parameters, the columns selected across timesteps must form a trajectory-consistent tuple. Let $\mathcal T_{L-1}$ denote the set of trajectory-consistent final-layer tuples, and write
$$
D^{\mathrm{traj}}
=
\begin{bmatrix}
D_{\mathcal T}^{L-1,1}\\
\vdots\\
D_{\mathcal T}^{L-1,T}
\end{bmatrix}
\in
\{0,1\}^{nT\times|\mathcal T_{L-1}|}.
$$
Then the dual constraint reduces to
$$
\max_{i\in[|\mathcal T_{L-1}|]}
\sqrt{m_{L-1}}
\left|
\sum_{t=1}^T
\lambda_t^\top D_{\mathcal T}^{L-1,t}[:,i]
\right|
\le \beta_{\mathrm{reg}}.
$$
Equivalently, for each $i$,
$$
\sqrt{m_{L-1}}
\sum_{t=1}^T
\lambda_t^\top D_{\mathcal T}^{L-1,t}[:,i]
-\beta_{\mathrm{reg}}
\le 0,
$$
and
$$
-\sqrt{m_{L-1}}
\sum_{t=1}^T
\lambda_t^\top D_{\mathcal T}^{L-1,t}[:,i]
-\beta_{\mathrm{reg}}
\le 0.
$$
At $\lambda_1=\cdots=\lambda_T=0$, all inequalities are strict because $\beta_{\mathrm{reg}}>0$, so Slater's condition holds.
Introduce nonnegative multipliers $\gamma_i^+$ and $\gamma_i^-$ for the two inequalities. The Lagrangian terms involving $\lambda_t$ become
$$
-\sum_{t=1}^T\mathcal L^*(-\lambda_t)
-
\sqrt{m_{L-1}}
\sum_{t=1}^T
\lambda_t^\top
\sum_i(\gamma_i^+-\gamma_i^-)
D_{\mathcal T}^{L-1,t}[:,i]
+
\beta_{\mathrm{reg}}\sum_i(\gamma_i^++\gamma_i^-).
$$
Define
$$
v_t
=
-\sqrt{m_{L-1}}
D_{\mathcal T}^{L-1,t}
(\gamma^+-\gamma^-).
$$
Using the Fenchel identity of Appendix~\ref{app:feedforward_duality} for each timestep gives the primal equivalent
$$
\min_{\gamma^+,\gamma^-\ge 0}
\sum_{t=1}^T
\mathcal L\!\left(
-\sqrt{m_{L-1}}
D_{\mathcal T}^{L-1,t}
(\gamma^+-\gamma^-),
Y^t
\right)
+
\beta_{\mathrm{reg}}\sum_i(\gamma_i^++\gamma_i^-).
$$
Let $w=\gamma^+-\gamma^-$. For fixed $w$, minimizing over nonnegative $\gamma^+,\gamma^-$ gives
$$
\sum_i(\gamma_i^++\gamma_i^-)=\|w\|_1.
$$
Therefore,
$$
\min_{w\in\mathbb{R}^{|\mathcal T_{L-1}|}}
\sum_{t=1}^T
\mathcal L\!\left(
-\sqrt{m_{L-1}}
D_{\mathcal T}^{L-1,t}w,
Y^t
\right)
+
\beta_{\mathrm{reg}}\|w\|_1.
$$
Finally, set
$$
\widetilde w=-\sqrt{m_{L-1}}\,w.
$$
Then
$$
-\sqrt{m_{L-1}}D_{\mathcal T}^{L-1,t}w
=
D_{\mathcal T}^{L-1,t}\widetilde w,
\qquad
\|w\|_1
=
\frac{1}{\sqrt{m_{L-1}}}\|\widetilde w\|_1.
$$
Thus,
$$
\widetilde P_{\mathrm{all}\text{-}t}^{L,T}
=
\min_{\widetilde w\in\mathbb{R}^{|\mathcal T_{L-1}|}}
\sum_{t=1}^T
\mathcal L(D_{\mathcal T}^{L-1,t}\widetilde w,Y^t)
+
\frac{\beta_{\mathrm{reg}}}{\sqrt{m_{L-1}}}\|\widetilde w\|_1.
$$
The generator now lies in $\mathbb{R}^{nT}$ because predictions are stacked across all $T$ timesteps. Therefore, the Carath\'eodory bound becomes $K^*\le nT+1$ for scalar outputs.

\providecommand{\R}{\mathbb{R}}
\subsection{Completeness of the witnessed dictionaries}
\label{app:complete}

This subsection identifies the layer, or time step, from which the witnessed dictionary is complete, that is equal to $\{0,1\}^n$. On a complete dictionary the convex program of Theorems~\ref{thm:finite_convex_L_T} and~\ref{thm:finite_convex_lif_snn} is posed over all $2^n$ patterns, and its interpolation value has the closed form $p^\star=\max_i(y_i)_++\max_i(-y_i)_+$ (Lemma~\ref{lem:ap_cube}). We write $\bar s:=2s-\mathbf 1$ for $s\in\{0,1\}^n$, so $\bar s_i=1$ if $s_i=1$ and $\bar s_i=-1$ if $s_i=0$.

\paragraph{Two facts.}
\begin{itemize}
  \item[(F2)] \emph{Margin.} If $c>0$ and $e\in\R^n$ with $|e_i|<c$ for all $i$, then $c\bar s+e$ has no zero entry and $\ind{c\bar s+e\ge0}=s$. Indeed, if $s_i=1$ then $c+e_i\ge c-|e_i|>0$, and if $s_i=0$ then $-c+e_i\le-c+|e_i|<0$.
  \item[(F3)] \emph{Sign preservation.} If $M,E$ have the same size and $|E_{ij}|<|M_{ij}|$ for all $i,j$, then $M+E$ has no zero entry and $\ind{M+E\ge0}=\ind{M\ge0}$. Indeed, $M_{ij}=0$ is excluded, and $M_{ij}+E_{ij}$ has the sign of $M_{ij}$ because $|E_{ij}|<|M_{ij}|$.
\end{itemize}

\emph{Intuition.} To build any pattern at a later layer we first need patterns that differ in exactly one point. A hyperplane through one input point that misses all the others can be tilted to flip that point alone, and non-collinearity is what makes such a hyperplane exist.

\begin{lemma}[flip]\label{lem:ap_flip}
Let $Z\in\R^{n\times q}$, $q\ge2$, have non-zero, pairwise non-collinear rows $z_1,\dots,z_n$. For every $i$ there are $u^{(i)}_\pm\in\R^q$ such that no entry of $Zu^{(i)}_\pm$ is zero and $\ind{Zu^{(i)}_+\ge0}-\ind{Zu^{(i)}_-\ge0}=\mathbf e_i$. Consequently there is $V\in\R^{q\times n}$ with $\ind{ZV\ge0}$ invertible and $\min_{ij}|(ZV)_{ij}|>0$.
\end{lemma}
\begin{proof}
Multiplying a row by a positive number changes no sign, so we may assume $\|z_k\|=1$ for all $k$.

\emph{Step 1: a direction orthogonal to $z_i$ only.} Fix $i$ and let $H_i:=\{b\in\R^q:\ b^\top z_i=0\}$, a subspace of dimension $q-1\ge1$. For $k\ne i$, the linear form $b\mapsto b^\top z_k$ is not identically zero on $H_i$, since otherwise $z_k\in H_i^\perp=\mathrm{span}\{z_i\}$, which non-collinearity excludes. A product of finitely many non-zero polynomials on $H_i$ is non-zero, so $\prod_{k\ne i}b^\top z_k$ does not vanish identically on $H_i$. Hence there is $b\in H_i$ with $b^\top z_i=0$ and $b^\top z_k\ne0$ for all $k\ne i$.

\emph{Step 2: tilt.} Choose $0<\tau<\min_{k\ne i}|b^\top z_k|$ and set $u^{(i)}_\pm:=b\pm\tau z_i$. Then $z_i^\top u^{(i)}_\pm=\pm\tau$. For $k\ne i$, $z_k^\top u^{(i)}_\pm=b^\top z_k\pm\tau z_k^\top z_i$ with $|\tau z_k^\top z_i|\le\tau<|b^\top z_k|$, so by (F3) $z_k^\top u^{(i)}_\pm$ is non-zero with the sign of $b^\top z_k$. Hence no entry of $Zu^{(i)}_\pm$ is zero, the two patterns agree off coordinate $i$, and their difference is $\mathbf e_i$.

\emph{Step 3: $n$ independent patterns.} The $2n$ patterns $\ind{Zu^{(i)}_\pm\ge0}$, $i\le n$, span $\R^n$, because their differences are $\mathbf e_1,\dots,\mathbf e_n$. A spanning set contains a basis. Let the columns of $V$ be the $n$ directions $u^{(i)}_\pm$ that give a basis. Then $\ind{ZV\ge0}$ is invertible, and $\min_{ij}|(ZV)_{ij}|>0$ by Step 2.
\end{proof}

\subsubsection{Feedforward threshold networks}
\label{app:complete_ff}

For a threshold network of depth $L$, with hidden layers $1,\dots,L-1$ as in \S\ref{sec:notation_prelim}, we assume the following. (G1) $X\in\R^{n\times d}$, $d\ge2$, has non-zero, pairwise non-collinear rows: $x_k\notin\mathrm{span}\{x_i\}$ for $i\ne k$. (G2) Layer $\ell=1,\dots,L-1$ has $m_\ell$ neurons and no bias, $A_0=X$, $A_\ell=\ind{A_{\ell-1}G_\ell\ge0}\in\{0,1\}^{n\times m_\ell}$, where $G_\ell\in\R^{m_{\ell-1}\times m_\ell}$ has i.i.d.\ $N(0,1)$ entries and $G_1,\dots,G_{L-1}$ are independent. The intermediate widths satisfy $m_\ell\ge n$ for $1\le\ell\le L-2$. We write $m=m_{L-1}$ and $A=A_{L-1}$. The dictionary of layer $1$ is $\mathcal D^{(1)}:=\{\ind{Xu\ge0}:u\in\R^d\}$. For $\ell\ge2$, a matrix $B\in\{0,1\}^{n\times m_{\ell-1}}$ is a \emph{layer-$(\ell-1)$ output} if its columns lie in $\mathcal D^{(\ell-1)}$, and $\mathcal D^{(\ell)}:=\{\ind{Bg\ge0}:\ g\in\R^{m_{\ell-1}},\ B\text{ a layer-}(\ell-1)\text{ output}\}$.

\emph{Intuition.} If layer $1$ outputs a matrix of rank $n$, a neuron of layer $2$ can make its pre-activation any vector in $\R^n$, in particular $+1$ where it should fire and $-1$ elsewhere. The flip lemma provides the rank-$n$ layer, and once one layer is complete every later layer is too.

\begin{theorem}[every pattern from layer 2 on]\label{thm:ap_full}
Under (G1)--(G2), $\mathcal D^{(\ell)}=\{0,1\}^n$ for every $2\le\ell\le L-1$.
\end{theorem}
\begin{proof}
\emph{Step 1: a rank-$n$ layer-1 output.} By Lemma~\ref{lem:ap_flip} applied to $X$ there is $V\in\R^{d\times n}$ with $\ind{XV\ge0}$ invertible. Its $n$ columns lie in $\mathcal D^{(1)}$. Since $m_1\ge n$, put them as $n$ columns of $B$ and fill the remaining columns with any elements of $\mathcal D^{(1)}$. Then $B$ is a layer-1 output with $\mathrm{rank}\,B=n$.

\emph{Step 2: rank $n$ gives every pattern.} If $\mathrm{rank}\,B=n$, the map $g\mapsto Bg$ is onto $\R^n$. For $s\in\{0,1\}^n$ choose $g$ with $Bg=\bar s$. By (F2) with $c=1$ and $e=0$, $\ind{Bg\ge0}=s$. With Step 1 this gives $\mathcal D^{(2)}=\{0,1\}^n$.

\emph{Step 3: induction.} If $\mathcal D^{(\ell-1)}=\{0,1\}^n$, the matrix $B=[\mathbf e_1|\cdots|\mathbf e_n|\ast]$, with any patterns in the remaining $m_{\ell-1}-n\ge0$ columns, is a layer-$(\ell-1)$ output of rank $n$. Step 2 gives $\mathcal D^{(\ell)}=\{0,1\}^n$.
\end{proof}

\subsubsection{Threshold RNN with shared recurrent weights}
\label{app:complete_rnn}

Consider one subnetwork of \S\ref{sec:rnn_section} with one hidden layer of width $m$, inputs $X^1,\dots,X^T\in\R^{n\times d}$, initial state $H^0\in\R^{n\times m}$, and weights $W_{\rm in}=P_{1,\mathrm{in}}\in\R^{d\times m}$ and $W_{\rm rec}=P_{1,\mathrm{rec}}\in\R^{m\times m}$ shared over time:
\begin{equation}\label{eq:ap_rnn}
  H^t=\ind{X^tW_{\rm in}+H^{t-1}W_{\rm rec}\ge0}\in\{0,1\}^{n\times m},\qquad t=1,\dots,T .
\end{equation}
The dictionary of step $t$ is $\mathcal D^{(t)}:=\{H^t[:,j]:\ W_{\rm in},W_{\rm rec}\text{ arbitrary},\ j\le m\}$. For a matrix $M$ let $\|M\|_\infty:=\max_i\sum_j|M_{ij}|$ and $\| \Delta \|_{\max} := \max_{i,j} \vert\Delta_{i,j}\vert$, so that $\|M\Delta\|_{\max}\le\|M\|_\infty\|\Delta\|_{\max}$. Let $C_0:=\max\{m,\|H^0\|_\infty\}$, which equals $m$ when $H^0$ has entries in $[-1,1]$. Since $H^t$ is a $0/1$ matrix for $t\ge1$, $\|H^{t-1}\|_\infty\le C_0$ for every $t\ge1$. The assumptions are (R1): for every $t$, $x^t_i\ne0$ and the rows of $X^t$ are pairwise non-collinear, with $d\ge2$; and (R2): $T\ge2$ and $m\ge n+1$.

\paragraph{Blocks.} Split the neurons into a set $S$ of $n$ neurons, one output neuron $o$, and the remaining $m-n-1\ge0$ neurons $R$. Write $H^t=[H^t_S|H^t_o|H^t_R]$, $W_{\rm in}=[V|w|V_R]$, and let $\Delta:=W_{\rm rec}[:,S]\in\R^{m\times n}$ collect everything $S$ reads from the state. Let $z:=W_{\rm rec}[S,o]\in\R^n$ and $a:=W_{\rm rec}[o,o]$ be what $o$ reads from $S$ and from itself, and set $W_{\rm rec}[R,o]=0$. By block multiplication of \eqref{eq:ap_rnn},
\begin{equation}\label{eq:ap_rnn_blocks}
  H^t_S=\ind{X^tV+H^{t-1}\Delta\ge0},\qquad H^t_o=\ind{X^tw+H^{t-1}_Sz+a\,H^{t-1}_o\ge0} .
\end{equation}

\emph{Intuition.} The recurrent term is bounded, because the states it multiplies are binary. If the recurrent weights into the first $n$ neurons are small, they cannot change the signs set by the input, so these neurons behave like a feedforward layer at every step.

\begin{lemma}[carry-forward and sign preservation]\label{lem:ap_rnn_carry}
Let $\|\Delta\|_{\max}\le\eta_0$ and $\gamma_t:=\min_{ij}|(X^tV)_{ij}|$. For every $t\le T$, $\|H^{t-1}\Delta\|_{\max}\le C_0\eta_0$, whatever the other weights and the earlier states. If $\gamma_t>C_0\eta_0$, then $H^t_S=\ind{X^tV\ge0}$.
\end{lemma}
\begin{proof}
$\|H^{t-1}\Delta\|_{\max}\le\|H^{t-1}\|_\infty\|\Delta\|_{\max}\le C_0\eta_0$. The bound depends on the past only through $\|H^{t-1}\|_\infty\le C_0$, which holds because the threshold resets every state to $0/1$ before it is carried. Nothing accumulates over the steps. If $\gamma_t>C_0\eta_0$, then $|(H^{t-1}\Delta)_{ij}|<|(X^tV)_{ij}|$ for all $i,j$, and (F3) applied to \eqref{eq:ap_rnn_blocks} gives $H^t_S=\ind{X^tV\ge0}$.
\end{proof}

\emph{Intuition.} At step $T-1$ the first $n$ neurons hold a rank-$n$ state by the flip lemma. At step $T$ the output neuron reads this state through its recurrent weights, which then act as a free bias in $\R^n$, and a free bias can produce any pattern.

\begin{theorem}[completeness at step $T$, width $n+1$]\label{thm:ap_rnn_complete}
Under (R1) at $t=T-1$ and (R2), $\mathcal D^{(T)}=\{0,1\}^n$.
\end{theorem}
\begin{proof}
\emph{Choice of $V$.} By Lemma~\ref{lem:ap_flip} applied to $X^{T-1}$ there is $V\in\R^{d\times n}$ with $Q:=\ind{X^{T-1}V\ge0}$ invertible and $\gamma_{T-1}>0$. Fix $0\le\eta_0<\gamma_{T-1}/C_0$, any $\Delta$ with $\|\Delta\|_{\max}\le\eta_0$, any $a\in\R$, $c>|a|$ and $w\in\R^d$, and arbitrary $V_R$ and remaining recurrent blocks. For $s\in\{0,1\}^n$ set $z_s:=Q^{-1}(c\bar s-X^Tw)$, which exists because $Q$ is invertible.

\emph{Step $T-1$.} By Lemma~\ref{lem:ap_rnn_carry} at $t=T-1$, $H^{T-1}_S=Q$. This holds for every choice of $z$, since the lemma assumes nothing about the blocks other than $V$ and $\Delta$.

\emph{Step $T$.} By \eqref{eq:ap_rnn_blocks} with $z=z_s$ and $H^{T-1}_S=Q$,
\[
  H^T_o=\ind{X^Tw+QQ^{-1}(c\bar s-X^Tw)+aH^{T-1}_o\ge0}=\ind{c\bar s+aH^{T-1}_o\ge0} .
\]
Since $T-1\ge1$, $H^{T-1}_o\in\{0,1\}^n$, so $e:=aH^{T-1}_o$ has $|e_i|\le|a|<c$. By (F2), $H^T_o=s$.

\emph{All patterns.} With the other weights fixed, the $2^n$ choices $z=z_s$ give $H^T_o=s$ for every $s\in\{0,1\}^n$. So $\{0,1\}^n\subseteq\mathcal D^{(T)}$, and $\mathcal D^{(T)}\subseteq\{0,1\}^n$ because every state is $0/1$.
\end{proof}

The construction with $\eta_0=0$ is allowed. The convention $\ind{0\ge0}=1$ never decides an entry used in the proof, since $X^{T-1}V+H^{T-2}\Delta$ and $c\bar s+aH^{T-1}_o$ have no zero entries by (F3) and (F2).

\subsubsection{LIF-SNN}
\label{app:complete_lif}

We use the LIF recurrence of Appendix~\ref{app:snn_dag} with scalar leak $\beta\in[0,1)$ and reset $R=U_{\mathrm{thr}}\ge0$, the diagonal entry of $R_l=\operatorname{Diag}(U^l_{\mathrm{thr}})$. The results hold entrywise for diagonal $B_l,R_l$ with $\beta:=\max\beta_l$ and $R:=\max R_l$. With $S^t_0:=X^t$ and $l=1,\dots,L-1$, the hidden layers,
\begin{equation}\label{eq:ap_lif}
  U^t_l=S^t_{l-1}P_l+\beta U^{t-1}_l-R\,S^{t-1}_l,\qquad S^t_l=\ind{U^t_l\ge0},\qquad U^0_l=0,\ S^0_l=0 .
\end{equation}
Iterating \eqref{eq:ap_lif} from $U^0_l=0$,
\begin{equation}\label{eq:ap_unroll}
  U^t_l=\sum_{\tau=1}^t\beta^{t-\tau}\big(S^\tau_{l-1}P_l-RS^{\tau-1}_l\big)=Z^t_lP_l-R\sum_{\tau=2}^t\beta^{t-\tau}S^{\tau-1}_l,\qquad Z^t_l:=\sum_{\tau=1}^t\beta^{t-\tau}S^\tau_{l-1},
\end{equation}
the leaky-integrated input of layer $l$. For $l=1$ its rows are $z^t_i:=\sum_{\tau\le t}\beta^{t-\tau}x^\tau_i$. The dictionary of layer $l$ is $\mathcal D^{(l)}:=\{S^T_l[:,j]\}$ over all weights.

\paragraph{Reset bound and assumptions.} (F1) For all $l$, $t$ and weights, entrywise $0\le R\sum_{\tau=2}^t\beta^{t-\tau}S^{\tau-1}_l\le R\sum_{k\ge0}\beta^k=\bar\rho:=\frac R{1-\beta}$, since every spike is $0$ or $1$. (L1) $d\ge2$; $z^t_i\ne0$ for all $t,i$; the rows of $Z^T_1$ are pairwise non-collinear; and $z^\tau_k\notin\mathrm{span}\{z^T_i\}$ for all $\tau<T$ and all $i,k$. (L2) $L\ge3$, so that there are at least two hidden layers, and $m_1\ge n$. (L1) holds for Lebesgue-almost every input.

\emph{Intuition.} The reset term of an LIF neuron is bounded by $\bar\rho$ whatever the earlier spikes were. If the input weights are scaled up enough, the input dominates the reset, and the neuron spikes exactly when its integrated input is positive.

\begin{lemma}[large-scale limit]\label{lem:ap_lif_scale}
Fix $l\ge1$, the trajectories of layer $l-1$, and $V$ with $\gamma:=\min_{t\le T}\min_{i,j}|(Z^t_lV)_{ij}|>0$. If $P_l=cV$ with $c>\bar\rho/\gamma$, then $S^t_l=\ind{Z^t_lV\ge0}$ for every $t\le T$.
\end{lemma}
\begin{proof}
By \eqref{eq:ap_unroll}, $U^t_l=cZ^t_lV-E^t$, where $E^t$ is the reset term. By (F1), $0\le E^t_{ij}\le\bar\rho<c\gamma\le|c(Z^t_lV)_{ij}|$. By (F3), $\ind{U^t_l\ge0}=\ind{cZ^t_lV\ge0}=\ind{Z^t_lV\ge0}$. The bound on $E^t$ does not depend on the earlier spikes, so this holds at every $t$ directly.
\end{proof}

\emph{Intuition.} This is the flip lemma applied to the leaky-integrated inputs, with one more requirement. The flip must change the spike at the final step only, so that the earlier steps, which also enter the next layer, stay the same.

\begin{lemma}[flip on the integrated inputs]\label{lem:ap_lif_flip}
Under (L1), for every $i$ there are $u^{(i)}_\pm\in\R^d$ such that no entry of $Z^\tau_1u^{(i)}_\pm$ is zero for any $\tau\le T$, $\ind{Z^\tau_1u^{(i)}_+\ge0}=\ind{Z^\tau_1u^{(i)}_-\ge0}$ for $\tau<T$, and $\ind{Z^T_1u^{(i)}_+\ge0}-\ind{Z^T_1u^{(i)}_-\ge0}=\mathbf e_i$.
\end{lemma}
\begin{proof}
Normalise every $z^\tau_k$ to unit norm, which changes no sign. Let $H_i:=\{b:\ b^\top z^T_i=0\}$, of dimension $d-1\ge1$. For $(\tau,k)\ne(T,i)$ the form $b\mapsto b^\top z^\tau_k$ is not identically zero on $H_i$, since otherwise $z^\tau_k\in\mathrm{span}\{z^T_i\}$, which (L1) excludes: by non-collinearity for $\tau=T$, $k\ne i$, and by the last clause for $\tau<T$. As in Lemma~\ref{lem:ap_flip} there is $b\in H_i$ with $b^\top z^\tau_k\ne0$ for all $(\tau,k)\ne(T,i)$. Take $0<\eta<\min_{(\tau,k)\ne(T,i)}|b^\top z^\tau_k|$ and $u^{(i)}_\pm:=b\pm\eta z^T_i$. Then $z^{T\top}_iu^{(i)}_\pm=\pm\eta$. For $(\tau,k)\ne(T,i)$ the perturbation $\pm\eta\,z^{\tau\top}_kz^T_i$ has modulus at most $\eta<|b^\top z^\tau_k|$, so by (F3) $z^{\tau\top}_ku^{(i)}_\pm$ is non-zero with the sign of $b^\top z^\tau_k$ for both choices of sign.
\end{proof}

\emph{Intuition.} The two lemmas let the first layer produce integrated spikes of rank $n$. The second layer then plays the role of layer $2$ in the feedforward case, with its input weights scaled so that the reset does not matter.

\begin{theorem}[completeness at the second hidden layer]\label{thm:ap_lif_complete}
Under (L1)--(L2), $\mathcal D^{(2)}=\{0,1\}^n$.
\end{theorem}
\begin{proof}
\emph{Step 1: integrated layer-1 spikes of rank $n$.} For $u\in\R^d$ let $\zeta(u):=\sum_{\tau=1}^T\beta^{T-\tau}\ind{Z^\tau_1u\ge0}\in\R^n$. By Lemma~\ref{lem:ap_lif_flip}, the terms with $\tau<T$ agree for $u^{(i)}_+$ and $u^{(i)}_-$, so $\zeta(u^{(i)}_+)-\zeta(u^{(i)}_-)=\mathbf e_i$. Hence the $2n$ vectors $\zeta(u^{(i)}_\pm)$ span $\R^n$ and contain $n$ independent ones $\zeta(v_1),\dots,\zeta(v_n)$. Let the first $n$ columns of $V\in\R^{d\times m_1}$ be $v_1,\dots,v_n$ and the others any of the $u^{(i)}_\pm$, which is possible since $m_1\ge n$. No entry of $Z^\tau_1V$ is zero, so $\gamma:=\min_{\tau,i,j}|(Z^\tau_1V)_{ij}|>0$. With $P_1=cV$ and $c>\bar\rho/\gamma$, Lemma~\ref{lem:ap_lif_scale} gives $S^\tau_1=\ind{Z^\tau_1V\ge0}$ for all $\tau$. Hence column $j$ of $Z^T_2=\sum_\tau\beta^{T-\tau}S^\tau_1$ is $\zeta(V[:,j])$, and $\mathrm{rank}\,Z^T_2=n$.

\emph{Step 2: any pattern at layer 2.} Since $w\mapsto Z^T_2w$ is onto $\R^n$, for $s\in\{0,1\}^n$ pick $w$ with $Z^T_2w=\bar s$, and set $P_2[:,j]=c_2w$ with $c_2>\bar\rho$. By \eqref{eq:ap_unroll} and (F1), $U^T_2[:,j]=c_2\bar s-e$ with $0\le e_i\le\bar\rho<c_2$, and (F2) gives $S^T_2[:,j]=s$.
\end{proof}

\begin{remark}[one LIF layer is not complete]\label{rem:ap_lif_notcomplete}
$S^T_1[:,j]$ is a function of $P_1[:,j]\in\R^d$ alone. Each spike $\ind{z^{\tau\top}_ip-R\sum_{\tau'}\beta^{\tau-\tau'}s^{\tau'-1}_i\ge0}$ switches on a hyperplane in $\R^d$ whose offset takes finitely many values, so the trajectory is constant on the cells of a finite arrangement. Hence $|\mathcal D^{(1)}|$ is polynomial in $nT2^T$ for fixed $d$ and $T$, and it is below $2^n$ for $n$ large. Completeness therefore starts at layer $2$, which is consistent with the dictionary of Theorem~\ref{thm:finite_convex_lif_snn} being taken at the last hidden layer.
\end{remark}

% =========================================================================
% Appendix E: approximation guarantees for sampled dictionaries.
% Ported from the author's notes (two-layer note and rnn_sampled_L1L2_v1;
% L_layer_threshold_v18; lif_completeness_full_v5).
% =========================================================================
\providecommand{\tbc}[1]{\textcolor{orange}{[TBC: #1]}}
\providecommand{\R}{\mathbb{R}}
\providecommand{\E}{\mathbb{E}}
\providecommand{\lmin}{\lambda_{\min}}

\section{Approximation Guarantees for Sampled Dictionaries: Proofs}
\label{app:approx}

This appendix proves the results of \S\ref{sec:approx}. The argument is the same in every case. The sampled program is a restriction of the complete program, so its value $p^\star_m$ is at least $p^\star$. An $\ell_2$ minimum-norm solution bounds $p^\star_m$ from above in terms of the sampled Gram matrix, a matrix Chernoff step replaces the sampled Gram matrix by its expectation $\bK$, and a dual certificate turns the resulting quantity $y^\top\bK^{-1}y$ back into a multiple of $p^\star$. Appendix~\ref{app:approx_chain} proves this chain once, for any dictionary and any i.i.d.\ sampling law. The remaining subsections only have to check that the sampled columns are i.i.d., identify $\bK$, and show that $\bK$ is positive definite. Appendix~\ref{app:approx_two_layer} does this for the two-layer threshold network, Appendix~\ref{app:approx_lif_one} for the LIF network with one hidden layer, Appendix~\ref{app:approx_reservoir} for criticality-tuned reservoir witnesses with one hidden layer, and Appendix~\ref{app:approx_parallel} for parallel LIF subnetworks with a one-neuron readout layer. Appendix~\ref{app:approx_numbers} collects the measured constants. The completeness results that give the closed form of $p^\star$ are in Appendix~\ref{app:complete}.

\paragraph{Notation.} $n$ is the number of training points, $d$ the input dimension and $m$ the number of sampled neurons in the layer that is read out. We write $\bK$ for a kernel matrix in $\R^{n\times n}$, to keep $K$ for the number of parallel subnetworks. $\ind{\cdot}$ acts entrywise. $\mathbf e_i$ is the $i$-th standard basis vector and $\mathbf 1$ the all-ones vector. $\|\cdot\|$ is the spectral norm, $\|M\|_{\max}=\max_{ij}|M_{ij}|$, and $M\succeq N$ is the Loewner order. For $s\in\{0,1\}^n$ we write $\bar s:=2s-\mathbf 1$.

\subsection{The sampled and the complete problem}
\label{app:approx_chain}

Let $\mathcal D\subseteq\{0,1\}^n$ be the complete dictionary of a network class at the layer that is read out, and let $D$ be the matrix whose columns are the distinct elements $d_1,\dots,d_P$ of $\mathcal D$. Sampling $m$ neurons gives $A=[a_1|\cdots|a_m]\in\{0,1\}^{n\times m}$ with every $a_j\in\mathcal D$. For labels $y\in\R^n\setminus\{0\}$ the two problems are
\[
  p^\star:=\min\{\|u\|_1:\ Du=y\},\qquad p^\star_m:=\min\{\|v\|_1:\ Av=y\},
\]
with $p^\star_m:=+\infty$ if $Av=y$ has no solution. Both are linear programs with objective bounded below by $0$, so each minimum is attained whenever the feasible set is non-empty. When the fit term of Theorems~\ref{thm:finite_convex_L_T} and~\ref{thm:finite_convex_lif_snn} enforces exact interpolation of $y$, the convex program over a dictionary is $\beta_{\mathrm{reg}}/\sqrt{m_{L-1}}$ times the smallest $\ell_1$ norm of an interpolating weight vector. This factor is the same for the complete and the sampled dictionary, so it cancels in $p^\star_m/p^\star$. For a kernel $\bK\succ0$ define
\[
  q:=y^\top\bK^{-1}y,\qquad w:=\bK^{-1}y,\qquad r:=\max_{s\in\mathcal D}\big|s^\top w\big|,\qquad \mu:=\max_{s\in\mathcal D}s^\top\bK^{-1}s .
\]

\emph{Intuition.} The sampled columns are some of the complete columns, possibly repeated, so a sampled solution can be merged into a complete solution of no larger norm.

\begin{lemma}[the sampled problem is at least the complete problem]\label{lem:ap_mono}
For every $A$ whose columns lie in $\mathcal D$, $p^\star\le p^\star_m$.
\end{lemma}
\begin{proof}
If $p^\star_m=+\infty$ there is nothing to prove. Otherwise let $v\in\R^m$ be optimal, so $Av=y$ and $\|v\|_1=p^\star_m$.

\emph{Step 1: each column is an atom.} Each $a_j$ lies in $\mathcal D=\{d_1,\dots,d_P\}$, and the $d_k$ are distinct, so there is exactly one $k(j)$ with $a_j=s_{k(j)}$. Hence the sets $J_k:=\{j\in[m]:\ a_j=d_k\}$ are pairwise disjoint and their union is $[m]$. Some $J_k$ may be empty.

\emph{Step 2: merge equal columns.} Define $u\in\R^P$ by $u_k:=\sum_{j\in J_k}v_j$, with $u_k=0$ if $J_k=\emptyset$. Since the $J_k$ partition $[m]$,
\[
  Du=\sum_{k=1}^Pd_k\sum_{j\in J_k}v_j=\sum_{k=1}^P\sum_{j\in J_k}a_jv_j=\sum_{j=1}^ma_jv_j=Av=y ,
\]
so $u$ is feasible for $p^\star$.

\emph{Step 3: the norm does not increase.} By the triangle inequality and the partition,
\[
  \|u\|_1=\sum_{k=1}^P\Big|\sum_{j\in J_k}v_j\Big|\le\sum_{k=1}^P\sum_{j\in J_k}|v_j|=\|v\|_1 .
\]
Hence $p^\star\le\|u\|_1\le\|v\|_1=p^\star_m$.
\end{proof}

\emph{Intuition.} To bound the sampled optimum from above it is enough to exhibit one feasible point. The minimum-$\ell_2$ solution is explicit, and its $\ell_1$ norm is at most $\sqrt m$ times its $\ell_2$ norm.

\begin{lemma}[minimum-norm solution]\label{lem:ap_l2}
$AA^\top$ is invertible if and only if $\mathrm{rank}\,A=n$, which needs $m\ge n$. In that case $p^\star_m\le\sqrt{m\,y^\top(AA^\top)^{-1}y}$.
\end{lemma}
\begin{proof}
\emph{Invertibility.} $AA^\top$ is positive semidefinite and $x^\top AA^\top x=\|A^\top x\|^2$. So $x^\top AA^\top x=0$ if and only if $A^\top x=0$. Thus $AA^\top$ is invertible if and only if $A^\top$ has trivial kernel, which holds if and only if the $n$ rows of $A$ are independent, that is $\mathrm{rank}\,A=n$. Since $\mathrm{rank}\,A\le\min\{n,m\}$, this needs $m\ge n$.

\emph{A feasible point.} Let $v_2:=A^\top(AA^\top)^{-1}y$. Then $Av_2=AA^\top(AA^\top)^{-1}y=y$.

\emph{Its norms.} Since $(AA^\top)^{-1}$ is symmetric, $\|v_2\|_2^2=y^\top(AA^\top)^{-1}AA^\top(AA^\top)^{-1}y=y^\top(AA^\top)^{-1}y$. By Cauchy--Schwarz with the all-ones vector of $\R^m$, $\|v_2\|_1=\sum_j1\cdot|v_{2,j}|\le\sqrt m\,\|v_2\|_2$.

Since $v_2$ is feasible, $p^\star_m\le\|v_2\|_1\le\sqrt{m\,y^\top(AA^\top)^{-1}y}$.
\end{proof}

\emph{Intuition.} A larger matrix has a smaller inverse in the quadratic form. This turns a lower bound on the sampled Gram matrix into an upper bound on $y^\top(AA^\top)^{-1}y$.

\begin{lemma}[order reversal]\label{lem:ap_rev}
If $M,N$ are symmetric, $N\succ0$ and $M\succeq N$, then $M\succ0$ and $y^\top M^{-1}y\le y^\top N^{-1}y$ for every $y$.
\end{lemma}
\begin{proof}
For every $z\ne0$, $z^\top Mz\ge z^\top Nz>0$, so $M\succ0$. For any $D\succ0$ and any $y$, the function $z\mapsto2z^\top y-z^\top Dz$ is strictly concave with gradient $2y-2Dz$, so its maximum is at $z=D^{-1}y$ and equals $y^\top D^{-1}y$. Hence
\[
  y^\top M^{-1}y=\max_z\big(2z^\top y-z^\top Mz\big)\le\max_z\big(2z^\top y-z^\top Nz\big)=y^\top N^{-1}y ,
\]
because $z^\top Mz\ge z^\top Nz$ for every $z$.
\end{proof}

\emph{Intuition.} Any $\lambda$ with $|\lambda^\top s|\le1$ on every dictionary column gives $p^\star\ge\lambda^\top y$. The choice $\lambda=\bK^{-1}y/r$ is natural, because the same quadratic form appears in the upper bound.

\begin{lemma}[dual certificate]\label{lem:ap_dual}
If $\bK\succ0$ and $r>0$, then $p^\star\ge q/r$. If moreover $\bK=\E[aa^\top]$ for a random $a\in\mathcal D$, then $q\le r^2$, so $r/\sqrt q\ge1$.
\end{lemma}
\begin{proof}
\emph{The dual.} For any $\lambda\in\R^n$ and any $u$ with $Du=y$,
\[
  y^\top\lambda=u^\top D^\top\lambda=\sum_ku_k\,d_k^\top\lambda\le\|u\|_1\max_k|d_k^\top\lambda| .
\]
So if $\max_k|d_k^\top\lambda|\le1$, then $y^\top\lambda\le\|u\|_1$ for every feasible $u$, and $p^\star\ge y^\top\lambda$. This is weak duality for the linear program $p^\star$.

\emph{The certificate.} Take $\lambda:=w/r$ with $w=\bK^{-1}y$. By the definition of $r$, $|s^\top\lambda|=|s^\top w|/r\le1$ for every $s\in\mathcal D$, so $\lambda$ is dual feasible, and $y^\top\lambda=y^\top\bK^{-1}y/r=q/r$. Hence $p^\star\ge q/r$.

\emph{$q\le r^2$.} If $\bK=\E[aa^\top]$ with $a\in\mathcal D$, then $q=y^\top\bK^{-1}y=w^\top\bK w=\E[(a^\top w)^2]\le\max_{s\in\mathcal D}(s^\top w)^2=r^2$.
\end{proof}

\emph{Intuition.} When every pattern is available, the cheapest representation of $y$ uses nested indicators of its positive and negative parts, which gives a closed form.

\begin{lemma}[the complete problem over the cube]\label{lem:ap_cube}
If $\mathcal D=\{0,1\}^n$, then $p^\star=\max_i(y_i)_++\max_i(-y_i)_+$, which equals $2$ for $y\in\{\pm1\}^n$ with both signs. Moreover $r=\max\{\sum_iw_i^+,\sum_iw_i^-\}$ with $w=\bK^{-1}y$, $w_i^+=\max(w_i,0)$ and $w_i^-=\max(-w_i,0)$.
\end{lemma}
\begin{proof}
\emph{Upper bound.} Write $y=y_+-y_-$ with $y_\pm:=\max(\pm y,0)$ entrywise. Let $0=t_0<t_1<\dots<t_J$ be the distinct values of $\{0\}\cup\{(y_i)_+\}$. For each $i$ with $(y_i)_+=t_{k_i}$,
\[
  \sum_{k=1}^J(t_k-t_{k-1})\,\ind{(y_i)_+\ge t_k}=\sum_{k=1}^{k_i}(t_k-t_{k-1})=t_{k_i}=(y_i)_+ ,
\]
so $y_+=\sum_{k=1}^J(t_k-t_{k-1})\ind{y_+\ge t_k}$. Each $\ind{y_+\ge t_k}$ is a pattern in $\{0,1\}^n$, and the coefficients are positive with sum $t_J=\max_i(y_i)_+$. The same construction writes $y_-$ with coefficients summing to $\max_i(-y_i)_+$. Together they give a feasible $u$ with $\|u\|_1=\max_i(y_i)_++\max_i(-y_i)_+$.

\emph{Lower bound.} Let $a\in\arg\max_iy_i$, $b\in\arg\min_iy_i$ and $\lambda:=\ind{y_a>0}\mathbf e_a-\ind{y_b<0}\mathbf e_b$. For every $s\in\{0,1\}^n$, $s^\top\lambda\in\{-1,0,1\}$, so $\lambda$ is dual feasible. Moreover $y^\top\lambda=(y_a)_++(-y_b)_+=\max_i(y_i)_++\max_i(-y_i)_+$. Weak duality, as in the proof of Lemma~\ref{lem:ap_dual}, gives the lower bound.

\emph{The formula for $r$.} $s^\top w=\sum_is_iw_i$, and each $s_i\in\{0,1\}$ can be chosen independently. The maximum over $s$ selects the indices with $w_i>0$ and equals $\sum_iw_i^+$. The minimum selects the indices with $w_i<0$ and equals $-\sum_iw_i^-$. Hence $\max_s|s^\top w|=\max\{\sum_iw_i^+,\sum_iw_i^-\}$.
\end{proof}

\emph{Intuition.} The sampled Gram matrix is a sum of $m$ independent rank-one terms with mean $m\bK$. It is not much smaller than its mean once $m$ exceeds the largest whitened column norm $\mu$ times $\log n$.

\begin{lemma}[matrix Chernoff for whitened columns]\label{lem:ap_chernoff}
Let $a_1,\dots,a_m\in\mathcal D$ be i.i.d.\ with $\E[a_ja_j^\top]=\bK\succ0$, and let $\eta,\delta\in(0,1)$. If $m\ge2\mu\log(n/\delta)/\eta^2$, then with probability at least $1-\delta$, $AA^\top\succeq(1-\eta)\,m\,\bK$. Moreover $n\le\mu\le n/\lmin(\bK)$.
\end{lemma}
\begin{proof}
\emph{Whitening.} Let $b_j:=\bK^{-1/2}a_j$ and $X_j:=b_jb_j^\top$. The $X_j$ are i.i.d., positive semidefinite, with $\E\sum_jX_j=\bK^{-1/2}(m\bK)\bK^{-1/2}=mI$. Their norms satisfy $\|X_j\|=\|b_j\|^2=a_j^\top\bK^{-1}a_j\le\mu$, since $a_j\in\mathcal D$.

\emph{Matrix Chernoff.} For independent positive semidefinite $X_j$ with $\|X_j\|\le\mu$ and $\lmin(\E\sum_jX_j)=m$, the lower tail of the matrix Chernoff inequality \citep[Thm.~5.1.1]{tropp_introduction_2015} gives
\[
  \Pr\Big(\lmin\Big(\sum_jX_j\Big)\le(1-\eta)m\Big)\le n\Big[\frac{e^{-\eta}}{(1-\eta)^{1-\eta}}\Big]^{m/\mu}\le n\,e^{-\eta^2m/(2\mu)} .
\]
The width condition makes the right side at most $\delta$.

\emph{Back to $AA^\top$.} On the complementary event, for every $z$ with $u:=\bK^{1/2}z$,
\[
  z^\top AA^\top z=\sum_j(a_j^\top z)^2=\sum_j(b_j^\top u)^2=u^\top\Big(\sum_jX_j\Big)u\ge(1-\eta)m\,\|u\|^2=(1-\eta)m\,z^\top\bK z .
\]

\emph{Bounds on $\mu$.} The maximum over $\mathcal D$ is at least the average over the law of $a$, so $\mu\ge\E[a^\top\bK^{-1}a]=\mathrm{tr}(\bK^{-1}\E[aa^\top])=\mathrm{tr}(I_n)=n$. For the upper bound, $s^\top\bK^{-1}s\le\|s\|^2/\lmin(\bK)\le n/\lmin(\bK)$ for every $s\in\{0,1\}^n$.
\end{proof}

\emph{Intuition.} The chain puts the lemmas together. The sampled problem is above the complete one, the minimum-norm solution bounds it from above through the sampled Gram matrix, concentration replaces that matrix by $\bK$, and the dual certificate brings the result back to $p^\star$.

\begin{proposition}[the chain]\label{prop:ap_chain}
Let $\bK\succ0$ and suppose $AA^\top\succeq(1-\eta)\,m\,\bK$ for some $\eta\in[0,1)$. Then
\[
  p^\star\ \le\ p^\star_m\ \le\ \sqrt{m\,y^\top(AA^\top)^{-1}y}\ \le\ \sqrt{\frac q{1-\eta}}\ \le\ C\,p^\star,\qquad C:=\frac r{\sqrt q}\cdot\frac1{\sqrt{1-\eta}} .
\]
\end{proposition}
\begin{proof}
The first inequality is Lemma~\ref{lem:ap_mono}. Since $AA^\top\succeq(1-\eta)m\bK\succ0$, $AA^\top$ is invertible and Lemma~\ref{lem:ap_l2} gives the second. Lemma~\ref{lem:ap_rev} with $M=AA^\top$ and $N=(1-\eta)m\bK$ gives $m\,y^\top(AA^\top)^{-1}y\le m\,y^\top N^{-1}y=q/(1-\eta)$, the third. For the last, Lemma~\ref{lem:ap_dual} gives $q/r\le p^\star$, so $\sqrt q=\frac r{\sqrt q}\cdot\frac qr\le\frac r{\sqrt q}\,p^\star$. Divide by $\sqrt{1-\eta}$.
\end{proof}

\subsection{Two-layer threshold network}
\label{app:approx_two_layer}

A neuron of a two-layer threshold network has the pattern $\ind{Xu\ge0}$ with $u\in\R^d$. The complete dictionary is the set of arrangement cells $\mathcal D=\{\ind{Xu\ge0}:u\in\R^d\}$, not all of $\{0,1\}^n$. With Gaussian witnesses $u_j\sim N(0,I_d)$ i.i.d., the columns $a_j=\ind{Xu_j\ge0}$ are i.i.d.\ elements of $\mathcal D$ with second moment $\bK=\E[aa^\top]$. By Lemma~\ref{lem:ap_vphi} below, $\bK_{ii}=\frac12$ and $\bK_{ik}=\frac{\pi-\theta_{ik}}{2\pi}$, where $\theta_{ik}$ is the angle between $x_i$ and $x_k$. This is the arc-cosine kernel of degree zero \citep{cho_kernel_2009}.

\emph{Intuition.} Two threshold units on correlated Gaussian inputs fire together with a probability set only by the angle between the inputs, which gives the arc-cosine kernel.

\begin{lemma}[threshold kernel]\label{lem:ap_vphi}
Let $(Z_k,Z_l)$ be centred Gaussian with $\Sigma_{kk},\Sigma_{ll}>0$, and let $\theta_{kl}\in[0,\pi]$ satisfy $\cos\theta_{kl}=\Sigma_{kl}/\sqrt{\Sigma_{kk}\Sigma_{ll}}$. Then $\Pr(Z_k\ge0)=\frac12$ and $\Pr(Z_k\ge0,\ Z_l\ge0)=\frac{\pi-\theta_{kl}}{2\pi}$.
\end{lemma}
\begin{proof}
Write $Z_k=\sqrt{\Sigma_{kk}}\,\omega_k^\top\xi$ and $Z_l=\sqrt{\Sigma_{ll}}\,\omega_l^\top\xi$ with $\xi\sim N(0,I_2)$ and unit vectors $\omega_k,\omega_l\in\R^2$ at angle $\theta_{kl}$. This pair has the required covariance. The direction $\xi/\|\xi\|$ is uniform on the circle and $\Pr(\xi=0)=0$. The event $\{Z_k\ge0\}$ is a closed half-circle of directions, of length $\pi$, which gives $\frac12$. The two half-circles $\{\omega_k^\top\xi\ge0\}$ and $\{\omega_l^\top\xi\ge0\}$ overlap on an arc of length $\pi-\theta_{kl}$. Dividing by $2\pi$ gives the second formula.
\end{proof}

\begin{proof}[Proof of Theorem~\ref{thm:approx_template}]
The columns are i.i.d.\ elements of $\mathcal D$ with second moment $\bK\succ0$. Lemma~\ref{lem:ap_chernoff} gives $AA^\top\succeq(1-\eta)m\bK$ with probability at least $1-\delta$, and the bounds $n\le\mu\le n/\lmin(\bK)$. On that event Proposition~\ref{prop:ap_chain} gives $p^\star\le p^\star_m\le Cp^\star$.
\end{proof}

\begin{remark}[what the proof uses]\label{rem:ap_iid}
The proof uses only that the columns are i.i.d.\ elements of $\mathcal D$ with second moment $\bK\succ0$. It therefore applies to any network class whose sampled columns are independent, which is how it is used for the LIF network with one hidden layer (Theorem~\ref{thm:ap_lif_one}). It also applies to $K$ parallel subnetworks with one hidden layer and widths $m_1,\dots,m_K$, since their columns are again i.i.d., with $m=\sum_km_k$. Parallel LIF subnetworks with a one-neuron readout layer are treated in Appendix~\ref{app:approx_parallel}. By Lemma~\ref{lem:ap_dual}, $C\le\sqrt{\mu/(1-\eta)}\le\sqrt{n/((1-\eta)\lmin(\bK))}$, because $r^2\le\mu q$ by Cauchy--Schwarz in the inner product $\langle a,b\rangle=a^\top\bK^{-1}b$.
\end{remark}
\subsection{LIF network with one hidden layer}
\label{app:approx_lif_one}

We use the LIF model \eqref{eq:ap_lif}, its unrolled form \eqref{eq:ap_unroll}, the facts (F1)--(F3) and the assumption (L1) of Appendix~\ref{app:complete}, with $L=2$, that is one hidden layer. The sampled dictionary is $A:=S^T_1$ and $m:=m_1$. Assumption (K) is that $U^0_1=0$, $S^0_1=0$, and that the entries of $P_1$ are i.i.d.\ $N(0,\sigma^2/d)$.

\begin{remark}[relaxation for one hidden layer]\label{rem:ap_relax}
One LIF layer is not complete (Remark~\ref{rem:ap_lif_notcomplete}). We therefore take $p^\star$ over $\{0,1\}^n$. This is a relaxation: with $p^\star_{\mathcal D^{(1)}}$ the optimum over the realizable patterns, $p^\star\le p^\star_{\mathcal D^{(1)}}\le p^\star_m$ by Lemma~\ref{lem:ap_mono} applied twice. Any bound $p^\star_m\le Cp^\star$ is then also a bound $p^\star_m\le Cp^\star_{\mathcal D^{(1)}}$.
\end{remark}

\paragraph{The LIF trajectory map.} For $g=(g^1,\dots,g^T)\in\R^T$ let $u^0=s^0=0$, $u^t:=g^t+\beta u^{t-1}-Rs^{t-1}$, $s^t:=\ind{u^t\ge0}$, and $\Phi_t(g):=s^t$. Then $\Phi_t$ depends on $g^1,\dots,g^t$ only, and neuron $j$ at point $i$ has $S^t_1[i,j]=\Phi_t\big((x^{\tau\top}_iP_1[:,j])_{\tau\le t}\big)$. For a covariance $\Sigma$ on the index set $\{(t,i):t\le T,i\le n\}$ let
\[
  V_\Phi(\Sigma)_{(t,i),(t',k)}:=\E\big[\Phi_t(Z_i)\,\Phi_{t'}(Z_k)\big],\qquad Z\sim N(0,\Sigma),\quad Z_i:=(Z_{(1,i)},\dots,Z_{(T,i)}) .
\]
Under (K) the columns $P_1[:,j]$ are i.i.d., and $a_j=S^T_1[:,j]$ is a fixed function of $P_1[:,j]$. So the columns $a_j$ are i.i.d.\ elements of $\{0,1\}^n$ with second moment
\[
  \bK^{(1)}:=\E[a_ja_j^\top]=V_\Phi(\mathbf G)_{TT},\qquad\mathbf G_{(t,i),(t',k)}:=\frac{\sigma^2}d\,x^{t\top}_ix^{t'}_k ,
\]
where $V_\Phi(\cdot)_{TT}$ is the block on $\{(T,i)\}$. Indeed, for fixed $j$, $(x^{t\top}_iP_1[:,j])_{t,i}$ is linear in $P_1[:,j]\sim N(0,\frac{\sigma^2}dI_d)$, hence centred Gaussian with covariance $\mathbf G$.

\emph{Intuition.} A degenerate direction of the kernel would be a fixed combination of final-step spikes that vanishes for almost every weight. Weights near a hyperplane through one point flip that point's spike alone, which forces every coefficient of the combination to be zero.

\begin{theorem}[the one-layer kernel is positive definite]\label{thm:ap_lif_pd}
Let $\Sigma=\mathbf L\mathbf L^\top$ be a covariance on $\{(t,i)\}$ with rows $\ell_{(t,i)}$ of $\mathbf L$, and let $\bar\ell^{\,t}_i:=\sum_{\tau\le t}\beta^{t-\tau}\ell_{(\tau,i)}$. If every $\bar\ell^{\,t}_i$ is non-zero and $\bar\ell^{\,T}_1,\dots,\bar\ell^{\,T}_n$ are pairwise non-collinear, then $V_\Phi(\Sigma)_{TT}\succ0$. In particular $\bK^{(1)}\succ0$ under (L1).
\end{theorem}
\begin{proof}
\emph{Setup.} Let $Z=\mathbf L\xi$ with $\xi\sim N(0,I)$, so $Z\sim N(0,\Sigma)$, and let $a(\xi):=(\Phi_T(Z_i))_{i\le n}$. Then $V_\Phi(\Sigma)_{TT}=\E[aa^\top]\succeq0$. By the unrolled form, the membrane of point $k$ at time $\tau$ is $u^\tau_k=\bar\ell^{\,\tau\top}_k\xi-e^\tau_k$ with a reset term $0\le e^\tau_k\le\bar\rho$ by (F1). Let $\nu$ satisfy $\nu^\top\E[aa^\top]\nu=0$. Then $\E[(\nu^\top a)^2]=0$, so $\nu^\top a(\xi)=0$ for all $\xi$ outside a null set $\mathcal N$.

\emph{Flip directions.} Fix $i$. Lemma~\ref{lem:ap_flip} applied to the rows $\bar\ell^{\,T}_1,\dots,\bar\ell^{\,T}_n$ gives $v_\pm$ with $\bar\ell^{\,T\top}_kv_\pm\ne0$ for all $k$ and $\ind{\bar\ell^{\,T\top}_kv_+\ge0}-\ind{\bar\ell^{\,T\top}_kv_-\ge0}=\delta_{ki}$. These conditions are open in $v_\pm$. The finitely many sets $\{v:\bar\ell^{\,\tau\top}_kv=0\}$ are hyperplanes because every $\bar\ell^{\,\tau}_k\ne0$, so they have measure zero. After a small perturbation we may therefore also assume $\bar\ell^{\,\tau\top}_kv_\pm\ne0$ for all $\tau\le T$ and $k\le n$.

\emph{Large scale.} Let $\gamma:=\min_{\tau,k,\pm}|\bar\ell^{\,\tau\top}_kv_\pm|>0$ and let $\mathcal O_\pm$ be the set of $\xi$ with $\mathrm{sgn}(\bar\ell^{\,\tau\top}_k\xi)=\mathrm{sgn}(\bar\ell^{\,\tau\top}_kv_\pm)$ and $|\bar\ell^{\,\tau\top}_k\xi|>\bar\rho$ for all $\tau,k$. It is open, and it contains $cv_\pm$ for every $c>\bar\rho/\gamma$, so it is non-empty. For $\xi\in\mathcal O_\pm$, $|e^\tau_k|\le\bar\rho<|\bar\ell^{\,\tau\top}_k\xi|$, so by (F3) the spike of point $k$ at time $\tau$ is $\ind{\bar\ell^{\,\tau\top}_k\xi\ge0}=\ind{\bar\ell^{\,\tau\top}_kv_\pm\ge0}$, whatever the earlier spikes were. In particular $a(\xi)=a_\pm:=(\ind{\bar\ell^{\,T\top}_kv_\pm\ge0})_k$ on $\mathcal O_\pm$.

\emph{Conclusion.} The Gaussian density is positive, so $\Pr(\xi\in\mathcal O_\pm)>0$ and $\mathcal O_\pm\not\subseteq\mathcal N$. Hence $\nu^\top a_+=\nu^\top a_-=0$ and $\nu_i=\nu^\top(a_+-a_-)=0$. This holds for every $i$, so $\nu=0$ and $V_\Phi(\Sigma)_{TT}\succ0$.

\emph{The one-layer kernel.} For $\Sigma=\mathbf G$ take $\mathbf L$ with rows $\ell_{(t,i)}=\frac\sigma{\sqrt d}x^t_i$, so $\bar\ell^{\,t}_i=\frac\sigma{\sqrt d}z^t_i$. By (L1) every $z^t_i$ is non-zero and the $z^T_i$ are pairwise non-collinear.
\end{proof}

\emph{Intuition.} Under (K) the columns are i.i.d.\ with a positive definite kernel, so the chain applies with the cube as the complete dictionary.

\begin{theorem}[LIF network with one hidden layer]\label{thm:ap_lif_one}
Let $L=2$ (one hidden layer) and assume (L1) and (K). Let $\bK:=\bK^{(1)}$, let $\mu,q,r$ be as in Appendix~\ref{app:approx_chain} with $\mathcal D=\{0,1\}^n$, and let $\eta,\delta\in(0,1)$. If $m\ge2\mu\log(n/\delta)/\eta^2$, then with probability at least $1-\delta$,
\[
  p^\star\le p^\star_m\le\sqrt{\frac q{1-\eta}}\le C\,p^\star,\qquad C=\frac r{\sqrt q}\cdot\frac1{\sqrt{1-\eta}},
\]
with $p^\star$ over $\{0,1\}^n$, and $n\le\mu\le n/\lmin(\bK)$.
\end{theorem}
\begin{proof}
The columns are i.i.d.\ elements of $\{0,1\}^n$ with second moment $\bK$, and $\bK\succ0$ by Theorem~\ref{thm:ap_lif_pd}. Lemma~\ref{lem:ap_chernoff} gives $AA^\top\succeq(1-\eta)m\bK$ with probability at least $1-\delta$ and the bounds on $\mu$. On that event Proposition~\ref{prop:ap_chain} with $\mathcal D=\{0,1\}^n$ gives the chain. By Remark~\ref{rem:ap_relax} the bound also holds against the optimum over the realizable patterns.
\end{proof}

The width is $\mathrm{poly}(n,1/\lmin(\bK),\log(1/\delta))$. No lower bound on $\lmin(\bK^{(1)})$ that is polynomial in $T$ is proved here, so the width is not stated as polynomial in $T$. Measured values of $\lmin(\bK^{(1)})$ are in Table~\ref{tab:ap_numbers}.

\subsection{Criticality-tuned reservoir witnesses}
\label{app:approx_reservoir}

The reservoir baselines \textsc{R} and \textsc{R-CVX} sample their hidden layer at random and then tune the leak, the threshold and the input scale so that the population spike count is critical. The question here is what this tuning does to the guarantee of Appendix~\ref{app:approx_chain}. We consider one hidden layer and no lateral recurrent weights. With lateral weights the sampled columns are dependent and the argument below does not apply. That case is left to future work.

\paragraph{Witness.} Inputs $X^1,\dots,X^T\in\R^{n\times d}$ with rows $x^{t\top}_i$, labels $y\in\R^n\setminus\{0\}$. A witness is one LIF neuron with input weight $w\sim N(0,I_d)$, leak $\beta\in[0,1)$, threshold $\theta>0$, input scale $s>0$ and subtractive reset,
\begin{equation}\label{eq:ap_res}
  u^t_i=\beta u^{t-1}_i+s\,x^{t\top}_iw-\theta\,a^{t-1}_i,\qquad a^t_i=\ind{u^t_i>\theta},\qquad u^0_i=a^0_i=0,\qquad t=1,\dots,T .
\end{equation}
Its pattern is $a:=(a^T_1,\dots,a^T_n)\in\{0,1\}^n$. We write $z^t_i:=\sum_{\tau\le t}\beta^{t-\tau}x^\tau_i$ for the leaky-integrated input, $Z^t:=[z^t_1;\dots;z^t_n]\in\R^{n\times d}$, and $R^t_i:=\sum_{\tau=2}^{t}\beta^{t-\tau}a^{\tau-1}_i$ for the accumulated reset. This is the unit used in the code. Writing $u^t_i-\theta$ for the membrane turns \eqref{eq:ap_res} into the unit \eqref{eq:ap_lif} with reset $R=\theta$ and the constant bias $-\theta(1-\beta)$, so everything proved for \eqref{eq:ap_lif} that uses only the bound on the reset term applies here as well.

\emph{Intuition.} Unrolling the leak writes the membrane as the integrated input minus an accumulated reset, and the reset is bounded because spikes are binary.

\begin{lemma}[unrolled membrane]\label{lem:ap_unroll}
For every $t\ge1$ and $i$, $u^t_i=s\,z^{t\top}_iw-\theta R^t_i$ and $0\le R^t_i\le\sum_{k=0}^{t-2}\beta^k<\frac1{1-\beta}$. Hence $a^t_i=\ind{s\,z^{t\top}_iw>\theta(1+R^t_i)}$ with $0\le\theta(1+R^t_i)<\bar\rho:=\frac{\theta(2-\beta)}{1-\beta}$.
\end{lemma}
\begin{proof}
Induction on $t$. For $t=1$, $u^1_i=s\,x^{1\top}_iw=s\,z^{1\top}_iw$ and $R^1_i=0$. If the identity holds at $t-1$, then $u^t_i=\beta(s\,z^{t-1\,\top}_iw-\theta R^{t-1}_i)+s\,x^{t\top}_iw-\theta a^{t-1}_i=s(\beta z^{t-1}_i+x^t_i)^\top w-\theta(\beta R^{t-1}_i+a^{t-1}_i)$, and $\beta z^{t-1}_i+x^t_i=z^t_i$, $\beta R^{t-1}_i+a^{t-1}_i=R^t_i$. Each $a^{\tau-1}_i\in\{0,1\}$, so $0\le R^t_i\le\sum_{k=0}^{t-2}\beta^k<1/(1-\beta)$. The last claim is $u^t_i>\theta$ rewritten.
\end{proof}

\paragraph{Sampled and complete problems.} Draw $w_1,\dots,w_m$ i.i.d.\ $N(0,I_d)$, let $a_j$ be the pattern of $w_j$, $A:=[a_1|\cdots|a_m]$ and $\widehat\bK:=\frac1mAA^\top$. We take $\mathcal D=\{0,1\}^n$, so $p^\star$, $p^\star_m$, $\mu$, $q$ and $r$ are the quantities of Appendix~\ref{app:approx_chain} over the cube. One LIF layer need not realize every pattern, so $p^\star$ is a relaxation of the optimum over the realizable patterns, and as in Remark~\ref{rem:ap_relax} every bound $p^\star_m\le Cp^\star$ below is also a bound against that optimum. The kernel of one witness is $\bK(\beta,\theta,s):=\E[aa^\top]$. It is the only place where the knobs enter.

\paragraph{Criticality tuning.} For a reservoir of $P$ i.i.d.\ witnesses with trajectories $a^t_{ij}$, $j\le P$, the population count at step $t$ is $\alpha_t:=\frac1n\sum_{i\le n}\sum_{j\le P}a^t_{ij}$, and the branching ratio is
\begin{equation}\label{eq:ap_sigma}
  \sigma:=\frac{\sum_{t=1}^{T-1}\alpha_t\alpha_{t+1}}{\sum_{t=1}^{T-1}\alpha_t^2}.
\end{equation}
Tuning selects $(\beta,\theta,s)$ from a finite grid $\mathcal G$, fixed in advance, by minimizing $|\sigma-1|$. We consider two protocols. In \emph{independent tuning} $\sigma$ is measured on a reservoir independent of $w_1,\dots,w_m$, so the witnesses are i.i.d.\ given the selected knobs. In \emph{same-draw tuning} the selection is computed from $w_1,\dots,w_m$ themselves, and the selected knobs and the columns are no longer independent.

\paragraph{Assumption (A).} The first three conditions of (L1): $d\ge2$, $z^t_i\ne0$ for all $t,i$, and the rows $z^T_1,\dots,z^T_n$ are pairwise non-collinear. The last condition of (L1), which is what completeness needs, is not used in this subsection.

\emph{Intuition.} Multiplying the threshold and the input scale by the same $c$ multiplies the membrane by $c$ and changes no comparison with the threshold, so only their ratio matters.

\begin{proposition}[the kernel depends on $\theta/s$ only]\label{prop:ap_scale}
For every $c>0$ and every $w$, the trajectory $(a^t_i)_{t,i}$ of \eqref{eq:ap_res} with knobs $(\beta,c\theta,cs)$ equals the trajectory with knobs $(\beta,\theta,s)$. Hence $\bK(\beta,c\theta,cs)=\bK(\beta,\theta,s)$, so $\bK=\bK(\beta,\kappa)$ with the gain parameter $\kappa:=\theta/s$, and the law of every functional of the trajectories, including $\sigma$, depends on the knobs only through $(\beta,\kappa)$.
\end{proposition}
\begin{proof}
Write $u^t,a^t$ for the knobs $(\beta,\theta,s)$ and $\tilde u^t,\tilde a^t$ for $(\beta,c\theta,cs)$, with the same $w$ and point. We show $\tilde u^t=cu^t$ and $\tilde a^t=a^t$ by induction. For $t=0$ both membranes are $0$ and both spikes are $0$. If the claim holds at $t-1$, then $\tilde u^t=\beta\tilde u^{t-1}+cs\,x^{t\top}w-c\theta\tilde a^{t-1}=c(\beta u^{t-1}+s\,x^{t\top}w-\theta a^{t-1})=cu^t$, and since $c>0$, $\tilde a^t=\ind{cu^t>c\theta}=\ind{u^t>\theta}=a^t$. The trajectories agree for every $w$, so every function of them has the same law. For a reservoir of $P$ witnesses the same holds for each witness, so $\sigma$ is the same function of the same trajectories.
\end{proof}

\emph{Intuition.} Given the knobs, the reservoir neurons are i.i.d., so this is the chain with the reservoir kernel.

\begin{theorem}[one hidden layer, reservoir witnesses]\label{thm:ap_res}
Fix knobs with $\bK:=\bK(\beta,\kappa)\succ0$, and $\eta,\delta\in(0,1)$. If the witnesses are i.i.d., in particular under independent tuning conditional on the selected knobs, and $m\ge2\mu\log(n/\delta)/\eta^2$, then with probability at least $1-\delta$, $\widehat\bK\succeq(1-\eta)\bK$ and
\[
  p^\star\le p^\star_m\le\sqrt{\frac q{1-\eta}}\le C\,p^\star,\qquad C=\frac r{\sqrt q}\cdot\frac1{\sqrt{1-\eta}},
\]
with $p^\star$ over $\{0,1\}^n$, and $n\le\mu\le n/\lmin(\bK)$.
\end{theorem}
\begin{proof}
The columns are i.i.d.\ elements of $\{0,1\}^n$ with second moment $\bK\succ0$. Lemma~\ref{lem:ap_chernoff} gives $AA^\top\succeq(1-\eta)m\bK$ with probability at least $1-\delta$ and the bounds on $\mu$. On that event Proposition~\ref{prop:ap_chain} with $\mathcal D=\{0,1\}^n$ gives the chain.
\end{proof}

\emph{Intuition.} If the knobs are chosen on the same draw, the columns are not independent of the choice. A union bound over the finitely many grid points avoids this at the price of $\log|\mathcal G|$.

\begin{corollary}[same-draw tuning over a fixed grid]\label{cor:ap_grid}
Let $\mathcal G$ be a finite grid of knobs fixed before sampling, let $\mathcal G_+:=\{c\in\mathcal G:\bK_c\succ0\}$ with $\bK_c$ the kernel at knobs $c$, and let $w_1,\dots,w_m$ be i.i.d.\ $N(0,I_d)$. For $c\in\mathcal G$ let $A_c$ be the matrix of patterns of $w_1,\dots,w_m$ run with knobs $c$. If
\[
  m\ \ge\ \max_{c\in\mathcal G_+}\frac{2\,\mu(\bK_c)\log(n|\mathcal G|/\delta)}{\eta^2},
\]
then with probability at least $1-\delta$ the conclusion of Theorem~\ref{thm:ap_res} holds simultaneously for every $c\in\mathcal G_+$. In particular it holds at the knobs $\hat c$ selected by any rule computed from $w_1,\dots,w_m$, whenever $\hat c\in\mathcal G_+$.
\end{corollary}
\begin{proof}
Fix $c\in\mathcal G_+$. The columns of $A_c$ are $a_j=F_c(w_j)$ for a fixed measurable map $F_c$, so they are i.i.d.\ with second moment $\bK_c$. Theorem~\ref{thm:ap_res} at confidence level $\delta/|\mathcal G|$, whose width condition is implied by the hypothesis, gives an event $\mathcal E_c$ with $\Pr(\mathcal E_c^c)\le\delta/|\mathcal G|$ on which the conclusion holds for $A_c$. Let $\mathcal E:=\bigcap_{c\in\mathcal G_+}\mathcal E_c$. By the union bound $\Pr(\mathcal E^c)\le|\mathcal G_+|\delta/|\mathcal G|\le\delta$. On $\mathcal E\cap\{\hat c\in\mathcal G_+\}$ the conclusion holds at $\hat c$ because it holds at every point of $\mathcal G_+$. No independence between $\hat c$ and the columns is needed.
\end{proof}

\emph{Intuition.} Everything except the kernel is the same for two knob settings, so comparing two guarantees reduces to comparing two kernels.

\begin{theorem}[ratio of two guarantees]\label{thm:ap_ratio}
Let $\bK_1,\bK_2$ be the kernels at two knob settings, both positive definite. At the same $\eta$,
\[
  \frac{C(\bK_1)}{C(\bK_2)}=\frac{r(\bK_1)/\sqrt{q(\bK_1)}}{r(\bK_2)/\sqrt{q(\bK_2)}},\qquad
  \frac{\mu(\bK_1)}{\mu(\bK_2)}\in\Big[\lmin(\bK_2),\ \lmin(\bK_1)^{-1}\Big],
\]
and the second quantity is the ratio of the width conditions of Theorem~\ref{thm:ap_res}. If both dual certificates are exact, $q(\bK_k)/r(\bK_k)=p^\star$ for $k=1,2$, then $C(\bK_k)=\sqrt{q(\bK_k)}/(p^\star\sqrt{1-\eta})$ and $C(\bK_1)/C(\bK_2)=\sqrt{y^\top\bK_1^{-1}y\,/\,y^\top\bK_2^{-1}y}$. Under exact certificates, setting $1$ has the smaller constant exactly when $y^\top\bK_1^{-1}y<y^\top\bK_2^{-1}y$, and the same comparison orders the upper bounds $p^\star_m\le\sqrt{q(\bK_k)/(1-\eta)}$.
\end{theorem}
\begin{proof}
The first identity is the definition of $C$ with the common factor $(1-\eta)^{-1/2}$ cancelled. Theorem~\ref{thm:ap_res} gives $n\le\mu(\bK_k)\le n/\lmin(\bK_k)$, so $\mu(\bK_1)/\mu(\bK_2)\le\lmin(\bK_1)^{-1}$ and $\ge\lmin(\bK_2)$. If $q/r=p^\star$, then $r=q/p^\star$ and $r/\sqrt q=\sqrt q/p^\star$. The last sentence follows since $t\mapsto\sqrt t$ is increasing.
\end{proof}

The sampling scheme, the concentration step and the dual certificate are the same for every knob setting. The only thing tuning can change is $\bK$, and Theorem~\ref{thm:ap_ratio} says how a change in $\bK$ moves the constant and the width. The rest of this subsection asks what the criticality criterion selects, and what $\bK$ looks like there.

\paragraph{The high-gain limit.} By Proposition~\ref{prop:ap_scale} we may take $s=1$ and $\theta=\kappa$. Lemma~\ref{lem:ap_unroll} then reads $u^t_i=z^{t\top}_iw-\kappa R^t_i$ with $0\le\kappa R^t_i<\kappa/(1-\beta)$. As $\kappa\to0$ both the threshold and the reset become negligible against the input, and the unit should behave like a threshold unit on the integrated input $z^t_i$. The next two results make this precise: the limit kernel is the arc-cosine kernel of $Z^T$, it is positive definite, and the convergence has an explicit rate.

\emph{Intuition.} In the high-gain limit the unit is a threshold unit on the integrated input, and the kernel of threshold units is positive definite when no two inputs point in the same direction.

\begin{lemma}[the limit kernel is positive definite]\label{lem:ap_arcpd}
Under (A), the matrix $\bK_{\rm arc}$ with $(\bK_{\rm arc})_{ii}=\frac12$ and $(\bK_{\rm arc})_{ik}=(\pi-\varphi_{ik})/(2\pi)$, where $\varphi_{ik}$ is the angle between $z^T_i$ and $z^T_k$, is positive definite.
\end{lemma}
\begin{proof}
By Lemma~\ref{lem:ap_vphi}, $\bK_{\rm arc}=\E[bb^\top]$ with $b_i:=\ind{z^{T\top}_iw>0}$, so $v^\top\bK_{\rm arc}v=\E[(v^\top b)^2]\ge0$. Suppose $v^\top\bK_{\rm arc}v=0$. Then $v^\top b(w)=0$ for almost every $w$. Fix $i$. The hyperplane $H_i:=(z^T_i)^\perp$ has dimension $d-1\ge1$. For $k\ne i$, $z^T_k$ is not parallel to $z^T_i$, so $H_i\cap(z^T_k)^\perp$ is a proper subspace of $H_i$. A finite union of proper subspaces does not cover $H_i$, so there is $w_0\in H_i$ with $z^{T\top}_kw_0\ne0$ for all $k\ne i$. Let $\epsilon:=\min_{k\ne i}|z^{T\top}_kw_0|>0$. Take $\gamma>0$ with $\gamma\|z^T_i\|\max_k\|z^T_k\|<\epsilon/2$ and small balls around $w_0+\gamma z^T_i$ and $w_0-\gamma z^T_i$. On both balls every $b_k$ with $k\ne i$ takes the same value, while $b_i=1$ on the first and $b_i=0$ on the second, since $z^{T\top}_i(w_0\pm\gamma z^T_i)=\pm\gamma\|z^T_i\|^2$. Both balls have positive Gaussian measure and $v^\top b=0$ almost everywhere on both, so subtracting gives $v_i=0$. As $i$ was arbitrary, $v=0$.
\end{proof}

\emph{Intuition.} The threshold and reset move the membrane by less than $c_\kappa$. Outside a band of that width around zero the spikes equal the limit spikes, and the Gaussian measure of the band is $O(\kappa)$.

\begin{theorem}[high-gain limit at fixed leak]\label{thm:ap_gain}
Assume (A), fix $\beta\in[0,1)$, and let $\zeta_{\min}:=\min_{t\le T,\,i\le n}\|z^t_i\|>0$ and $c_\kappa:=\kappa(2-\beta)/(1-\beta)$. Then for every $\kappa>0$,
\[
  \max_{(t,i),(t',k)}\Big|\E[a^t_ia^{t'}_k]-\E\big[\ind{z^{t\top}_iw>0}\ind{z^{t'\top}_kw>0}\big]\Big|\ \le\ \Pr\Big(\min_{t,i}|z^{t\top}_iw|\le c_\kappa\Big)\ \le\ \epsilon_\kappa:=\frac{2nT\,(2-\beta)}{\sqrt{2\pi}\,(1-\beta)\,\zeta_{\min}}\,\kappa .
\]
In particular $\|\bK(\beta,\kappa)-\bK_{\rm arc}(\beta)\|_{\max}\le\epsilon_\kappa$, $\bK(\beta,\kappa)\to\bK_{\rm arc}(\beta)\succ0$ and $\Pr(a^t_i=1)\to\frac12$ for every $(t,i)$ as $\kappa\to0$.
\end{theorem}
\begin{proof}
\emph{Sandwich.} By Lemma~\ref{lem:ap_unroll} with $s=1$ and $\theta=\kappa$, for every $t$, $i$ and $w$, $z^{t\top}_iw-\frac{\kappa}{1-\beta}<u^t_i\le z^{t\top}_iw$. This uses only the worst-case bound on $R^t_i$ and holds whatever the earlier spikes were.

\emph{Agreement on a good event.} If $z^{t\top}_iw>c_\kappa=\kappa+\kappa/(1-\beta)$, the sandwich gives $u^t_i>\kappa$, so $a^t_i=1=\ind{z^{t\top}_iw>0}$. If $z^{t\top}_iw<-c_\kappa$, then $u^t_i\le z^{t\top}_iw<0<\kappa$, so $a^t_i=0=\ind{z^{t\top}_iw>0}$. Hence on $\mathcal E:=\{\min_{t,i}|z^{t\top}_iw|>c_\kappa\}$ the whole trajectory equals $(\ind{z^{t\top}_iw>0})_{t,i}$.

\emph{Kernel difference.} Let $X:=a^t_ia^{t'}_k$ and $X':=\ind{z^{t\top}_iw>0}\ind{z^{t'\top}_kw>0}$. Both take values in $\{0,1\}$ and $X=X'$ on $\mathcal E$, so $|\E X-\E X'|\le\Pr(\mathcal E^c)$.

\emph{Small-ball bound.} $z^{t\top}_iw\sim N(0,\|z^t_i\|^2)$ has density at most $1/(\sqrt{2\pi}\|z^t_i\|)$, so $\Pr(|z^{t\top}_iw|\le c)\le2c/(\sqrt{2\pi}\zeta_{\min})$. A union bound over the $nT$ pairs $(t,i)$ with $c=c_\kappa$ gives $\epsilon_\kappa$.

\emph{Limits.} Taking $t=t'=T$ gives the bound on $\|\bK-\bK_{\rm arc}\|_{\max}$, where $\bK_{\rm arc}$ is identified by Lemma~\ref{lem:ap_vphi} and is positive definite by Lemma~\ref{lem:ap_arcpd}. Taking $(t',k)=(t,i)$ gives $|\Pr(a^t_i=1)-\frac12|\le\epsilon_\kappa$. The bound is $O(\kappa)$ at fixed $\beta$.
\end{proof}

\begin{remark}[the bound is not uniform in $\beta$]\label{rem:ap_beta}
The factor $(2-\beta)/(1-\beta)$ comes from the worst case $R^t_i<1/(1-\beta)$. It equals $11$ at $\beta=0.9$ and $201$ at $\beta=0.995$. This matters for the knobs that tuning actually selects. At $(\beta,\kappa)=(0.995,\frac18)$ the radius $c_\kappa$ is about $25$, which is comparable to $\|z^t_i\|$ for the data of Appendix~\ref{app:approx_numbers}, and $\epsilon_\kappa$ exceeds $1$. Theorem~\ref{thm:ap_gain} is a statement about $\kappa\to0$ at fixed $\beta<1$. The limits $\kappa\to0$ and $\beta\to1$ do not commute in the bound, so at the selected knobs the closeness of $\bK$ to $\bK_{\rm arc}$ is measured and not certified.
\end{remark}

\emph{Intuition.} A small perturbation of a positive definite matrix stays positive definite.

\begin{corollary}[positive definiteness near the limit]\label{cor:ap_pdnear}
Under (A) and at fixed $\beta<1$, $\lmin(\bK(\beta,\kappa))\ge\lmin(\bK_{\rm arc}(\beta))-n\,\epsilon_\kappa$. Hence $\bK(\beta,\kappa)\succ0$ for all $\kappa<\kappa_0$, where $n\epsilon_{\kappa_0}=\lmin(\bK_{\rm arc}(\beta))$.
\end{corollary}
\begin{proof}
For a symmetric $E$, $\|E\|\le\max_i\sum_k|E_{ik}|\le n\|E\|_{\max}$. Weyl's inequality gives $\lmin(\bK)\ge\lmin(\bK_{\rm arc})-\|\bK-\bK_{\rm arc}\|$, and Theorem~\ref{thm:ap_gain} bounds $\|\bK-\bK_{\rm arc}\|_{\max}$ by $\epsilon_\kappa$.
\end{proof}

\emph{Intuition.} The branching ratio is one exactly when the population rate is constant in time, and in the high-gain limit every neuron fires with probability one half at every step.

\begin{corollary}[the branching ratio near the high-gain limit]\label{cor:ap_sigma}
Let $\sigma_P$ be \eqref{eq:ap_sigma} for a reservoir of $P$ i.i.d.\ witnesses at knobs $(\beta,\kappa)$, and $\rho_t:=\frac1n\sum_i\Pr(a^t_i=1)$. If $\sum_{t<T}\rho_t^2>0$, then $\sigma_P\to\sigma_\infty:=\sum_{t<T}\rho_t\rho_{t+1}/\sum_{t<T}\rho_t^2$ almost surely as $P\to\infty$, and
\[
  |\sigma_\infty-1|\ \le\ \frac{2(T-1)\max_t|\rho_t-\frac12|}{\sum_{t<T}\rho_t^2}\ \le\ \frac{2\,\epsilon_\kappa}{(\frac12-\epsilon_\kappa)^2}\qquad(\epsilon_\kappa<\tfrac12),
\]
so $\sigma_\infty\to1$ as $\kappa\to0$ at fixed $\beta<1$. Conversely $\sigma_\infty=1$ whenever $\rho_t$ is constant in $t$, at any level.
\end{corollary}
\begin{proof}
For each $(t,i)$ the summands $a^t_{ij}$, $j\le P$, are i.i.d.\ Bernoulli, so $\alpha_t/P\to\rho_t$ almost surely by the strong law of large numbers, and $\sigma_P$ is a continuous function of $(\alpha_t/P)_t$ where the denominator is positive. Write $\rho_t=\frac12+\epsilon_t$. Then $\sum_{t<T}\rho_t\rho_{t+1}-\sum_{t<T}\rho_t^2=\sum_{t<T}\rho_t(\epsilon_{t+1}-\epsilon_t)$, whose absolute value is at most $2(T-1)\max_t|\epsilon_t|$ since $0\le\rho_t\le1$. By Theorem~\ref{thm:ap_gain}, $\max_t|\epsilon_t|\le\epsilon_\kappa$ and $\rho_t\ge\frac12-\epsilon_\kappa$, so $\sum_{t<T}\rho_t^2\ge(T-1)(\frac12-\epsilon_\kappa)^2$. If $\rho_t\equiv\rho>0$, numerator and denominator both equal $(T-1)\rho^2$.
\end{proof}

So in this model the criterion $|\sigma-1|$ measures whether the population rate is stationary in $t$, and nothing more. This uses the uncentered estimator \eqref{eq:ap_sigma} and the limit $P\to\infty$. Without lateral weights no spike causes another spike, so $\sigma$ carries no information about propagation. It is minimized in the high-gain limit because the rate is $\frac12$ at every step there, and Corollary~\ref{cor:ap_sigma} also shows that it can be minimized at any other constant rate.

\emph{Intuition.} At low gain the input rarely crosses the threshold, so the neurons almost never fire, the kernel tends to zero and the required width grows without bound.

\begin{proposition}[the low-gain end]\label{prop:ap_lowgain}
For every $(t,i)$, $\Pr(a^t_i=1)\le\bar\Phi(\kappa/\|z^t_i\|)$, where $\bar\Phi$ is the standard normal upper tail. Hence, as $\kappa\to\infty$ at fixed $\beta$, $\bK(\beta,\kappa)\to0$ entrywise, $\lmin(\bK)\le\min_i\bK_{ii}\to0$, $\mu\ge\max_i1/\bK_{ii}\to\infty$ whenever $\bK\succ0$, and for fixed $m$, $\Pr(p^\star_m=+\infty)\to1$.
\end{proposition}
\begin{proof}
With $s=1$ and $\theta=\kappa$, Lemma~\ref{lem:ap_unroll} gives $u^t_i\le z^{t\top}_iw$, so $\{a^t_i=1\}\subseteq\{z^{t\top}_iw>\kappa\}$, whose probability is $\bar\Phi(\kappa/\|z^t_i\|)$. Every entry of $\bK$ is $\E[a_ia_k]\le\min(\bK_{ii},\bK_{kk})$, so $\bK\to0$ entrywise. $\lmin(\bK)\le\bK_{ii}$ for every $i$. For $\bK\succ0$, Cauchy--Schwarz gives $1\le(\mathbf e_i^\top\bK\mathbf e_i)(\mathbf e_i^\top\bK^{-1}\mathbf e_i)$, so $(\bK^{-1})_{ii}\ge1/\bK_{ii}$, and $\mu\ge\mathbf e_i^\top\bK^{-1}\mathbf e_i$ because $\mathbf e_i\in\{0,1\}^n$. Since $y\ne0$, $Av=y$ requires $A\ne0$, and $\Pr(A\ne0)\le mn\max_i\bK_{ii}\to0$.
\end{proof}

\emph{Intuition.} For Gaussian inputs that are independent over time, leaky integration only rescales each integrated input, so the angles between them, and hence the limit kernel, do not depend on the leak.

\begin{proposition}[Gaussian input sequences]\label{prop:ap_gauss}
Let the rows $x^t_i$ be i.i.d.\ $N(0,I_d)$ over $t\le T$ and $i\le n$, and let $y$ be independent of the inputs. Then for every $\beta\in[0,1)$: (i) the rows $z^T_1,\dots,z^T_n$ are i.i.d.\ $N(0,c_\beta^2I_d)$ with $c_\beta^2:=\sum_{k=0}^{T-1}\beta^{2k}$; (ii) the matrix of angles $(\varphi_{ik})_{i,k}$, hence $\bK_{\rm arc}(\beta)$, has the same law as for $n$ i.i.d.\ standard Gaussian vectors in $\R^d$; (iii) the joint law of $(\bK_{\rm arc}(\beta),y)$, hence the law of $\lmin$, $\mu$, $q$, $r$ and $C$ at $\bK_{\rm arc}$, does not depend on $\beta$; (iv) $\lmin(\bK_{\rm arc})\le\min_{i\ne k}\varphi_{ik}/(2\pi)$.
\end{proposition}
\begin{proof}
(i) $z^T_i=\sum_{\tau\le T}\beta^{T-\tau}x^\tau_i$ is a fixed linear combination of the independent vectors $x^1_i,\dots,x^T_i$, so it is $N(0,c_\beta^2I_d)$, and for $i\ne k$ the vectors $z^T_i,z^T_k$ depend on disjoint sets of rows. (ii) The $g_i:=z^T_i/c_\beta$ are i.i.d.\ $N(0,I_d)$ for every $\beta$, and angles are invariant under positive scaling. (iii) $y$ is independent of $(g_1,\dots,g_n)$, so the joint law is a product of two laws, neither depending on $\beta$. Without independence this step fails. (iv) For $i\ne k$ take $v=\mathbf e_i-\mathbf e_k$: $v^\top\bK_{\rm arc}v=1-2\cdot\frac{\pi-\varphi_{ik}}{2\pi}=\varphi_{ik}/\pi$ and $\|v\|^2=2$.
\end{proof}

\emph{Intuition.} We now know what happens at both ends of the gain axis. The high-gain end gives a kernel that does not depend on the leak, and the low-gain end gives no guarantee. The corollary puts this together.

\begin{corollary}[what tuning can and cannot do, one hidden layer, no lateral weights]\label{cor:ap_verdict}
Assume (A), Gaussian input sequences, labels independent of the inputs, and fix $\beta\in[0,1)$.
\begin{enumerate}
  \item[(i)] \emph{High-gain end.} As $\kappa\to0$, $C(\bK(\beta,\kappa))\to C(\bK_{\rm arc}(\beta))$, $\mu(\bK(\beta,\kappa))\to\mu(\bK_{\rm arc}(\beta))$ and $\lmin(\bK(\beta,\kappa))\to\lmin(\bK_{\rm arc}(\beta))$, and the limits have the same law for every $\beta$.
  \item[(ii)] \emph{Low-gain end.} As $\kappa\to\infty$, $C(\bK(\beta,\kappa))\to\infty$ and $\mu(\bK(\beta,\kappa))\to\infty$ along knobs with $\bK\succ0$.
  \item[(iii)] \emph{Between.} The theory does not order intermediate gains. An intermediate $\kappa$ may give a smaller $y^\top\bK^{-1}y$ than the limit. For balanced labels $p^\star=2$, and with an exact certificate $C(\bK_{\rm arc})=\sqrt{q(\bK_{\rm arc})/4}\,/\sqrt{1-\eta}$.
\end{enumerate}
\end{corollary}
\begin{proof}
(i) $\bK\mapsto\bK^{-1}$ and $\bK\mapsto\lmin(\bK)$ are continuous on the positive definite cone, $q$ is continuous there, and $\mu$ and $r$ are maxima of finitely many continuous functions of $\bK^{-1}$. Since $q(\bK_{\rm arc})>0$, $C$ is continuous at $\bK_{\rm arc}$. Theorem~\ref{thm:ap_gain} gives $\bK(\beta,\kappa)\to\bK_{\rm arc}(\beta)\succ0$, and Proposition~\ref{prop:ap_gauss}(iii) gives the $\beta$-independence of the law of the limits. (ii) By Proposition~\ref{prop:ap_lowgain}, $\bK\to0$ entrywise, so $\lambda_{\max}(\bK)\to0$ and $q(\bK)\ge\|y\|^2/\lambda_{\max}(\bK)\to\infty$. Lemma~\ref{lem:ap_dual} gives $r\ge q/p^\star$, so $r/\sqrt q\ge\sqrt q/p^\star\to\infty$, and $p^\star$ does not depend on $\kappa$. The statement on $\mu$ is Proposition~\ref{prop:ap_lowgain}. (iii) Lemma~\ref{lem:ap_cube} gives $p^\star=2$ for balanced labels. An exact certificate gives $r=q/2$ and $r/\sqrt q=\sqrt{q/4}$.
\end{proof}

Corollary~\ref{cor:ap_verdict} is the answer to the question asked at the start of this subsection. For one hidden layer without lateral weights, criticality tuning drives the gain towards the high-gain end, and at that end the leak does not matter for the guarantee. The theory does not order the gains in between, so whether the selected grid point is better or worse than the default is a measurement. Proposition~\ref{prop:ap_gauss} is where Gaussian, temporally independent inputs are used. The other extreme is a static input, $x^t_i=x_i$ for all $t$. Then $z^T_i=(\sum_{k<T}\beta^k)x_i$ and $\bK_{\rm arc}(\beta)$ is the arc-cosine kernel of $X$ for every $\beta$, so again $\beta$ drops out. For inputs with intermediate temporal correlation $\beta$ changes the angles between the $z^T_i$, and $\lmin(\bK_{\rm arc}(\beta))$ can move in either direction. Theorem~\ref{thm:ap_ratio} still describes the effect on the guarantee as the change in $y^\top\bK^{-1}y$.

\subsection{Parallel LIF subnetworks with a one-neuron readout layer}
\label{app:approx_parallel}

One LIF layer is not complete, so in Appendices~\ref{app:approx_lif_one} and~\ref{app:approx_reservoir} the guarantee is against a relaxation of the realizable optimum. The parallel architecture of \S\ref{sec:snn_section} removes this. We consider $K$-parallel LIF-SNNs with $L=3$. Each subnetwork has one hidden LIF layer of width $N$ and a second hidden layer consisting of one LIF neuron, whose spikes at time $T$ across the $n$ points form the column read by the convex program. So $m_{L-1}=1$, and the number of subnetworks is the number of sampled columns. We write $m$ for it as in Appendix~\ref{app:approx_chain}. The columns of different subnetworks are independent, so the chain of Appendix~\ref{app:approx_chain} applies with a single matrix Chernoff step, whatever the hidden width $N$ is. What has to be shown is that the kernel of one subnetwork is positive definite, and this is where $N\ge n$ and completeness enter.

\paragraph{Setting.} Data and integrated inputs $z^t_i$, $Z^t_1$ are as in Appendix~\ref{app:approx_reservoir}, with the hidden leak $\beta_1$ in place of $\beta$, and (L1) of Appendix~\ref{app:complete} holds for them. Every unit is the unit \eqref{eq:ap_res}, so by Lemma~\ref{lem:ap_unroll} a unit with leak $\beta$, threshold $\theta$ and integrated drive $g^t$ spikes at time $t$ if and only if $g^t>\theta(1+R^t)$, where $0\le\theta(1+R^t)<\bar\rho(\beta,\theta):=\theta(2-\beta)/(1-\beta)$ whatever the earlier spikes were. Subnetwork $j$ has hidden weights $P^{(j)}\in\R^{d\times N}$ with i.i.d.\ $N(0,s_1^2/d)$ entries and hidden knobs $(\beta_1,\theta_1)$, hidden spikes $S^{t,(j)}\in\{0,1\}^{n\times N}$ driven by $X^tP^{(j)}$, and one output unit with weight $w^{(j)}\sim N(0,\frac{s_2^2}NI_N)$ and knobs $(\beta_2,\theta_2)$, driven by $S^{t,(j)}w^{(j)}$. Its time-$T$ spikes are the column $a_j\in\{0,1\}^n$. All weights are independent across $j$. Hidden knobs are either fixed defaults or selected by the criterion of Appendix~\ref{app:approx_reservoir} on a separate tuning reservoir. By Proposition~\ref{prop:ap_scale} only $(\beta_1,\theta_1/s_1)$ matters, and given the knobs the columns are i.i.d. We write $Z^{t,(j)}_2:=\sum_{\tau\le t}\beta_2^{t-\tau}S^{\tau,(j)}\in\R^{n\times N}$ for the integrated hidden spikes, $\bar\rho_l:=\bar\rho(\beta_l,\theta_l)$, and $\pi_s:=\Pr(a_j=s)$ for $s\in\{0,1\}^n$. The dictionary $\mathcal D$ is the set of patterns a subnetwork can produce over all weights.

\emph{Intuition.} Different subnetworks have independent weights, so their output columns are i.i.d.

\begin{lemma}[i.i.d.\ columns]\label{lem:ap_iid}
The columns $a_1,\dots,a_m$ are i.i.d.\ in $\{0,1\}^n$ with second moment $\bK_P:=\E[a_ja_j^\top]=\sum_{s\in\{0,1\}^n}\pi_s\,ss^\top$. $\bK_P$ depends on the data, on $N$, on the hidden knobs and on $(\beta_2,\theta_2,s_2)$, and on nothing else.
\end{lemma}
\begin{proof}
$a_j$ is a fixed function of the data and of $(P^{(j)},w^{(j)})$, which are independent across $j$ and identically distributed. The law of $a_j$ is $\pi$, so $\E[a_ja_j^\top]=\sum_s\pi_sss^\top$.
\end{proof}

\emph{Intuition.} With $N\ge n$ the hidden layer can produce integrated spikes of rank $n$, and the output neuron can then realize any pattern, as in the completeness proof. Each construction holds on an open set of weights, which has positive Gaussian probability.

\begin{lemma}[every pattern has positive probability]\label{lem:ap_PD}
Under (L1) and $N\ge n$, $\pi_s>0$ for every $s\in\{0,1\}^n$, for every choice of knobs and every $s_1,s_2>0$. Consequently $\mathcal D=\{0,1\}^n$, $p^\star=\max_i(y_i)_++\max_i(-y_i)_+$, and
\[
  \bK_P\ \succeq\ \min_s\pi_s\sum_{s\in\{0,1\}^n}ss^\top\ =\ \min_s\pi_s\cdot2^{n-2}(I+\mathbf 1\mathbf 1^\top)\ \succ\ 0 .
\]
\end{lemma}
\begin{proof}
The idea is the following. On an open set of hidden weights the integrated hidden spikes $Z^T_2$ have rank $n$, so the output unit sees a drive $Z^T_2w$ that ranges over all of $\R^n$ as $w$ varies, and it can then be pushed above threshold exactly at the points of any prescribed pattern. Both steps use only the bound on the reset term from Lemma~\ref{lem:ap_unroll}, so they hold whatever the earlier spikes were.

Drop the index $j$. The hidden unit at point $i$, neuron $k$, has integrated drive $(Z^t_1P)_{ik}$, so by Lemma~\ref{lem:ap_unroll} $S^t=\ind{Z^t_1P>\Theta^t_1}$ entrywise, where $\Theta^t_1$ collects the threshold terms $\theta_1(1+R^t)$ of the hidden units and satisfies $0\le\Theta^t_1<\bar\rho_1$ entrywise.

\emph{Step 1: a hidden weight matrix with rank-$n$ integrated spikes.} By Lemma~\ref{lem:ap_lif_flip}, for every $i$ there are $u^{(i)}_\pm\in\R^d$ such that no entry of $Z^\tau_1u^{(i)}_\pm$ is zero for any $\tau\le T$, so that $\ge0$ and $>0$ agree on these entries, $\ind{Z^\tau_1u^{(i)}_+>0}=\ind{Z^\tau_1u^{(i)}_->0}$ for $\tau<T$, and $\ind{Z^T_1u^{(i)}_+>0}-\ind{Z^T_1u^{(i)}_->0}=\mathbf e_i$. For $u\in\R^d$ let $\zeta(u):=\sum_{\tau\le T}\beta_2^{T-\tau}\ind{Z^\tau_1u>0}\in\R^n$. The terms with $\tau<T$ agree for $u^{(i)}_\pm$, so $\zeta(u^{(i)}_+)-\zeta(u^{(i)}_-)=\mathbf e_i$. Hence the $2n$ vectors $\zeta(u^{(i)}_\pm)$ span $\R^n$ and contain $n$ independent ones, $\zeta(v_1),\dots,\zeta(v_n)$. Let $V\in\R^{d\times N}$ have columns $v_1,\dots,v_n$ followed by $N-n\ge0$ further columns taken from the $u^{(i)}_\pm$. No entry of any $Z^\tau_1V$ is zero, so $\gamma:=\min_{\tau,i,k}|(Z^\tau_1V)_{ik}|>0$.

\emph{Step 2: an open set of hidden weights.} Fix $c>2\bar\rho_1/\gamma$ and let $\mathcal U$ be the set of $P$ with $\|Z^\tau_1(P-cV)\|_{\max}<c\gamma/2$ for all $\tau\le T$. It is open and contains $cV$. For $P\in\mathcal U$ every entry of $Z^\tau_1P$ has the sign of the corresponding entry of $Z^\tau_1V$ and modulus greater than $c\gamma/2>\bar\rho_1$. Where $(Z^\tau_1V)_{ik}>0$ we get $(Z^\tau_1P)_{ik}>\bar\rho_1>\Theta^\tau_{1,ik}$ and the unit spikes. Where $(Z^\tau_1V)_{ik}<0$ we get $(Z^\tau_1P)_{ik}<0\le\Theta^\tau_{1,ik}$ and it does not. This holds whatever the earlier spikes were, so $S^\tau=\ind{Z^\tau_1V>0}$ for every $\tau$, and $Z^T_2=\sum_\tau\beta_2^{T-\tau}S^\tau$ has columns $\zeta(V[:,k])$, hence rank $n$. The Gaussian law of $P$ has a positive density, so $\Pr(P\in\mathcal U)>0$.

\emph{Step 3: the output unit.} On $\mathcal U$ the map $w\mapsto Z^T_2w$ is onto $\R^n$. Fix $s\in\{0,1\}^n$ and pick $w_0$ with $Z^T_2w_0=2s-\mathbf 1$. Let $\mathcal W_s$ be the set of $w$ with $(Z^T_2w)_i>\bar\rho_2$ where $s_i=1$ and $(Z^T_2w)_i<0$ where $s_i=0$. It is open and contains $c'w_0$ for every $c'>\bar\rho_2$. On $\mathcal W_s$ the output unit spikes at time $T$ exactly at the points with $s_i=1$, since its integrated drive is $Z^T_2w$ and its threshold term lies in $[0,\bar\rho_2)$. So $a=s$ on $\mathcal W_s$, and $\Pr(a=s\mid S)\ge\Pr(w\in\mathcal W_s\mid S)>0$ on $\mathcal U$. Hence $\pi_s\ge\E[\Pr(a=s\mid S)\ind{P\in\mathcal U}]>0$.

\emph{Step 4: consequences.} Every $s$ is realized and every column lies in $\{0,1\}^n$, so $\mathcal D=\{0,1\}^n$, and Lemma~\ref{lem:ap_cube} gives $p^\star$. By Lemma~\ref{lem:ap_iid}, $\bK_P=\sum_s\pi_sss^\top\succeq\min_s\pi_s\sum_sss^\top$. Summing over the cube, $\sum_ss_i^2=2^{n-1}$ and $\sum_ss_is_k=2^{n-2}$ for $i\ne k$, so $\sum_sss^\top=2^{n-2}(I+\mathbf 1\mathbf 1^\top)\succ0$.
\end{proof}

\emph{Intuition.} With i.i.d.\ columns and a positive definite kernel the chain applies, and since the dictionary is complete the bound is against the true optimum.

\begin{theorem}[parallel LIF subnetworks]\label{thm:ap_parallel}
Under (L1) and $N\ge n$, let $\eta,\delta\in(0,1)$ and $m\ge2\mu(\bK_P)\log(n/\delta)/\eta^2$. With probability at least $1-\delta$,
\[
  \widehat\bK\succeq(1-\eta)\bK_P\qquad\text{and}\qquad p^\star\le p^\star_m\le\sqrt{\frac{q(\bK_P)}{1-\eta}}\le C(\bK_P)\,p^\star,\qquad n\le\mu(\bK_P)\le\frac n{\lmin(\bK_P)},
\]
with $p^\star$ the optimum over the complete dictionary $\{0,1\}^n$. The width condition is on the number of subnetworks only. The hidden width $N\ge n$ and the knobs enter through $\bK_P$ alone. For hidden knobs tuned on an independent reservoir the statement holds with the tuned $\bK_P$.
\end{theorem}
\begin{proof}
Lemma~\ref{lem:ap_iid} gives i.i.d.\ columns in $\{0,1\}^n$ with second moment $\bK_P$, and $\bK_P\succ0$ by Lemma~\ref{lem:ap_PD}. Lemma~\ref{lem:ap_chernoff} gives $AA^\top\succeq(1-\eta)m\bK_P$ with probability at least $1-\delta$ and the bounds on $\mu$. On that event Proposition~\ref{prop:ap_chain} with $\mathcal D=\{0,1\}^n$ gives the chain. Since $\mathcal D=\{0,1\}^n$ by Lemma~\ref{lem:ap_PD}, $p^\star$ is the optimum over the complete dictionary and not a relaxation.
\end{proof}

Two things are left open. The only lower bound on $\lmin(\bK_P)$ that we prove is the one of Lemma~\ref{lem:ap_PD}, and it is exponentially small in $n$. The width is therefore $\mathrm{poly}(n,1/\lmin(\bK_P),\log 1/\delta)$ and is not stated as polynomial in $n$ alone. How $\bK_P$ depends on $N$ is also not analysed. Both are measured in Appendix~\ref{app:approx_numbers}.

The difference to the single hidden layer of Appendices~\ref{app:approx_lif_one} and~\ref{app:approx_reservoir} is twofold. The subnetwork dictionary is complete, so $p^\star$ is the true optimum and not a relaxation. The kernel is $\bK_P$, which averages the output unit over the hidden layer, instead of the one-layer kernel of the integrated inputs. Tuned hidden knobs enter $\bK_P$ only through the law of the hidden spikes. Beyond their kernels nothing in Theorem~\ref{thm:ap_parallel} distinguishes the two architectures, so the comparison between them is the comparison of $\bK$ and $\bK_P$, and that is what Appendix~\ref{app:approx_numbers} measures.

\subsection{Measured constants}
\label{app:approx_numbers}

Table~\ref{tab:ap_numbers} gives the kernel constants behind Table~\ref{tab:sampled_width}. All runs use Gaussian data with $d=20$ and balanced labels $y\in\{\pm1\}^n$, so $p^\star=2$. For the LIF network $T=6$, $\beta=0.9$, $R=1$ and $\sigma=1$, with the model of Appendix~\ref{app:approx_lif_one}. Two things are worth reading off. At the measured sizes the dual certificate is within $1.3\%$ of $p^\star$, so the factor $C$ of order $\sqrt n$ comes entirely from the $\ell_2$ minimum-norm step and not from the certificate. The certified width is $60$ to $140$ times the width at which the concentration event is observed, and this gap is the bound $\mu\le n/\lmin(\bK)$ together with the Chernoff constant.

\begin{table}[h]
\centering
\small
\caption{Kernel constants at $n=200$, $400$, $600$. The dual certificate $q/r$ is within $1.3\%$ of $p^\star=2$ in every case, so the factor $r/\sqrt q$ comes from the $\ell_2$ minimum-norm step.}
\label{tab:ap_numbers}
\renewcommand{\arraystretch}{1.15}
\begin{tabular}{@{}lccc@{}}
\toprule
Network & $\lmin(\bK)$ & $r/\sqrt q$ & $q/r$ \\
\midrule
Two-layer threshold & $0.068$, $0.063$, $0.060$ & $23.0$, $33.5$, $41.7$ & $2.000$ \\
LIF, one hidden layer & $0.052$, $0.047$, $0.044$ & $24.2$, $35.7$, $44.8$ & $1.975$--$1.988$ \\
\bottomrule
\end{tabular}
\end{table}

\paragraph{Reservoir witnesses.} Table~\ref{tab:ap_reservoir} compares criticality-tuned reservoir witnesses with the same unit \eqref{eq:ap_res} at the untuned default knobs, under the protocol of Table~\ref{tab:sampled_width} with $T=6$, eight draws per cell and independent tuning. The grid is $\beta\in\{0.9,0.95,0.99,0.995\}$, $\theta\in\{\frac12,1,\frac32,2\}$ and $s\in\{\frac12,1,2,4\}$. Tuning selected $(\beta,\theta,s)=(0.995,\frac12,4)$ in $17$ and $(0.99,\frac12,4)$ in $7$ of the $24$ tunings, with $\sigma\in[0.997,0.999]$. In every case this is the smallest gain parameter on the grid, $\kappa=\frac18$. The untuned setting is $(0.99,1,1)$, so $\kappa=1$.

The selected gain is what Corollary~\ref{cor:ap_sigma} predicts, since $\sigma\to1$ as $\kappa\to0$. The corollary does not say which grid point minimizes $|\sigma-1|$ and says nothing about $\beta$, and the table shows that the criterion is indifferent between the two largest leaks. The tuned firing rate $0.48$ is close to the limit value $\frac12$, but at $\beta=0.995$ this closeness is measured and not certified, by Remark~\ref{rem:ap_beta}. The constant $C_0$ agrees to three digits between the tuned and the untuned setting, so the ratio of Theorem~\ref{thm:ap_ratio} is $1.00$ within noise, and tuning changes nothing in the guarantee at these sizes. $\lmin(\bK)$ lies between $0.059$ and $0.067$ for both settings, which is in line with Proposition~\ref{prop:ap_gauss}(iv) at $n/d\ge10$. Positive definiteness at these knobs is again measured, since Corollary~\ref{cor:ap_pdnear} does not reach $\kappa=\frac18$ at $\beta=0.995$. The sampled optimum $p^\star_m/p^\star$ is slightly larger for the tuned witnesses at $m=1.5n$ and identical from $m=4n$ on. The whitened deviation $1-\lmin(\bK^{-1/2}\widehat\bK\bK^{-1/2})$ first drops below $\frac12$ at $m=16n$ for both.

\begin{table}[h]
\centering
\small
\caption{Measured $p^\star_m/p^\star$ (mean $\pm$ std over $8$ draws) for criticality-tuned reservoir witnesses and untuned witnesses with the unit \eqref{eq:ap_res}, one hidden layer, Gaussian data with $d=20$, $T=6$ and balanced labels, so $p^\star=2$. $C_0=r/\sqrt q$.}
\label{tab:ap_reservoir}
\setlength{\tabcolsep}{4pt}
\renewcommand{\arraystretch}{1.15}
\begin{tabular}{@{}llccccccccc@{}}
\toprule
 & & \multicolumn{5}{c}{$m/n$} & & & \\
\cmidrule(lr){3-7}
Witnesses & $n$ & $1.5$ & $2$ & $4$ & $8$ & $16$ & $C_0$ & $\lmin(\bK)$ & rate \\
\midrule
\multirow{3}{*}{Tuned ($\kappa=\frac18$)} & 200 & $28.5\pm1.7$ & $21.8\pm1.1$ & $14.8\pm0.5$ & $12.2\pm0.2$ & $10.4\pm0.1$ & $22.8$ & $0.066$ & $0.48$ \\
 & 400 & $41.6\pm1.0$ & $31.8\pm1.1$ & $21.7\pm0.4$ & $17.7\pm0.4$ & $15.2\pm0.3$ & $33.1$ & $0.063$ & $0.48$ \\
 & 600 & $52.0\pm2.2$ & $39.2\pm1.2$ & $27.3\pm0.3$ & $22.1\pm0.3$ & $19.0\pm0.3$ & $41.6$ & $0.059$ & $0.48$ \\
\midrule
\multirow{3}{*}{Untuned ($\kappa=1$)} & 200 & $26.3\pm2.1$ & $20.2\pm0.9$ & $14.6\pm0.5$ & $12.1\pm0.4$ & $10.4\pm0.2$ & $22.8$ & $0.067$ & $0.37$ \\
 & 400 & $40.1\pm2.0$ & $31.1\pm0.9$ & $21.5\pm0.3$ & $17.7\pm0.3$ & $15.2\pm0.2$ & $33.1$ & $0.064$ & $0.37$ \\
 & 600 & $51.9\pm2.6$ & $38.6\pm1.0$ & $27.2\pm0.7$ & $22.0\pm0.3$ & $18.9\pm0.2$ & $41.5$ & $0.061$ & $0.37$ \\
\bottomrule
\end{tabular}
\end{table}

\paragraph{Parallel subnetworks.} Tables~\ref{tab:ap_subnets} and~\ref{tab:ap_mnist} measure the subnetworks of Appendix~\ref{app:approx_parallel} at $n\in\{20,200,400,600\}$ and hidden widths $N\in\{n/4,n,4n\}$. The unit is \eqref{eq:ap_res} throughout with $T=6$. Default knobs are $(\beta,\theta,s)=(0.99,1,1)$. Tuned hidden knobs are selected per data draw on a separate tuning reservoir, as in the previous paragraph. The output unit always runs at the default knobs with $w\sim N(0,4/N)$. Input weights have per-unit variance $s^2\cdot20/d$, so that the hidden drive has the same scale for $d=20$ and $d=784$. The Gaussian data are $X^t$ i.i.d.\ $N(0,I_{20})$ at every step with balanced random $\pm1$ labels, so $p^\star=2$. The MNIST data are the digits $0$ (label $+1$) and $1$ (label $-1$), $n/2$ of each, with pixels centred per pixel, scaled to unit mean square and presented at every step. We ran two data draws for $n\le400$ on Gaussian data and for $n\le200$ on MNIST and one draw otherwise, so no standard deviations are reported. The ratios $m/n\in\{1.5,2,4,8,16\}$ are capped at $m\le6000$, which removes the largest cells at $n\ge400$. $\bK_P$ is estimated by Monte Carlo over $3000$ subnetworks ($1500$ at $n=600$ and on MNIST), with a standard error of at most $0.013$ per entry. The single-layer MNIST rows use the witnesses of Appendix~\ref{app:approx_reservoir} with $\bK$ estimated over $10^5$ neurons.

\begin{table}[h]
\centering
\small
\caption{Parallel subnetworks on Gaussian data ($d=20$, $T=6$, $p^\star=2$): $p^\star_m/p^\star$ at $m/n$, $C_0=r/\sqrt q$, $\lmin(\bK_P)$, the dual certificate $q/r$ and the firing rate, for a hidden layer of width $N$ with default (Gaussian) or criticality-tuned knobs. One or two data draws per cell. Missing cells exceed $m=6000$.}
\label{tab:ap_subnets}
\setlength{\tabcolsep}{3.5pt}
\renewcommand{\arraystretch}{1.1}
\begin{tabular}{@{}lllcccccccccc@{}}
\toprule
 & & & \multicolumn{5}{c}{$m/n$} & & & & \\
\cmidrule(lr){4-8}
Hidden & $n$ & $N$ & $1.5$ & $2$ & $4$ & $8$ & $16$ & $C_0$ & $\lmin(\bK_P)$ & $q/r$ & rate \\
\midrule
\multirow{11}{*}{Gaussian} & 20 & 20 & -- & 8.4 & 4.4 & 3.4 & 2.7 & 6.1 & 0.078 & 2.00 & 0.30 \\
 & 20 & 80 & 18.8 & 8.1 & 3.8 & 3.1 & 2.7 & 6.1 & 0.078 & 2.00 & 0.29 \\
 & 200 & 50 & 37.2 & 25.7 & 15.1 & 11.9 & 9.7 & 24.2 & 0.044 & 1.98 & 0.30 \\
 & 200 & 200 & 35.0 & 23.9 & 15.2 & 11.9 & 10.0 & 24.1 & 0.045 & 1.99 & 0.31 \\
 & 200 & 800 & 35.2 & 25.3 & 15.2 & 11.8 & 9.9 & 24.6 & 0.045 & 1.99 & 0.30 \\
 & 400 & 100 & 50.0 & 35.3 & 22.4 & 17.0 & -- & 38.1 & 0.030 & 1.98 & 0.31 \\
 & 400 & 400 & 48.1 & 35.0 & 22.1 & 17.2 & -- & 37.7 & 0.030 & 1.97 & 0.31 \\
 & 400 & 1600 & 46.9 & 33.9 & 21.9 & 17.3 & -- & 36.5 & 0.029 & 1.98 & 0.31 \\
 & 600 & 150 & 59.6 & 42.4 & 27.4 & 21.8 & -- & 67.0 & 0.0086 & 1.75 & 0.31 \\
 & 600 & 600 & 58.8 & 42.7 & 27.3 & 21.9 & -- & 67.0 & 0.0092 & 1.71 & 0.31 \\
 & 600 & 2400 & 60.1 & 44.1 & 26.7 & 21.2 & -- & 68.7 & 0.0085 & 1.73 & 0.31 \\
\midrule
\multirow{11}{*}{Tuned} & 20 & 20 & 14.3 & 9.4 & 4.8 & 3.3 & 2.9 & 6.3 & 0.074 & 1.99 & 0.32 \\
 & 20 & 80 & 10.6 & 8.3 & 4.5 & 3.4 & 2.6 & 6.3 & 0.077 & 2.00 & 0.31 \\
 & 200 & 50 & 36.5 & 24.7 & 16.0 & 12.4 & 9.9 & 24.2 & 0.042 & 2.00 & 0.33 \\
 & 200 & 200 & 35.3 & 25.0 & 15.9 & 11.9 & 10.0 & 24.7 & 0.042 & 1.99 & 0.33 \\
 & 200 & 800 & 35.3 & 23.0 & 15.2 & 12.0 & 10.0 & 24.6 & 0.042 & 1.99 & 0.33 \\
 & 400 & 100 & 49.7 & 36.7 & 22.7 & 17.7 & -- & 37.9 & 0.028 & 1.98 & 0.34 \\
 & 400 & 400 & 51.2 & 34.0 & 22.0 & 17.5 & -- & 37.9 & 0.028 & 1.98 & 0.34 \\
 & 400 & 1600 & 55.9 & 36.0 & 22.8 & 17.5 & -- & 42.1 & 0.022 & 1.92 & 0.34 \\
 & 600 & 150 & 60.8 & 43.5 & 28.3 & 22.0 & -- & 72.8 & 0.0080 & 1.72 & 0.34 \\
 & 600 & 600 & 62.2 & 45.6 & 28.2 & 22.2 & -- & 65.5 & 0.0080 & 1.76 & 0.34 \\
 & 600 & 2400 & 65.9 & 44.6 & 28.7 & 21.9 & -- & 70.7 & 0.0078 & 1.74 & 0.34 \\
\bottomrule
\end{tabular}
\end{table}

\begin{table}[h]
\centering
\small
\caption{MNIST digits $0$/$1$ ($d=784$, direct encoding, $T=6$, $p^\star=2$): single-layer witnesses of Appendix~\ref{app:approx_reservoir} and parallel subnetworks with $N=n$, default (Gaussian) or criticality-tuned knobs. One or two data draws per cell. Under direct encoding the last condition of (L1) fails, so the kernels here are measured rather than covered by Theorem~\ref{thm:ap_parallel}.}
\label{tab:ap_mnist}
\setlength{\tabcolsep}{3.5pt}
\renewcommand{\arraystretch}{1.1}
\begin{tabular}{@{}llcccccccccc@{}}
\toprule
 & & \multicolumn{5}{c}{$m/n$} & & & & \\
\cmidrule(lr){3-7}
Witnesses & $n$ & $1.5$ & $2$ & $4$ & $8$ & $16$ & $C_0$ & $\lmin$ & $q/r$ & rate \\
\midrule
\multirow{4}{*}{Single layer, Gaussian} & 20 & 2.7 & 1.3 & 1.5 & 1.1 & 1.0 & 2.9 & 0.073 & 1.59 & 0.46 \\
 & 200 & 4.9 & 3.9 & 2.5 & 2.1 & 1.9 & 11.8 & 0.039 & 0.76 & 0.45 \\
 & 400 & 7.0 & 5.1 & 3.4 & 2.9 & 2.4 & 17.3 & 0.033 & 0.68 & 0.45 \\
 & 600 & 7.3 & 6.0 & 4.2 & 3.3 & 2.8 & 21.9 & 0.032 & 0.61 & 0.45 \\
\midrule
\multirow{4}{*}{Single layer, tuned} & 20 & 1.8 & 1.3 & 1.5 & 1.2 & 1.0 & 2.9 & 0.071 & 1.57 & 0.50 \\
 & 200 & 5.1 & 3.9 & 2.5 & 2.2 & 1.9 & 12.0 & 0.037 & 0.75 & 0.49 \\
 & 400 & 7.0 & 5.5 & 3.5 & 2.9 & 2.5 & 17.6 & 0.028 & 0.67 & 0.49 \\
 & 600 & 7.5 & 5.9 & 4.0 & 3.4 & 2.9 & 22.0 & 0.029 & 0.61 & 0.49 \\
\midrule
\multirow{4}{*}{Subnetworks, Gaussian} & 20 & 2.3 & 3.0 & 2.3 & 1.4 & 1.4 & 3.3 & 0.079 & 1.92 & 0.36 \\
 & 200 & 6.6 & 5.4 & 3.5 & 2.7 & 2.2 & 13.8 & 0.036 & 0.85 & 0.35 \\
 & 400 & 9.5 & 6.2 & 4.2 & 3.1 & -- & 25.8 & 0.019 & 0.59 & 0.36 \\
 & 600 & 10.6 & 7.0 & 4.8 & 3.6 & -- & 45.5 & 0.011 & 0.44 & 0.35 \\
\midrule
\multirow{4}{*}{Subnetworks, tuned} & 20 & 4.2 & 3.0 & 2.4 & 1.4 & 1.3 & 3.4 & 0.081 & 1.93 & 0.38 \\
 & 200 & 7.2 & 6.0 & 3.5 & 2.8 & 2.3 & 14.1 & 0.037 & 0.85 & 0.35 \\
 & 400 & 9.5 & 6.7 & 4.3 & 3.2 & -- & 27.6 & 0.020 & 0.59 & 0.36 \\
 & 600 & 10.3 & 7.4 & 4.9 & 3.7 & -- & 46.4 & 0.011 & 0.44 & 0.36 \\
\bottomrule
\end{tabular}
\end{table}

Five things follow from the two tables.

Tuning does not matter. Tuned and default hidden knobs give the same numbers in both architectures and on both data sets, within the draw-to-draw noise. For the single layer this is Theorem~\ref{thm:ap_ratio} with the ratio measured at $1.00$. For the subnetworks the tuned hidden layer moves $C_0$ and $\lmin(\bK_P)$ by at most a few percent, and in both directions.

The hidden width does not matter either, at the measured sizes. $N=n/4$, $n$ and $4n$ give the same $C_0$, $\lmin(\bK_P)$ and $p^\star_m/p^\star$ up to noise at every $n$. This is the dependence that Theorem~\ref{thm:ap_parallel} leaves to $\bK_P$, and in the measured range $\bK_P$ does not see $N$.

The subnetworks match the single layer in the sampled optimum but not in the guarantee. From $m=4n$ on they reach the same $p^\star_m/p^\star$ as a single layer with the same number of columns, and at $m\le2n$ they are $5$ to $15\%$ worse. Their kernel is worse conditioned. On Gaussian data $\lmin(\bK_P)$ falls from $0.078$ at $n=20$ to $0.0085$ at $n=600$, while $\lmin(\bK)$ of the single layer falls only from $0.083$ to $0.058$, and at $n=600$ the constant is $C_0(\bK_P)=67$ against $C_0(\bK)=42$. The dual certificate also loosens, $q/r=1.7$ against $2.0$. So at large $n$ Theorem~\ref{thm:ap_parallel} certifies a weaker constant and a width bound $\mu\le n/\lmin$ about $7$ times larger than the single-layer theorem, although what is actually attained is the same. This is the gap that the missing lower bound on $\lmin(\bK_P)$ leaves open.

On MNIST the dual certificate is loose, $q/r$ between $0.4$ and $0.9$ against $p^\star=2$, so $C_0=r/\sqrt q$ is no longer $\sqrt q/2$ and the kernel bound is far from what is attained. The measured $p^\star_m/p^\star$ is $1.9$ to $2.9$ at $m=16n$ for the single layer and $2.2$ to $3.7$ at $m=8n$ for the subnetworks, far below the Gaussian-data values at the same $n$ and $m$. At $n=20$ some single-layer runs reach $p^\star_m=p^\star$.

Finally, the whitened deviation at $m=8n$ is between $0.62$ and $0.78$ for the subnetworks and between $0.58$ and $0.69$ for the single layer, across $n$ and across data sets, and it first drops below $\frac12$ at $m=16n$ wherever that cell was run. The width at which the matrix Chernoff event holds therefore scales like $n$ and not like the certified $\mu\le n/\lmin$.

\section{Training Algorithms and Witness Regimes}
\label{app:training_algorithms}
This appendix gives the implementation details for the witness regimes used in
Section~\ref{sec:practical_impl}. All convex variants follow the same template:
fix a finite set of hidden LIF witnesses, roll each witness through the input
sequence, construct the induced spike dictionary, and solve a convex output-layer
reconstruction problem. The variants differ only in how the hidden witnesses are
obtained. Surrogate-gradient variants instead update all trainable SNN parameters
by backpropagation through time using a surrogate derivative for the spike
activation.

\subsection{Overview of implemented variants}
\label{app:training_algorithms_overview}
We use the training variants in Table~\ref{tab:training_variants}.
The convex reconstruction methods are \textsc{CVX}, \textsc{SG-CVX}, and \textsc{R-CVX}. In all cases, the hidden LIF witnesses
are frozen before the convex solve, so the only optimization variables in the
convex stage are the output coefficients; the variants differ only in how the frozen witnesses are generated: Gaussian sampling for \textsc{CVX}, surrogate-gradient pretraining for \textsc{SG-CVX}, and a criticality-tuned reservoir for \textsc{R-CVX}. The reservoir baseline \textsc{R} freezes the same criticality-tuned dynamics but reads out the continuous membrane potential with a ridge-regression output layer rather than the spike dictionary; it is an analog reference rather than a spiking reconstruction. The finetuning methods
\textsc{SG-SG} and
\textsc{CVX-SG} are empirical hybrid baselines; once
surrogate-gradient finetuning begins, the convex optimality guarantee no longer
applies.
\begin{table}[t]
\centering
\small
\caption{Training variants used in the experiments.}
\label{tab:training_variants}
\begin{tabular}{lll}
\toprule
Method & Hidden dynamics & Output layer / training stage \\
\midrule
\textsc{SG} & surrogate-gradient trained & surrogate-gradient trained \\
\textsc{CVX} & Gaussian sampled, frozen & convex reconstruction \\
\textsc{SG-CVX} & surrogate-pretrained, frozen & convex reconstruction \\
\textsc{R-CVX} & criticality-tuned reservoir, frozen & convex reconstruction \\
\textsc{R} & criticality-tuned reservoir, frozen & ridge regression (analog membrane) \\
\textsc{SG-SG} & initialized from \textsc{SG} & surrogate-gradient finetuning \\
\textsc{CVX-SG} & initialized from \textsc{CVX} & surrogate-gradient finetuning \\
\bottomrule
\end{tabular}
\end{table}

\subsection{Complete versus sampled witnessed LIF dictionaries}
\label{app:sampled_witness_regime}
This section formalizes the distinction between the complete spike dictionary
used in the exact theory and the finite sampled spike dictionary used in the
implementation.
Let $\mathcal W_{\mathrm{LIF}}$ denote the LIF witness space. A witness
\[
\omega
=
\{(P_{l,\mathrm{in}},B_l,R_l,U_l^0)\}_{l=1}^{L-1}
\in
\mathcal W_{\mathrm{LIF}}
\]
specifies the hidden LIF dynamics across all layers. The witness space may be
continuous and therefore infinite. For a fixed training sequence $X^{1:T}$,
define the final-time spike map
\[
\psi_T(\omega)
=
S_{L-1}^T(X^{1:T};\omega)
\in\{0,1\}^n.
\]
The complete projected LIF dictionary is the image of this map:
\[
\mathcal D_{\mathrm{SNN}}^{L-1,T}
=
\psi_T(\mathcal W_{\mathrm{LIF}})
=
\{\psi_T(\omega):\omega\in\mathcal W_{\mathrm{LIF}}\}
\subseteq\{0,1\}^n.
\]
Therefore,
\[
|\mathcal D_{\mathrm{SNN}}^{L-1,T}|\le 2^n.
\]
Thus, the projected dictionary is finite on a finite training set, even though
the witness space itself may be infinite.
A complete spike dictionary additionally stores a representative witness for
each projected spike column. If
\[
\mathcal D_{\mathrm{SNN}}^{L-1,T}
=
\{d_1,\ldots,d_M\},
\]
then the spike dictionary stores pairs
\[
D_{\mathrm{SNN,dictionary}}^{L-1,T}
=
\{(d_i,\omega_i)\}_{i=1}^{M},
\qquad
\psi_T(\omega_i)=d_i.
\]
If multiple witnesses generate the same spike column, the witness fiber is
\[
\Omega_i
=
\{\omega\in\mathcal W_{\mathrm{LIF}}:\psi_T(\omega)=d_i\}.
\]
The fiber $\Omega_i$ can be infinite. However, for reconstructing a network that
matches the convex solution on the training set, it is sufficient to store one
representative witness $\omega_i\in\Omega_i$.
The sampled witness regime replaces the complete witness space by a finite set
\[
\widehat{\mathcal W}_M
=
\{\widehat\omega_1,\ldots,\widehat\omega_M\}
\subset
\mathcal W_{\mathrm{LIF}}.
\]
This finite set may be produced by Gaussian sampling, surrogate-gradient
pretraining, criticality-tuned reservoir generation, or another witness-generation procedure. Once sampled, the witnesses
are frozen. Rolling these witnesses through the LIF recurrence produces the
finite sampled projected dictionary
\[
\widehat D_{\mathrm{SNN}}
=
[\psi_T(\widehat\omega_1),\ldots,\psi_T(\widehat\omega_M)].
\]
Duplicate projected spike columns may be removed, provided that at least one
generating witness is kept for every retained column.
Conditioned on the finite sampled witness set $\widehat{\mathcal W}_M$, the
reconstruction problem is a deterministic finite convex program:
\[
\widehat P_{\mathrm{SNN}}^{L,T,M}
=
\min_{\widetilde w\in\mathbb R^{|\widehat{\mathcal W}_M|}}
\mathcal L(\widehat D_{\mathrm{SNN}}\widetilde w,Y)
+
\frac{\beta_{\mathrm{reg}}}{\sqrt{m_{L-1}}}\|\widetilde w\|_1 .
\]
Thus, a continuous sampling distribution does not create an infinite optimization
problem. The randomness appears only in the construction of the finite
dictionary. After conditioning on the draw, the convex program optimizes only the
output coefficients and is globally optimal for the sampled spike dictionary.

\subsection{Reconstructing an LIF-SNN from the convex solution}
\label{app:lif_reconstruction}
The complete spike dictionary stores pairs $(d_i,\omega_i)$, where
$d_i\in\{0,1\}^n$ is a spike column and $\omega_i$ is at least one LIF witness
that realizes it on the training set. Therefore, after solving the finite convex
program, no inverse realization problem is needed: each active coefficient
retrieves a stored witness and becomes one parallel LIF subnetwork.
If witness metadata is discarded and only projected spike columns are stored,
reconstruction becomes a separate trajectory-realization problem. One must
recover membrane states and LIF parameters that reproduce the selected spike
trajectory across all layers and timesteps. This problem is substantially harder
than output-layer convex reconstruction and can become nonconvex when leak
parameters are optimized, because the recurrence contains products between
membrane variables and leak coefficients. Our theory avoids this inverse problem
by defining the complete dictionary as witnessed.
\begin{algorithm}[H]
\caption{Witnessed LIF-SNN Reconstruction from Convex Solution}
\label{alg:lif_reconstruction}
\KwIn{spike dictionary
$D_{\mathrm{SNN}}^{L-1,T}=[d_1,\ldots,d_M]$ with witness map
$i\mapsto\omega_i$; convex solution
$\widetilde w^*\in\mathbb{R}^{M}$; final hidden width $m_{L-1}$}
\KwOut{$K$-parallel LIF-SNN parameters
$\Theta=\{(\omega_k,P_{k,\mathrm{out}})\}_{k=1}^K$ realizing the convex predictor on the training set}
Set $\mathcal S\gets\{i:\widetilde w_i^*\neq 0\}$\;
\tcp{There exists an equivalent scalar-output solution with $|\mathcal S|\le n+1$.}
Set $K\gets|\mathcal S|$\;
\For{$k=1,\ldots,K$}{
    Let $i_k$ be the $k$-th active index in $\mathcal S$\;
    Retrieve the stored witness $\omega_k\gets\omega_{i_k}$\;
    Set subnetwork $k$'s hidden LIF parameters to $\omega_k$\;
    Set
    $$
    P_{k,\mathrm{out}}
    \gets
    \frac{\widetilde w_{i_k}^*}{m_{L-1}}\mathbf 1
    \in\mathbb{R}^{m_{L-1}}.
    $$
}
\Return{$\Theta=\{(\omega_k,P_{k,\mathrm{out}})\}_{k=1}^K$}\;
\end{algorithm}

\subsection{Surrogate-gradient pretraining: \textsc{SG}}
\label{app:alg_lif_pt}
\textsc{SG} is the standard surrogate-gradient LIF-SNN baseline. It trains
the hidden LIF dynamics and the output layer jointly by backpropagation through
time using a surrogate derivative for the spike activation.
\begin{algorithm}[H]
\caption{\textsc{SG}}
\label{alg:lif_pt}
\KwIn{Training data $\{X^t\}_{t=1}^T$, labels $Y$, LIF-SNN architecture, surrogate loss $\mathcal L_{\mathrm{sg}}$, number of epochs $E$}
\KwOut{Pretrained LIF-SNN parameters $\Theta_{\mathrm{LIF}}$}
Initialize hidden parameters
$\{P_{l,\mathrm{in}},B_l,R_l\}_{l=1}^{L-1}$
and output weights $P_{\mathrm{out}}$\;
\For{$e=1,\ldots,E$}{
    Initialize membrane and spike states $\{U_l^0,S_l^0\}_{l=1}^{L-1}$\;
    \For{$t=1,\ldots,T$}{
        \For{$l=1,\ldots,L-1$}{
            Compute
            $$
            U_l^t
            =
            S_{l-1}^tP_{l,\mathrm{in}}
            +
            U_l^{t-1}B_l
            -
            S_l^{t-1}R_l.
            $$
            Compute spikes $S_l^t\gets\sigma(U_l^t)$ using the surrogate-gradient spike function\;
        }
    }
    Compute prediction $f_{\Theta}(X^{1:T})$ using the chosen readout rule\;
    Compute surrogate-gradient loss
    $\mathcal L_{\mathrm{sg}}(f_{\Theta}(X^{1:T}),Y)$\;
    Update all trainable parameters $\Theta$ by BPTT with surrogate gradients\;
}
\Return{$\Theta_{\mathrm{LIF}}$}\;
\end{algorithm}

\subsection{Gaussian \texorpdfstring{spike}{spike} dictionaries: \textsc{CVX}}
\label{app:alg_cvx_g}
\textsc{CVX} follows the Convex-Lasso approximation of
\citet{ergen_globally_2023}. Instead of enumerating the complete witnessed
dictionary, we sample a finite set of hidden LIF witnesses and solve the convex
problem on the induced spike dictionary. For each layer $l$ and subnetwork $k$,
the sampled witness consists of input-to-membrane weights and a LIF-structured
recurrent block,
\[
G_{l,k,\mathrm{in}}\in\mathbb{R}^{m_{l-1}\times m_{l}},
\qquad
G_{l,k,\mathrm{rec}}
=
\begin{bmatrix}
B_{l,k}\\
-R_{l,k}
\end{bmatrix}.
\]
The same witness is reused across all timesteps; recurrent weights are not
resampled independently at each timestep. Related capacity and separation
results motivate Gaussian witnesses as high-capacity random features
\citep{vershynin_2020,dirksen_etal_2022}, but they do not prove completeness of
Gaussian recurrent or LIF-SNN dictionaries.
\begin{algorithm}[H]
\caption{\textsc{CVX}}
\label{alg:cvx_g}
\KwIn{Training data $\{X^t\}_{t=1}^T$, labels $Y$, widths $\{m_{l}\}$, number of subnetworks $K$, regularization $\beta_{\mathrm{reg}}$}
\KwOut{Gaussian-sampled dictionary $\widehat D_{\mathrm{G}}$ and convex output weights $\widetilde W$}
\For{$k=1,\ldots,K$}{
    Set $S_{0,k}^t\gets X^t$ for all $t\in[T]$\;
    \For{$l=1,\ldots,L-1$}{
        Sample and fix Gaussian input weights
        $G_{l,k,\mathrm{in}}\in\mathbb{R}^{m_{l-1}\times m_{l}}$\;
        Sample or fix LIF-structured recurrent parameters
        $B_{l,k}=\operatorname{Diag}(\beta_{l,k})$ and
        $R_{l,k}=\operatorname{Diag}(U_{\mathrm{thr},l,k})$\;
        Initialize $U_{l,k}^0\gets 0$ and $S_{l,k}^0\gets 0$\;
        \For{$t=1,\ldots,T$}{
            Compute
            $$
            U_{l,k}^t
            =
            S_{l-1,k}^tG_{l,k,\mathrm{in}}
            +
            U_{l,k}^{t-1}B_{l,k}
            -
            S_{l,k}^{t-1}R_{l,k}.
            $$
            Compute $S_{l,k}^t\gets\sigma(U_{l,k}^t)$\;
        }
    }
    Store final-layer spike features $\{S_{L-1,k}^t\}_{t=1}^T$\;
}
Construct the sampled spike dictionary $\widehat D_{\mathrm{G}}$ from the stored final-layer spike features\;
Solve
$$
\widetilde W
\gets
\arg\min_W\
\mathcal L(\widehat D_{\mathrm{G}}W,Y)
+
\beta_{\mathrm{reg}}\|W\|_1 .
$$
\Return{$\widehat D_{\mathrm{G}},\widetilde W$}\;
\end{algorithm}
\subsection{Surrogate-pretrained spike dictionaries: \textsc{SG-CVX}}
\label{app:alg_cvx_from_lif}
\textsc{SG-CVX} tests whether surrogate-gradient training can
produce a more useful spike dictionary than unstructured Gaussian sampling.
We first train a LIF-SNN on the same training distribution, extract its hidden
dynamics, freeze them, and solve convex reconstruction over the resulting spike
dictionary. We do not claim that surrogate-pretrained witnesses have higher
separation capacity than Gaussian witnesses; empirically, they may be better
aligned with the task distribution and more stable under length extrapolation.
\begin{algorithm}[H]
\caption{\textsc{SG-CVX}}
\label{alg:cvx_from_lif}
\KwIn{Pretrained LIF-SNN parameters $\Theta_{\mathrm{LIF}}$, training data $\{X^t\}_{t=1}^T$, labels $Y$, regularization $\beta_{\mathrm{reg}}$}
\KwOut{Pretrained LIF dictionary $\widehat D_{\mathrm{LIF}}$ and convex output weights $\widetilde W$}
Extract hidden LIF parameters
$\{\widehat P_{l,\mathrm{in}},\widehat B_l,\widehat R_l\}_{l=1}^{L-1}$
from $\Theta_{\mathrm{LIF}}$\;
Freeze the extracted hidden dynamics\;
\For{$k=1,\ldots,K$}{
    Set $S_{0,k}^t\gets X^t$ for all $t\in[T]$\;
    \For{$l=1,\ldots,L-1$}{
        Assign the frozen pretrained hidden block
        $(\widehat P_{l,k,\mathrm{in}},\widehat B_{l,k},\widehat R_{l,k})$\;
        Initialize $U_{l,k}^0\gets 0$ and $S_{l,k}^0\gets 0$\;
        \For{$t=1,\ldots,T$}{
            Compute
            $$
            U_{l,k}^t
            =
            S_{l-1,k}^t\widehat P_{l,k,\mathrm{in}}
            +
            U_{l,k}^{t-1}\widehat B_{l,k}
            -
            S_{l,k}^{t-1}\widehat R_{l,k}.
            $$
            Compute $S_{l,k}^t\gets\sigma(U_{l,k}^t)$\;
        }
    }
    Store final-layer spike features $\{S_{L-1,k}^t\}_{t=1}^T$\;
}
Construct the pretrained spike dictionary $\widehat D_{\mathrm{LIF}}$ from the stored final-layer spike features\;
Solve
$$
\widetilde W
\gets
\arg\min_W\
\mathcal L(\widehat D_{\mathrm{LIF}}W,Y)
+
\beta_{\mathrm{reg}}\|W\|_1 .
$$
\Return{$\widehat D_{\mathrm{LIF}},\widetilde W$}\;
\end{algorithm}

\subsection{Surrogate-gradient finetuning from LIF: \textsc{SG-SG}}
\label{app:alg_lif_ft_from_lif}
\textsc{SG-SG} is the standard finetuning baseline initialized
from a surrogate-pretrained LIF-SNN. It is included to separate gains from
pretraining alone from gains due to convex reconstruction.
\begin{algorithm}[H]
\caption{\textsc{SG-SG}}
\label{alg:lif_ft_from_lif}
\KwIn{Pretrained LIF-SNN parameters $\Theta_{\mathrm{LIF}}$, finetuning data $\{X_{\mathrm{ft}}^t\}_{t=1}^T$, labels $Y_{\mathrm{ft}}$, surrogate loss $\mathcal L_{\mathrm{sg}}$, finetuning epochs $E_{\mathrm{ft}}$}
\KwOut{Finetuned LIF-SNN parameters $\Theta_{\mathrm{FT}\leftarrow\mathrm{LIF}}$}
Initialize $\Theta\gets\Theta_{\mathrm{LIF}}$\;
\For{$e=1,\ldots,E_{\mathrm{ft}}$}{
    Initialize membrane and spike states $\{U_l^0,S_l^0\}_{l=1}^{L-1}$\;
    \For{$t=1,\ldots,T$}{
        Roll out the LIF-SNN on $X_{\mathrm{ft}}^t$\;
    }
    Compute prediction $f_{\Theta}(X_{\mathrm{ft}}^{1:T})$\;
    Compute surrogate-gradient loss
    $\mathcal L_{\mathrm{sg}}(f_{\Theta}(X_{\mathrm{ft}}^{1:T}),Y_{\mathrm{ft}})$\;
    Update all trainable parameters using surrogate gradients\;
}
Set $\Theta_{\mathrm{FT}\leftarrow\mathrm{LIF}}\gets\Theta$\;
\Return{$\Theta_{\mathrm{FT}\leftarrow\mathrm{LIF}}$}\;
\end{algorithm}

\subsection{Surrogate-gradient finetuning from CVX: \textsc{CVX-SG}}
\label{app:alg_lif_ft_from_cvx}
\textsc{CVX-SG} tests the reverse direction: whether a convex
output-layer solution gives a useful initialization for subsequent
surrogate-gradient training. Since the finetuning stage updates hidden recurrent
weights, this method is an empirical hybrid rather than a convex optimization
method.
\begin{algorithm}[H]
\caption{\textsc{CVX-SG}}
\label{alg:lif_ft_from_cvx}
\KwIn{\textsc{CVX} sampled hidden witnesses, convex output weights $\widetilde W$, finetuning data $\{X_{\mathrm{ft}}^t\}_{t=1}^T$, labels $Y_{\mathrm{ft}}$, surrogate loss $\mathcal L_{\mathrm{sg}}$, finetuning epochs $E_{\mathrm{ft}}$}
\KwOut{Finetuned LIF-SNN parameters $\Theta_{\mathrm{FT}\leftarrow\mathrm{CVX}}$}
Initialize hidden LIF weights using the sampled dictionary witnesses from \textsc{CVX}\;
Initialize the output layer using the convex solution $\widetilde W$\;
\For{$e=1,\ldots,E_{\mathrm{ft}}$}{
    Initialize membrane and spike states $\{U_l^0,S_l^0\}_{l=1}^{L-1}$\;
    \For{$t=1,\ldots,T$}{
        Roll out the initialized LIF-SNN on $X_{\mathrm{ft}}^t$\;
    }
    Compute prediction $f_{\Theta}(X_{\mathrm{ft}}^{1:T})$\;
    Compute surrogate-gradient loss
    $\mathcal L_{\mathrm{sg}}(f_{\Theta}(X_{\mathrm{ft}}^{1:T}),Y_{\mathrm{ft}})$\;
    Update all trainable parameters, including hidden recurrent weights and output weights\;
}
Set $\Theta_{\mathrm{FT}\leftarrow\mathrm{CVX}}\gets\Theta$\;
\Return{$\Theta_{\mathrm{FT}\leftarrow\mathrm{CVX}}$}\;
\end{algorithm}

\subsection{Criticality-tuned reservoir: \textsc{R}}
\label{app:alg_r}
\textsc{R} is the analog reservoir baseline. The hidden LIF dynamics are fixed and
tuned to criticality, and only a ridge-regression output layer is trained on the
continuous membrane potential at the final layer. No spike dictionary is
constructed; \textsc{R} is an analog reference rather than a spiking
reconstruction. The reservoir is tuned by sweeping the leak $\beta_l$, threshold
$U_{\mathrm{thr}}^l$, and input scale so that the branching ratio $\sigma$,
estimated by lag-one regression of successive population activity, is close to
one. The architecture contains no trainable recurrent weight matrix; the only
recurrence is the LIF leak, so criticality is measured by the branching ratio
rather than a spectral radius.
\begin{algorithm}[H]
\caption{\textsc{R}}
\label{alg:r}
\KwIn{Training data $\{X^t\}_{t=1}^T$, labels $Y$, criticality grid over $(\beta,U_{\mathrm{thr}},\text{input scale})$, ridge parameter $\lambda_{\mathrm{ridge}}$}
\KwOut{Tuned reservoir witness $\omega^\star$ and ridge output weights $W_{\mathrm{ridge}}$}
\For{each grid setting $(\beta,U_{\mathrm{thr}},\text{input scale})$}{
    Fix the LIF dynamics with these parameters\;
    Roll the reservoir over $\{X^t\}_{t=1}^T$ and record the membrane trajectory\;
    Estimate the branching ratio $\sigma$ by lag-one regression of successive population activity\;
}
Select the witness $\omega^\star$ whose $\sigma$ is closest to $1$\;
Roll $\omega^\star$ to obtain the final-layer membrane states $\{U_{L-1}^t\}_{t=1}^T$\;
Solve the ridge problem
$$
W_{\mathrm{ridge}}
\gets
\arg\min_W\
\mathcal L(U_{L-1}^T W,Y)
+
\lambda_{\mathrm{ridge}}\|W\|_2^2 .
$$
\Return{$\omega^\star,W_{\mathrm{ridge}}$}\;
\end{algorithm}

\subsection{Reservoir spike dictionaries: \textsc{R-CVX}}
\label{app:alg_r_cvx}
\textsc{R-CVX} uses the same criticality-tuned reservoir witness as \textsc{R},
but reads out the spike dictionary rather than the analog membrane. The tuned
reservoir is frozen and rolled through the input sequence, the induced spike
dictionary is constructed exactly as in \textsc{CVX}, and the convex output-layer
program is solved over it. \textsc{R-CVX} therefore differs from \textsc{R} only
in the readout: spikes and an $\ell_1$ convex reconstruction in place of the
analog membrane and ridge regression.
\begin{algorithm}[H]
\caption{\textsc{R-CVX}}
\label{alg:r_cvx}
\KwIn{Training data $\{X^t\}_{t=1}^T$, labels $Y$, tuned reservoir witness $\omega^\star$ from \textsc{R}, regularization $\beta_{\mathrm{reg}}$}
\KwOut{Reservoir spike dictionary $\widehat D_{\mathrm{R}}$ and convex output weights $\widetilde W$}
Freeze the reservoir witness $\omega^\star$\;
Roll $\omega^\star$ over $\{X^t\}_{t=1}^T$ and record the final-layer spikes $\{S_{L-1}^t\}_{t=1}^T$\;
Construct the sampled spike dictionary $\widehat D_{\mathrm{R}}$ from the recorded spikes\;
Solve
$$
\widetilde W
\gets
\arg\min_W\
\mathcal L(\widehat D_{\mathrm{R}}W,Y)
+
\beta_{\mathrm{reg}}\|W\|_1 .
$$
\Return{$\widehat D_{\mathrm{R}},\widetilde W$}\;
\end{algorithm}

\subsection{Reservoir analog reconstruction: \textsc{R-CVX-an}}
\label{app:alg_r_cvx_an}
\textsc{R-CVX-an} is a control that isolates the effect of the spike readout. It
uses the same criticality-tuned reservoir witness and the same $\ell_1$ convex
program as \textsc{R-CVX}, but reconstructs from the continuous membrane potential
rather than the spike dictionary. Comparing \textsc{R-CVX-an} with \textsc{R-CVX}
holds the reservoir and the convex readout fixed and varies only whether the
final representation is analog or thresholded, isolating the cost of the spike
nonlinearity from the choice of convex versus ridge readout.
\begin{algorithm}[H]
\caption{\textsc{R-CVX-an}}
\label{alg:r_cvx_an}
\KwIn{Training data $\{X^t\}_{t=1}^T$, labels $Y$, tuned reservoir witness $\omega^\star$ from \textsc{R}, regularization $\beta_{\mathrm{reg}}$}
\KwOut{Analog feature matrix $\widehat U_{\mathrm{R}}$ and convex output weights $\widetilde W$}
Freeze the reservoir witness $\omega^\star$\;
Roll $\omega^\star$ over $\{X^t\}_{t=1}^T$ and record the final-layer membrane states $\{U_{L-1}^t\}_{t=1}^T$\;
Form the analog feature matrix $\widehat U_{\mathrm{R}}$ from the final-timestep membrane $U_{L-1}^T$\;
Solve
$$
\widetilde W
\gets
\arg\min_W\
\mathcal L(\widehat U_{\mathrm{R}}W,Y)
+
\beta_{\mathrm{reg}}\|W\|_1 .
$$
\Return{$\widehat U_{\mathrm{R}},\widetilde W$}\;
\end{algorithm}

\section{Training accuracy}
\label{app:train_tables}
Table~\ref{tab:train_summary} reports last-step training accuracy (mean $\pm$ std over 3 seeds) of the best variant in each method class, alongside the reservoir ridge baseline \textsc{R}. It supports the analysis of \S\ref{sec:analysis}: the convex family reaches the training ceiling wherever the sampled dictionary is expressive, and falls short only at deep \textsc{MNIST-Seq} ($L=10$) and long-horizon XOR ($T=14$), where either no method fits well or the reservoir fits best.

Settings for both tasks are given in Table~\ref{tab:settings}. Both use $K=2$ parallel subnetworks, far below the Carath\'eodory width bound $K^*\le n_{\mathrm{train}}+1$ (2001 for XOR, 6001 for MNIST), so all runs are in the under-parameterized regime.

\begin{table}[t]
  \centering
  \small
  \caption{Experimental settings for the two tasks. $K^*\le n_{\mathrm{train}}+1$ is the Carath\'eodory width bound; both tasks use $K=2 \ll K^*$.}
  \label{tab:settings}
  \setlength{\tabcolsep}{8pt}
  \renewcommand{\arraystretch}{1.15}
  \begin{tabular}{@{}lcc@{}}
    \toprule
    & \textsc{first-last-XOR} & \textsc{MNIST-Seq} \\
    \midrule
    $n_{\mathrm{train}}$              & 2000 & 6000 \\
    $n_{\mathrm{val}}$                & 2000 & 2000 \\
    $n_{\mathrm{test}}$               & 4000 & 2000 \\
    $K$ (parallel)                    & 2 & 2 \\
    $P_{\mathrm{rec}}$                & 500 (250/branch) & 512 (256/branch) \\
    $P_{\mathrm{last}}$               & 1000 (500/branch) & 1024 (512/branch) \\
    $T$                               & $\{6,8,11,14\}$ & $\{4,14,28,56\}$ \\
    $L$                               & $\{3,5,10,15\}$ & $\{3,5,10\}$ \\
    $K^*\le n_{\mathrm{train}}+1$     & 2001 & 6001 \\
    \bottomrule
  \end{tabular}
\end{table}

\begin{table}[t]
  \centering
  \small
  \caption{Last-step training accuracy (mean $\pm$ std, 3 seeds), best variant per class with the winning variant in brackets.
    Convex class is the best of \{\textsc{CVX}, \textsc{SG-CVX}, \textsc{R-CVX},
    \textsc{R-CVX-an}\}; SG class is the best of \{\textsc{SG}, \textsc{SG-SG}\};
    \textsc{R} is the reservoir ridge baseline.}
  \label{tab:train_summary}
  \setlength{\tabcolsep}{6pt}
  \renewcommand{\arraystretch}{1.15}
  \begin{tabular}{@{}llccc@{}}
    \toprule
    \textbf{Task} & \textbf{Setting} & \textbf{Best convex} & \textbf{Best SG} & \textbf{\textsc{R}} \\
    \midrule
    \textsc{first-last-XOR} & $L{=}5,T{=}6$   & \textbf{1.000$\pm$.000} (R-CVX) & 0.730$\pm$.031 & \textbf{1.000$\pm$.000} \\
    \textsc{first-last-XOR} & $L{=}15,T{=}6$  & \textbf{1.000$\pm$.000} (R-CVX) & 0.493$\pm$.007 & \textbf{1.000$\pm$.000} \\
    \textsc{first-last-XOR} & $L{=}5,T{=}14$  & \textbf{0.733$\pm$.015} (CVX)    & 0.554$\pm$.057 & 0.722$\pm$.007 \\
    \textsc{first-last-XOR} & $L{=}15,T{=}14$ & 0.679$\pm$.064 (CVX-an)          & 0.504$\pm$.007 & \textbf{0.731$\pm$.012} \\
    \midrule
    \textsc{MNIST-Seq} & $L{=}3,T{=}4$   & \textbf{0.996$\pm$.002} (SG-CVX) & 0.950$\pm$.019 & 0.685$\pm$.085 \\
    \textsc{MNIST-Seq} & $L{=}10,T{=}4$  & 0.389$\pm$.049 (CVX-an)          & 0.098$\pm$.004 & \textbf{0.481$\pm$.005} \\
    \textsc{MNIST-Seq} & $L{=}3,T{=}14$  & \textbf{0.982$\pm$.004} (SG-CVX) & 0.873$\pm$.007 & 0.727$\pm$.023 \\
    \textsc{MNIST-Seq} & $L{=}3,T{=}28$  & \textbf{0.989$\pm$.004} (SG-CVX) & 0.909$\pm$.008 & 0.722$\pm$.014 \\
    \textsc{MNIST-Seq} & $L{=}3,T{=}56$  & \textbf{0.984$\pm$.004} (SG-CVX) & 0.871$\pm$.002 & 0.726$\pm$.006 \\
    \bottomrule
  \end{tabular}
\end{table}

\section{Loss landscape and the role of $K$}
\label{app:k_landscape}

PThe main experiments use $K=2$. This appendix examines how the training landscape changes with $K$ and what the convex readout does to the geometry of the landscape. The results describe the landscape that different algorithms converge to.

\subsection{Evaluation settings}
\label{app:kl_settings}
All landscape runs use $L=3$ (two LIF layers and a readout), $T=4$, pooled MNIST with $d_{\mathrm{in}}=49$, $n_{\mathrm{train}}=100$, and a fixed per-branch width of $64$. We sweep $K\in\{1,10,50,100,150\}$; interpolation studies extend to $K=200$. We use two width conventions. In \emph{fixed per-branch width}, each subnetwork keeps width $64$ and the total width grows with $K$, matching the Carath\'eodory construction of Theorems~\ref{thm:finite_convex_L_T} and~\ref{thm:finite_convex_lif_snn}, where each subnetwork is one atom. In \emph{fixed total width}, the total width is held at $8000$ and split across $K$ branches, so per-branch width shrinks as $K$ grows; we use it only as a control that separates parallelism from parameter count.

\subsection{Two probes and their pitfalls}
\label{app:kl_probes}
The \emph{neighbourhood survey} perturbs the parameters around a point along filter-normalized random directions and records the training loss on the surveyed grid \citep{li_visualizing_2018}. The \emph{interpolation} probe evaluates the loss along the straight line between two independently trained solutions \citep{garipov_loss_2018, frankle_linear_2020}.

Both have known failure modes, which we control for. Threshold and LIF networks are scale-invariant, so unnormalized directions confound curvature with weight norm; we filter-normalize throughout \citep{li_visualizing_2018}. Random directions in high dimension tend to look convex regardless of the true geometry, so we read the survey only for relative comparisons across matched settings, never as evidence of global convexity. For interpolation, the straight-line barrier between two parallel networks is dominated by permutation symmetry: solutions equivalent up to a relabeling of subnetworks show a large raw barrier that is an alignment artifact \citep{entezari_role_2021}, so we report raw and permutation-aligned barriers separately and do not treat the raw barrier as a benignity measure. Finally, a bowl in a survey is not automatically a minimum. Where the fit gradient is not near zero, the regularizer can still shrink weight norm along a flat, equal-loss valley, so we check the gradient rather than read minimality off the picture (Table~\ref{tab:grad}). For the convex-cone families at working $K$ the gradient is near zero and the centre is a genuine near-stationary minimum; we flag the cases where it is not.

\subsection{Metrics}
\label{app:kl_metrics}
All metrics are on the training loss with filter-normalized directions. \emph{Jaggedness} (jag) is the mean second difference over the mean first difference along the surveyed lines (lower is smoother). \emph{Sharpness} is the largest Hessian eigenvalue $\lambda_{\max}$ in a small ball, filter-normalized (lower is flatter). \emph{Negative-curvature fraction} (negc) is the fraction of local frames with $\lambda_{\min}<0$ (near one is a saddle, zero a locally convex neighbourhood). \emph{Basin radius} (basinR) is the half-line radius over which the loss stays within a tolerance of the centre. For interpolation we report the maximum training-CE barrier, raw ($\mathrm{bar}_R$) and permutation-aligned ($\mathrm{bar}_A$).

\begin{figure}[t]
\centering
\includegraphics[width=\textwidth]{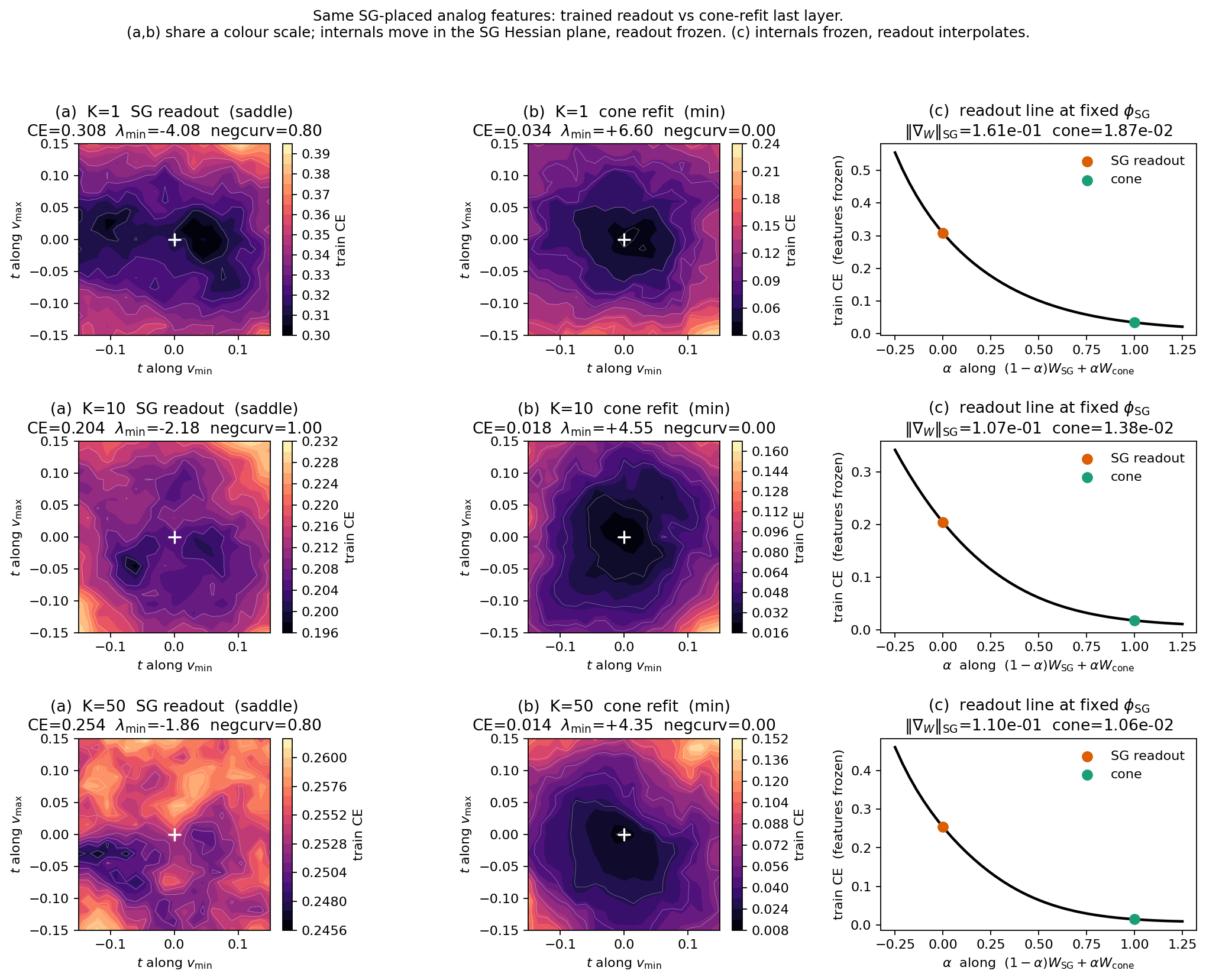}
\caption{Readout rescue on the same surrogate-gradient features. (a)~The surrogate-gradient readout sits at a saddle of the training loss ($\lambda_{\min}<0$). (b)~The convex cone readout over the same features has no negative-curvature direction and reaches lower CE. (c)~Interpolating the readout from $W_{\mathrm{SG}}$ to $W_{\mathrm{cone}}$ at frozen features is monotone down. Internals and readout are perturbed jointly along filter-normalized directions.}
\label{fig:rescue}
\end{figure}

\begin{table}[t]
\centering
\small
\caption{Readout rescue on the same surrogate-gradient features $\phi_{\mathrm{SG}}$ (per-branch width $64$). Only the last layer is swapped: the surrogate-gradient readout $W_{\mathrm{SG}}$ versus the convex cone solution $W_{\mathrm{cone}}$. Curvature is measured in the joint (internal $+$ readout) space along filter-normalized directions.}
\label{tab:readout_rescue}
\setlength{\tabcolsep}{5pt}
\begin{tabular}{@{}rccccccc@{}}
\toprule
$K$ & SG CE & cone CE & SG $\lambda_{\min}$ & cone $\lambda_{\min}$ & SG negc & cone negc & $\|\nabla_W\|$ SG / cone \\
\midrule
1  & 0.308 & 0.034 & $-4.08$ & $+6.60$ & 0.80 & 0.00 & $1.6\!\times\!10^{-1}$ / $1.9\!\times\!10^{-2}$ \\
10 & 0.204 & 0.018 & $-2.18$ & $+4.55$ & 1.00 & 0.00 & $1.1\!\times\!10^{-1}$ / $1.4\!\times\!10^{-2}$ \\
50 & 0.254 & 0.014 & $-1.86$ & $+4.35$ & 0.80 & 0.00 & $1.1\!\times\!10^{-1}$ / $1.1\!\times\!10^{-2}$ \\
\bottomrule
\end{tabular}
\end{table}

Plain surrogate gradients do not converge on the spike readout at all. Table~\ref{tab:plainsg_nbhd} shows that applying the trained readout to the last-layer spikes gives a point that stays a saddle as $K$ grows ($\mathrm{negc}=1$ for $K\ge 10$) while accuracy collapses from $0.56$ to $0.08$; the spike density drops from $0.09$ to $10^{-5}$, so the mismatch worsens with $K$. Figure~\ref{fig:plainsg_landscape} shows the saddle directly.

\begin{figure}[t]
\centering
\includegraphics[width=0.9\textwidth]{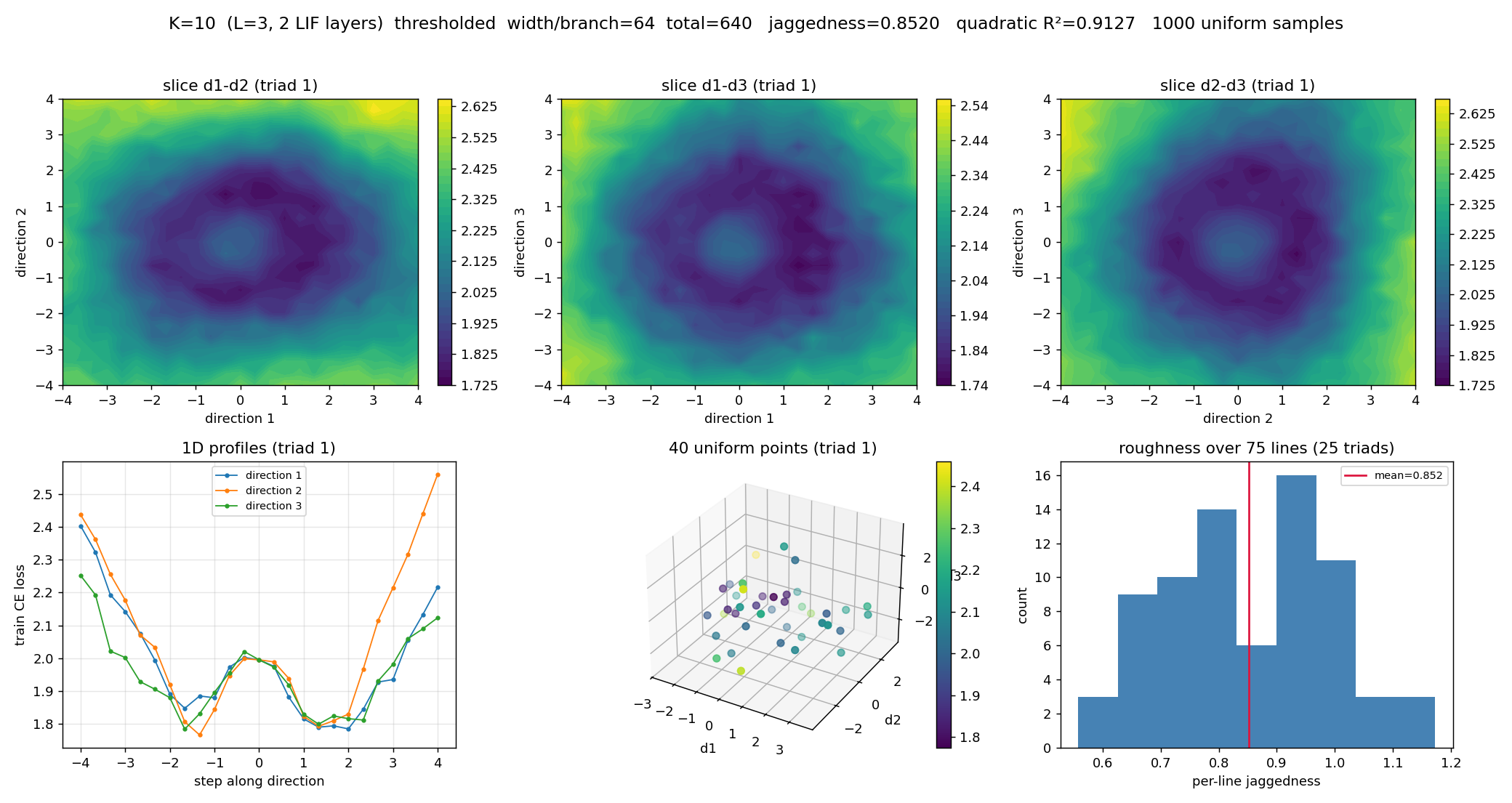}
\caption{Plain surrogate-gradient landscape on the last-layer spike readout ($K=10$, per-branch width $64$). The 1D profiles descend on both sides of the centre, so the centre is a saddle, not a minimum.}
\label{fig:plainsg_landscape}
\end{figure}

\begin{table}[t]
\centering
\small
\caption{Plain surrogate-gradient neighbourhood on the last-layer spike readout, $r=4$, per-branch width $64$, versus $K$. The point is a saddle for $K\ge 10$ ($\mathrm{negc}=1$, $\lambda_{\min}<0$) and accuracy collapses, confirming surrogate gradients do not converge on the spike readout. The cone solution over the same internals fits these spikes at accuracy $1.0$ with CE $0.038\to 0.003$.}
\label{tab:plainsg_nbhd}
\setlength{\tabcolsep}{6pt}
\begin{tabular}{@{}rcccccc@{}}
\toprule
$K$ & CE & train acc & test acc & jag & $\lambda_{\min}$ & negc \\
\midrule
1   & 1.748 & 0.56 & 0.386 & 0.905 & $+0.004$ & 0.50 \\
10  & 1.995 & 0.48 & 0.352 & 0.852 & $-0.412$ & 1.00 \\
50  & 2.272 & 0.18 & 0.119 & 0.795 & $-0.444$ & 1.00 \\
100 & 2.296 & 0.12 & 0.089 & 0.857 & $-0.233$ & 1.00 \\
150 & 2.301 & 0.08 & 0.097 & 0.832 & $-0.178$ & 1.00 \\
\bottomrule
\end{tabular}
\end{table}

\subsection{Stationarity of the convex readout}
\label{app:kl_stationary}
The curvature comparison above shows the convex readout removes the negative-curvature direction. The gradient tells the same story to first order. Table~\ref{tab:grad} reports the gradient of the training objective at each centre. The convex-cone readout is near-stationary at every family and working $K$: both the internal and readout gradient norms sit at $10^{-2}$ to $10^{-3}$, at a fit of accuracy~$1.0$. The residual is the weight-decay term acting along a flat, equal-loss valley, not a direction that lowers the fit.

Two families do not reach a stationary point. Plain surrogate gradients keep a large internal gradient ($\|\nabla_{\mathrm{int}}\|\approx 0.4$--$0.5$) and never fit, so the centre is a saddle rather than a minimum, consistent with the negative curvature in Table~\ref{tab:readout_rescue}. The reservoir ridge readout fits the labels (accuracy~$1.0$) but keeps a large readout gradient that grows with $K$ ($\|\nabla_W\|$ from $0.58$ to $1.66$); the ridge solution minimizes squared error, not the training objective measured here, so it is not stationary for that objective. The convex readout over the same reservoir features is near-stationary ($\|\nabla_W\|\approx 0.01$). The convex readout therefore reaches an optimum over the fixed features that neither surrogate gradients nor a ridge readout attains.

\begin{table}[t]
\centering
\small
\caption{Gradient norm and fit at the centre point, per family and $K$ (fixed per-branch width $64$, spike readout, seed~0). $\|\nabla_{\mathrm{int}}\|$ and $\|\nabla_W\|$ are the internal and readout gradient norms of the training objective. The cone readout is near-stationary ($\|\nabla\|\sim10^{-2}$--$10^{-3}$) at accuracy~$1.0$. The reservoir ridge readout fits the labels but keeps a large readout gradient, so it is not stationary for its objective. \textsc{SG-plain} does not fit and is not stationary.}
\label{tab:grad}
\setlength{\tabcolsep}{5pt}
\begin{tabular}{@{}llcccc@{}}
\toprule
Family & $K$ & $\|\nabla_{\mathrm{int}}\|$ & $\|\nabla_W\|$ & CE & acc \\
\midrule
\multirow{3}{*}{Gaussian-cone}
 & 10  & 0.013 & 0.015 & 0.012 & 1.00 \\
 & 50  & 0.003 & 0.012 & 0.004 & 1.00 \\
 & 150 & 0.002 & 0.011 & 0.002 & 1.00 \\
\midrule
\multirow{3}{*}{SG-cone}
 & 10  & 0.009 & 0.016 & 0.008 & 1.00 \\
 & 50  & 0.003 & 0.013 & 0.004 & 1.00 \\
 & 150 & 0.004 & 0.012 & 0.003 & 1.00 \\
\midrule
\multirow{3}{*}{LSM-cone}
 & 10  & 0.166 & 0.011 & 0.040 & 0.99 \\
 & 50  & 0.012 & 0.009 & 0.005 & 1.00 \\
 & 100 & 0.007 & 0.008 & 0.003 & 1.00 \\
\midrule
\multirow{3}{*}{LSM-ridge}
 & 10  & 0.267 & 0.576 & 1.084 & 0.99 \\
 & 50  & 0.134 & 1.183 & 0.855 & 1.00 \\
 & 100 & 0.096 & 1.659 & 0.826 & 1.00 \\
\midrule
\multirow{3}{*}{SG-plain}
 & 10  & 0.465 & 0.269 & 1.995 & 0.48 \\
 & 50  & 0.481 & 0.106 & 2.272 & 0.18 \\
 & 150 & 0.422 & 0.034 & 2.301 & 0.08 \\
\bottomrule
\end{tabular}
\end{table}

\subsection{Neighbourhood geometry versus $K$}
\label{app:kl_nbhd}
Table~\ref{tab:gauss_nbhd} reports the Gauss+cone neighbourhood (analog readout) against $K$ at fixed per-branch width. The centre is a last-layer solve, not a trained internal minimum, and has no negative-curvature direction at any $K$ ($\mathrm{negc}=0$). As $K$ grows the neighbourhood flattens (sharpness $14.3\to 9.5$) and the basin widens (basinR $0.03\to 0.22$); jaggedness is roughly flat. The surrogate-pretrained (SG+cone) family behaves similarly, and the reservoir (LSM) families are rougher throughout (jaggedness $\approx 1.0$--$1.3$ on spikes). We read these as relative comparisons under random-direction surveys, not as convexity, and we note that the flatter neighbourhood at larger $K$ does not translate into higher held-out accuracy on this small ($n=100$) set. Figure~\ref{fig:nbhd_K} shows the flattening across $K$.

\begin{figure}[t]
\centering
\begin{subfigure}{0.49\textwidth}
\centering
\includegraphics[width=\textwidth]{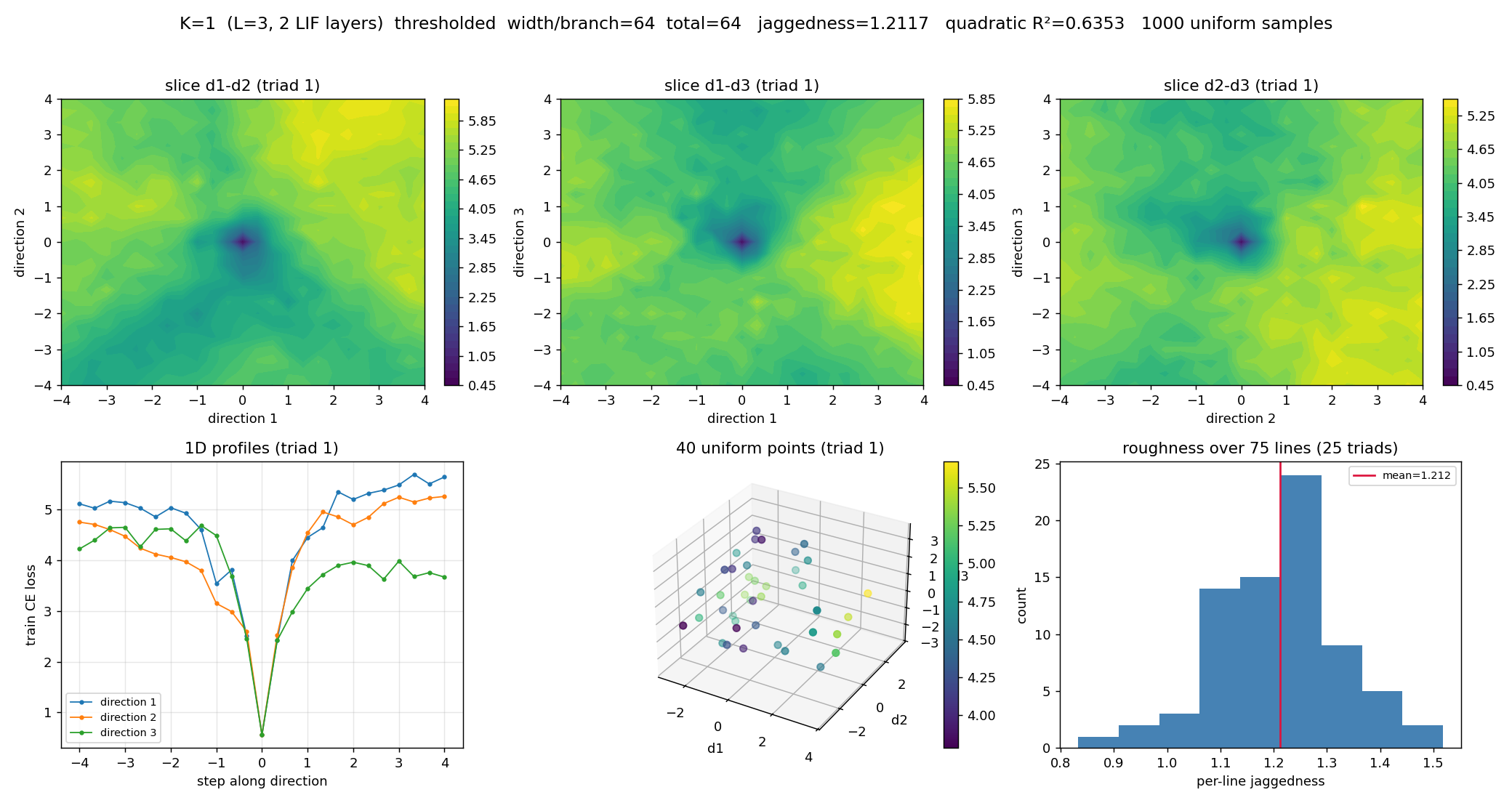}
\caption{$K=1$}
\end{subfigure}
\hfill
\begin{subfigure}{0.49\textwidth}
\centering
\includegraphics[width=\textwidth]{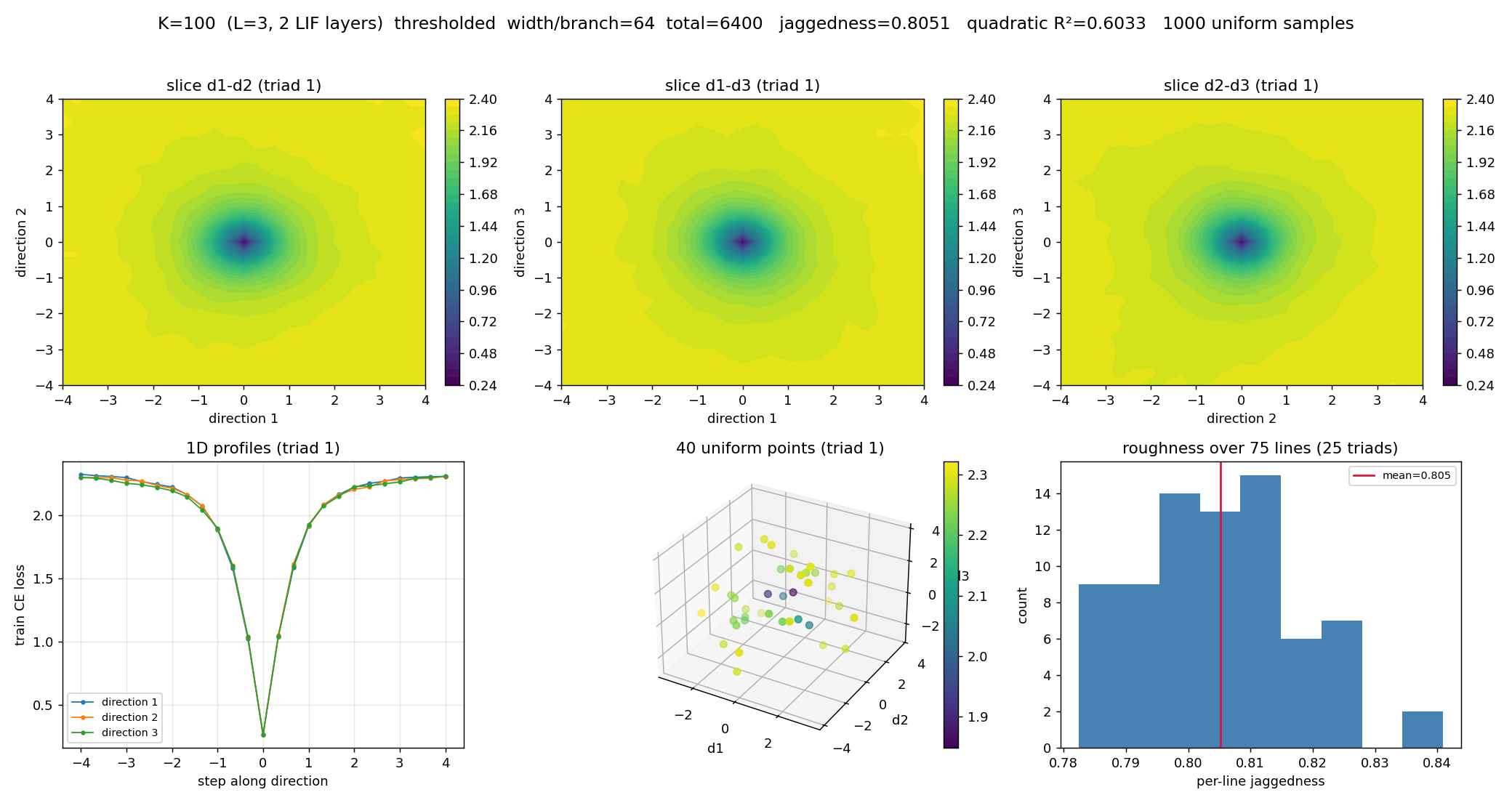}
\caption{$K=100$}
\end{subfigure}
\caption{Gauss+cone neighbourhood at $K=1$ and $K=100$, $r=4$, per-branch width $64$. The basin widens and the surface flattens as $K$ grows, while the centre remains free of negative-curvature directions.}
\label{fig:nbhd_K}
\end{figure}
\begin{table}[t]
\centering
\small
\caption{Gauss+cone neighbourhood (analog readout), $r=4$, per-branch width $64$, versus $K$. No negative-curvature direction at any $K$ ($\mathrm{negc}=0$); as $K$ grows the neighbourhood flattens (sharpness down) and the basin widens (basinR up).}
\label{tab:gauss_nbhd}
\setlength{\tabcolsep}{6pt}
\begin{tabular}{@{}rccccc@{}}
\toprule
$K$ & train acc & test acc & jag & sharp $\lambda_{\max}$ & basinR \\
\midrule
1   & 0.72 & 0.320 & 0.215 & 14.26 & 0.033 \\
10  & 1.00 & 0.491 & 0.217 & 15.00 & 0.096 \\
50  & 1.00 & 0.594 & 0.206 & 11.94 & 0.162 \\
100 & 1.00 & 0.626 & 0.214 & 10.06 & 0.216 \\
150 & 1.00 & 0.609 & 0.209 &  9.45 & 0.224 \\
\bottomrule
\end{tabular}
\end{table}

\subsection{Reservoir readout: ridge versus cone}
\label{app:kl_lsm}
On the same untrained criticality-tuned liquid, the convex cone readout reaches an order-of-magnitude lower training CE than the ridge readout (Table~\ref{tab:lsm_ridge_cone}), and both are last-layer minima ($\mathrm{negc}=0$); this is not a saddle-to-minimum flip, only a better last-layer fit. On this small set the lower training CE does not help generalization: the cone readout overfits, with lower held-out accuracy than ridge (analog test accuracy $0.60/0.68/0.66$ for cone versus $0.69/0.76/0.76$ for ridge at $K=10,50,100$).

\begin{table}[t]
\centering
\small
\caption{Reservoir readout on the same criticality-tuned liquid $\phi_{\mathrm{LSM}}$ (untrained internals, analog readout), per-branch width $64$. Ridge versus convex cone. The cone reaches far lower training CE; both are last-layer minima. On this $n=100$ set the cone overfits (lower test accuracy than ridge).}
\label{tab:lsm_ridge_cone}
\setlength{\tabcolsep}{6pt}
\begin{tabular}{@{}rcccc@{}}
\toprule
$K$ & ridge CE & cone CE & ridge test acc & cone test acc \\
\midrule
10  & 0.963 & 0.014 & 0.688 & 0.598 \\
50  & 0.828 & 0.005 & 0.760 & 0.675 \\
100 & 0.797 & 0.003 & 0.759 & 0.660 \\
\bottomrule
\end{tabular}
\end{table}

\subsection{Interpolation barriers}
\label{app:kl_interp}
Table~\ref{tab:kl_interp} reports the maximum training-CE barrier between independently trained solutions, raw and permutation-aligned. The raw barrier rises with $K$ and is dominated by permutation symmetry between the two networks \citep{entezari_role_2021}; after alignment it is much smaller. Under fixed total width the raw barrier peaks near $K=250$ and then falls only because each branch is starved to width two, not from parallelism. We report these for completeness and do not read the raw barrier as a benignity signal.

\begin{table}[t]
\centering
\small
\caption{Interpolation barriers (max training-CE along the path), per-branch width $64$. Raw ($\mathrm{bar}_R$) and permutation-aligned ($\mathrm{bar}_A$), for the SG-CVX (cone-on-SG-features) family. The raw barrier grows with $K$; alignment removes most of it.}
\label{tab:kl_interp}
\setlength{\tabcolsep}{7pt}
\begin{tabular}{@{}rccc@{}}
\toprule
$K$ & readout & $\mathrm{bar}_R$ (raw) & $\mathrm{bar}_A$ (aligned) \\
\midrule
1   & bits & 1.571 & 0.460 \\
10  & bits & 1.612 & 0.373 \\
50  & bits & 2.094 & 0.531 \\
100 & bits & 1.974 & 1.101 \\
150 & bits & 2.962 & 1.263 \\
200 & bits & 3.215 & 1.635 \\
\bottomrule
\end{tabular}
\end{table}
\section{Additional Results, Ablations, and Loss Landscape}
% Moved verbatim from main-text Section 5 to the appendix (AISTATS Item 6).
\label{sec:analysis}
The main results show the convex solve is exact and performance is governed by the witness source. This section ablates the factors that decide when the sampled dictionary generalizes, and examines the local geometry each method converges to.

% Moved verbatim from the main-text Experiments section (AISTATS Item 6).
\subsection{Additional Results from the Experiments Section}
\label{app:moved_results}

\paragraph{Tasks and variants.} The following detail was moved from the main text. We use three tasks, one per regime. \textsc{Addition} is symbolic: base-$b$ digits with carry propagation, evaluated in distribution (ID) and out of distribution (OOD) over sequence length, a standard probe of algorithmic length generalization~\citep{lee2023teaching,kazemnejad2023impact,mcleish2024transformers,anil2022exploring}. \textsc{first-last-XOR} is temporal credit assignment: the label depends on the first and last bits, so information is carried across the horizon $T$ and credit assigned to distant timesteps. \textsc{MNIST-Seq} \citep{lecun2010mnist} is visual: pixel input streamed as a sequence. Our parameter reconstruction method solves the same convex program over a fixed dictionary; the variants differ only in how the witnesses that build the dictionary are generated. \textsc{CVX} samples witnesses from a Gaussian distribution. \textsc{SG-CVX} takes witnesses from an SNN pretrained with surrogate gradients on the task. \textsc{R-CVX} takes witnesses from a criticality-tuned reservoir. The reservoir has no lateral recurrent weights. Its only recurrence is the LIF leak. All three feed the same convex spike reconstruction. We compare against two baselines that are not our method: \textsc{SG}, standard surrogate-gradient training, and \textsc{R}, a criticality-tuned reservoir with a ridge readout on the analog membrane, the classical reservoir-computing readout. \textsc{SG-SG} continues surrogate-gradient training from an \textsc{SG} checkpoint, and \textsc{R-CVX-an} applies the convex reconstruction to the analog membrane of the reservoir, isolating the spike readout. For the support experiment of Table~\ref{tab:support}, after $15$M draws each dictionary still produced new patterns (about $10^6$ distinct ones per instance), so each solution is optimal over the enumerated patterns and is not certified over the complete set, and rebuilding the network from the support reproduces the convex optimum with a primal--dual gap below $3\cdot10^{-8}$. In Table~\ref{tab:sampled_width} the subnetworks are certified later than the single layer, since their kernel is worse conditioned at large $n$ ($C_0=67$ against $42$ at $n=600$) although their sampled optimum is the same.

\begin{table*}[t]
  \centering
  \small
  \caption{Experiment results at depth $L=10$ (\textsc{first-last-XOR}: $L=15$). Columns and protocol as in Table~\ref{tab:main}.}
  \label{tab:main_L10}
  \setlength{\tabcolsep}{8pt}
  \renewcommand{\arraystretch}{1.15}
  \begin{tabular}{@{}llccc@{}}
    \toprule
    \textbf{Task} & \textbf{Timestep} & \textbf{Best CVX-variant} & \textbf{SG} & \textbf{R} \\
    \midrule
    \textsc{Addition}\textsubscript{Base-2}  & $T=6$ & 0.500 (SG-CVX) & 0.103 & \textbf{0.563} (R) \\
    \textsc{Addition}\textsubscript{Base-3}  & $T=6$ & 0.269 (SG-CVX)          & 0.239 & \textbf{0.342} (R) \\
    \textsc{Addition}\textsubscript{Base-5}  & $T=6$ & 0.235 (SG-CVX)          & 0.189 & \textbf{0.314} (R) \\
    \textsc{first-last-XOR}          & $T=6$  & \textbf{1.000$\pm$.000} (R-CVX) & 0.504$\pm$.002 & \textbf{1.000$\pm$.000} \\
    \textsc{first-last-XOR}          & $T=14$ & 0.508$\pm$.010 (R-CVX)    & 0.497$\pm$.003 & \textbf{0.509$\pm$.004} \\
    \textsc{MNIST-Seq}               & $T=4$  & \textbf{0.184$\pm$.006} (CVX-an) & 0.097$\pm$.006 & 0.161$\pm$.006 \\
    \bottomrule
  \end{tabular}
\end{table*}

\textbf{Random sampling finds it inefficiently.} Finding the supporting subnetworks by drawing witnesses at random is hard, even though there are few of them. We draw witnesses from three distributions (Gaussian; a wide-scale variant; and a near-critical variant with high leak) and solve over the drawn patterns, sweeping the number drawn $K$ from $1$ to $72$, three times the bound $n+1$. Table~\ref{tab:sampling} reports the gap to the optimum. At $K=2$, the support size on addition, random draws are far from the optimum on every task. The gap falls with $K$ but does not close: at $K=72$ the residual is still $0.03$ to $0.11$, and no distribution is consistently better than another. Random witness generation is the bottleneck, not the convex solve.

\begin{table}[t]
\centering
\small
\caption{Gap to the optimum ($p_K-p^\star$, mean $\pm$ std over three instances) when $K$ witnesses are drawn at random from three sampling distributions, per task. The gap does not close even at $K=72=3(n+1)$, and no distribution wins consistently.}
\label{tab:sampling}
\setlength{\tabcolsep}{4pt}
\renewcommand{\arraystretch}{1.15}
\begin{tabular}{@{}llcccccc@{}}
\toprule
Task & Dist. & $K{=}1$ & $K{=}2$ & $K{=}10$ & $K{=}24$ & $K{=}48$ & $K{=}72$ \\
\midrule
\multirow{3}{*}{\textsc{Addition}}
 & Gaussian & 0.872$\pm$.005 & 0.824$\pm$.038 & 0.448$\pm$.086 & 0.148$\pm$.024 & 0.064$\pm$.005 & 0.046$\pm$.008 \\
 & Wide     & 0.908$\pm$.023 & 0.872$\pm$.017 & 0.412$\pm$.051 & 0.092$\pm$.024 & 0.040$\pm$.010 & 0.032$\pm$.008 \\
 & Critical & 0.870$\pm$.022 & 0.819$\pm$.034 & 0.376$\pm$.055 & 0.089$\pm$.018 & 0.045$\pm$.006 & 0.030$\pm$.003 \\
\midrule
\multirow{3}{*}{\textsc{first-last-XOR}}
 & Gaussian & 0.940$\pm$.031 & 0.905$\pm$.027 & 0.552$\pm$.036 & 0.230$\pm$.029 & 0.145$\pm$.021 & 0.114$\pm$.021 \\
 & Wide     & 0.903$\pm$.040 & 0.880$\pm$.038 & 0.529$\pm$.027 & 0.178$\pm$.036 & 0.083$\pm$.006 & 0.067$\pm$.004 \\
 & Critical & 0.902$\pm$.024 & 0.864$\pm$.036 & 0.480$\pm$.011 & 0.142$\pm$.026 & 0.078$\pm$.007 & 0.059$\pm$.004 \\
\midrule
\multirow{3}{*}{\textsc{MNIST-Seq}}
 & Gaussian & 0.916$\pm$.014 & 0.885$\pm$.040 & 0.532$\pm$.088 & 0.241$\pm$.044 & 0.120$\pm$.021 & 0.086$\pm$.021 \\
 & Wide     & 0.880$\pm$.059 & 0.861$\pm$.048 & 0.509$\pm$.047 & 0.189$\pm$.034 & 0.094$\pm$.012 & 0.068$\pm$.006 \\
 & Critical & 0.909$\pm$.056 & 0.883$\pm$.050 & 0.512$\pm$.052 & 0.198$\pm$.044 & 0.095$\pm$.013 & 0.074$\pm$.011 \\
\bottomrule
\end{tabular}
\end{table}

\textbf{The convex solve is exact.} Table~\ref{tab:duality_gap} confirms the finite convex reconstruction is solved to optimality: across the number of subnetworks, the primal--dual gap is consistently below $10^{-6}$ on \textsc{first-last-XOR}. This verifies the convex stage attains the optimum of the sampled dictionary; it is not a claim of global optimality over the complete dictionary, which is intractable to enumerate. The solve is exact regardless of which dictionary it is given, so at full scale the remaining question is which witness source produces a dictionary that generalizes.

\begin{table}[H]
  \centering
  \small
  \caption{%
    Empirical optimality of the convex reconstruction solve on
    \textsc{first-last-XOR} ($T=11$, $L=3$).
    We sweep number of subnetworks $K$ at fixed total width. The primal and dual values are rounded to four decimal places and the precise gaps are below $10^{-6}$, numerically verifying that the sampled convex program is solved to optimality.}
  \label{tab:duality_gap}
  \setlength{\tabcolsep}{6pt}
  \begin{tabular}{@{}ccrrr@{}}
    \toprule
    $K$ & Test Acc. (\%) &  Primal & Dual & Abs. gap \\
    \midrule
     2 & 74.71 &  0.9817 & 0.9817 & $7.01\!\times\!10^{-8}$ \\
    16 & 82.23 &  0.8370 & 0.8370 & $2.94\!\times\!10^{-7}$ \\
    32 & 76.07 &  0.9459 & 0.9459 & $9.55\!\times\!10^{-8}$ \\
    64 & 86.43 &  0.8734 & 0.8734 & $2.65\!\times\!10^{-7}$ \\
    \bottomrule
  \end{tabular}
\end{table}

\textbf{Length extrapolation.} The OOD results on \textsc{Addition} probe the symbolic regime under length extrapolation (Table~\ref{tab:ood_methods}). The reservoir baseline \textsc{R} is strongest in distribution but degrades on longer sequences, while \textsc{R-CVX}, the convex readout over the same reservoir witnesses, is more stable: on base $2$ it improves from $0.507$ at $n=5$ to $0.551$ at ten times the training length, overtaking \textsc{R}, which falls from $0.904$ to $0.408$. The reconstruction over frozen reservoir witnesses extrapolates where the ridge readout on the same witnesses does not.

\begin{table}[t]
\centering
\small
\caption{Token-wise accuracy on arithmetic addition across training variants in OOD lengths.}
\label{tab:ood_methods}
\begin{tabular}{lllcccc}
\toprule
Base & Group & Method & ID ($n{=}5$) & OOD-$2\times$ & OOD-$5\times$ & OOD-$10\times$ \\
\midrule
\multirow{5}{*}{$b=2$}
  & \multirow{3}{*}{SG-variant}
    & SG     & 0.439 & 0.402 & 0.393 & 0.374 \\
  & & SG-SG  & 0.542 & 0.412 & 0.410 & 0.403 \\
  & & R      & 0.904 & \textbf{0.719} & \textbf{0.561} & 0.408 \\
  \cmidrule(lr){2-2}
  & \multirow{3}{*}{CVX-variant}
    & CVX    & 0.548 & 0.459 & 0.404 & 0.362 \\
  & & SG-CVX & \textbf{0.954} & 0.412 & 0.410 & 0.403 \\
  & & R-CVX  & 0.507 & 0.538 & 0.543 & \textbf{0.551} \\
\midrule
\multirow{5}{*}{$b=3$}
  & \multirow{3}{*}{SG-variant}
    & SG     & 0.217 & 0.197 & 0.182 & 0.177 \\
  & & SG-SG  & 0.238 & 0.220 & 0.207 & 0.202 \\
  & & R      & \textbf{0.407} & \textbf{0.306} & \textbf{0.233} & 0.187 \\
  \cmidrule(lr){2-2}
  & \multirow{3}{*}{CVX-variant}
    & CVX    & \textbf{0.366} & 0.302 & 0.254 & \textbf{0.235} \\
  & & SG-CVX & 0.340 & \textbf{0.314} & \textbf{0.307} & 0.310 \\
  & & R-CVX  & 0.163 & 0.124 & 0.103 & 0.089 \\
\midrule
\multirow{5}{*}{$b=5$}
  & \multirow{3}{*}{SG-variant}
    & SG     & 0.159 & 0.148 & 0.142 & 0.141 \\
  & & SG-SG  & 0.178 & 0.163 & 0.152 & 0.153 \\
  & & R      & \textbf{0.357} & \textbf{0.255} & 0.168 & 0.096 \\
  \cmidrule(lr){2-2}
  & \multirow{3}{*}{CVX-variant}
    & CVX    & 0.218 & 0.164 & 0.135 & 0.125 \\
  & & SG-CVX & 0.230 & 0.188 & 0.173 & 0.172 \\
  & & R-CVX  & \textbf{0.232} & \textbf{0.212} & \textbf{0.198} & \textbf{0.194} \\
\bottomrule
\end{tabular}
\end{table}

\subsection{Ablations on Model Depth, Timestep and Size of Training Set}
\label{sec:ablation_depth_timestep}
Tables~\ref{tab:abl_xor} and~\ref{tab:mnist_ablation} examine how model depth $L$ and timestep horizon $T$ each affect performance of \textsc{CVX}, \textsc{SG}, and the reservoir baseline \textsc{R}. We use \textsc{first-last-XOR} as a controlled benchmark, since its fixed binary input distribution lets us vary $L$ or $T$ in isolation. \textsc{MNIST-Seq} is the visual counterpart, where each image is streamed as sequential pixel patches so that increasing $T$ spreads the same information across more recurrent steps.

\textbf{Depth scaling.}
As shown on the left in Tables~\ref{tab:abl_xor} and~\ref{tab:mnist_ablation}, \textsc{SG} degrades sharply with depth, falling to chance at $L=10$ on both tasks. On \textsc{first-last-XOR} the reservoir and its convex readout stay at ceiling across all depths, so the convex reconstruction matches the reservoir rather than exceeding it in distribution. On \textsc{MNIST-Seq} the convex family degrades more slowly than \textsc{SG} and stays above it at the larger depths.

\textbf{Timestep scaling.}
Reported on the right, the best CVX-variant stays above \textsc{SG} at every evaluated $T$ on both tasks. On \textsc{first-last-XOR} it also stays above \textsc{R} at the longer horizons, where \textsc{R} falls toward chance as $T$ grows.

\begin{table}[t]
\centering
\footnotesize
\caption{
\textsc{first-last-XOR} ablations. Left: depth ablation at $T=6$. Right: timestep ablation at $L=3$.
Test accuracy (mean $\pm$ std, 3 seeds); best CVX-variant with the winner in brackets, \textsc{R} the reservoir ridge baseline.
}
\label{tab:abl_xor}
\setlength{\tabcolsep}{3pt}
\renewcommand{\arraystretch}{1.1}
\vspace{0.5em}
\begin{minipage}[t]{0.49\textwidth}
\centering
\textbf{(a) Depth, $T=6$}
\vspace{0.4em}
\begin{tabular}{lccc}
\toprule
$L$ & Best CVX & SG & R \\
\midrule
$3$  & \textbf{1.000$\pm$.000} (CVX)   & 0.758$\pm$.040 & 0.989$\pm$.019 \\
$5$  & \textbf{1.000$\pm$.000} (R-CVX) & 0.737$\pm$.022 & \textbf{1.000$\pm$.000} \\
$10$ & \textbf{1.000$\pm$.000} (R-CVX) & 0.521$\pm$.032 & \textbf{1.000$\pm$.000} \\
$15$ & \textbf{1.000$\pm$.000} (R-CVX) & 0.504$\pm$.002 & \textbf{1.000$\pm$.000} \\
\bottomrule
\end{tabular}
\end{minipage}
\hfill
\begin{minipage}[t]{0.49\textwidth}
\centering
\textbf{(b) Timestep, $L=3$}
\vspace{0.4em}
\begin{tabular}{lccc}
\toprule
$T$ & Best CVX & SG & R \\
\midrule
$6$  & \textbf{1.000$\pm$.000} (CVX) & 0.758$\pm$.040 & 0.989$\pm$.019 \\
$8$  & \textbf{1.000$\pm$.000} (CVX) & 0.698$\pm$.015 & 0.996$\pm$.008 \\
$11$ & \textbf{0.835$\pm$.011} (CVX) & 0.656$\pm$.004 & 0.651$\pm$.131 \\
$14$ & \textbf{0.704$\pm$.004} (CVX) & 0.537$\pm$.062 & 0.641$\pm$.024 \\
\bottomrule
\end{tabular}
\end{minipage}
\end{table}

\begin{table}[t]
\centering
\footnotesize
\caption{
\textsc{MNIST-Seq} ablations. Left: depth ablation at $T=4$. Right: timestep ablation at $L=3$.
Test accuracy (mean $\pm$ std, 3 seeds); best CVX-variant with the winner in brackets, \textsc{R} the reservoir ridge baseline. S-CVXa\,$=$\,\textsc{SG-CVX-an}, CVXa\,$=$\,\textsc{CVX-an}.
}
\label{tab:mnist_ablation}
\setlength{\tabcolsep}{3pt}
\renewcommand{\arraystretch}{1.1}
\vspace{0.5em}
\begin{minipage}[t]{0.49\textwidth}
\centering
\textbf{(a) Depth, $T=4$}
\vspace{0.4em}
\begin{tabular}{lccc}
\toprule
$L$ & Best CVX & SG & R \\
\midrule
$3$  & 0.851$\pm$.013 (S-CVXa)          & \textbf{0.852$\pm$.015} & 0.508$\pm$.053 \\
$5$  & \textbf{0.876$\pm$.004} (S-CVXa) & 0.873$\pm$.005          & 0.564$\pm$.003 \\
$10$ & \textbf{0.184$\pm$.006} (CVXa)   & 0.097$\pm$.006          & 0.161$\pm$.006 \\
\bottomrule
\end{tabular}
\end{minipage}\hfill
\begin{minipage}[t]{0.49\textwidth}
\centering
\textbf{(b) Timestep, $L=3$}
\vspace{0.4em}
\begin{tabular}{lccc}
\toprule
$T$ & Best CVX & SG & R \\
\midrule
$14$ & \textbf{0.808$\pm$.014} (S-CVXa) & 0.806$\pm$.015 & 0.585$\pm$.033 \\
$28$ & \textbf{0.830$\pm$.002} (CVXa)   & 0.822$\pm$.006 & 0.578$\pm$.009 \\
$56$ & \textbf{0.849$\pm$.003} (CVXa)   & 0.803$\pm$.006 & 0.599$\pm$.014 \\
\bottomrule
\end{tabular}
\end{minipage}
\end{table}

\textbf{Size of training set.} SNNs trained with CVX generalize better as the number of training samples increases, whereas SG shows no such benefit (Table~\ref{tab:abl_ntrain}). CVX token-wise accuracy rises from $0.907$ at $n_{\mathrm{train}}=512$ to $0.954$ once $n_{\mathrm{train}}\ge 4096$ and then saturates. Since the convex stage is solved to optimality over the sampled dictionary, the saturation below perfect accuracy reflects the sampled witness family, not the optimization; closing it is the witness-generation problem identified in \S\ref{sec:practical_impl}.

\begin{table}[H]
\centering
\small
\caption{Test accuracy on base-$2$ arithmetic addition as
numbers of training samples vary on models with $L{=}3$, $T{=}6$. Both
sequential accuracy (correct only if every output token in an addition sequence is correct)
and token-wise accuracy are reported. Darker numbers indicate better performance within each column.}
\label{tab:abl_ntrain}
\setlength{\tabcolsep}{4pt}
\begin{tabular}{ccccc}
\toprule
\# train & CVX (sequential) & CVX (token-wise) & SG (sequential) & SG (token-wise) \\
\midrule
512  & \textcolor{cvxs1}{0.571} & \textcolor{cvxt1}{0.907} & 0.160 & 0.737 \\
2304 & \textcolor{cvxs2}{0.717} & \textcolor{cvxt2}{0.951} & 0.183 & 0.744 \\
4096 & \textcolor{cvxs3}{0.730} & \textcolor{cvxt3}{0.954} & 0.163 & 0.738 \\
5888 & \textcolor{cvxs4}{0.725} & \textcolor{cvxt4}{0.954} & 0.184 & 0.738 \\
7680 & \textcolor{cvxs5}{0.728} & \textcolor{cvxt5}{0.954} & 0.152 & 0.738 \\
\bottomrule
\end{tabular}
\end{table}

\subsection{Loss Landscape}
\label{sec:landscape}
We examine the local geometry of the point each method converges to, on a medium-scale setting ($L=3$, $T=4$, per-branch width $64$). All curvature is on the training loss along filter-normalized directions~\citep{li_visualizing_2018}, averaged over $25$ direction triads (three perpendicular directions each). Full settings, per-$K$ tables, and interpolation results are in Appendix~\ref{app:k_landscape}.

\textbf{The convex readout is near-stationary; the alternatives are not.}
Table~\ref{tab:method_grad} reports the gradient and curvature at the trained centre for four methods. The two convex readouts, over Gaussian features (\textsc{Gauss-cone}) and reservoir features (\textsc{LSM-cone}), are near-stationary at a full fit: the gradient norms sit at $10^{-2}$ to $10^{-3}$ and there is no negative-curvature direction. The two non-convex readouts do not reach this point. Plain surrogate gradients sit at a saddle ($\mathrm{negc}=1$) and do not fit the spike readout at all. The reservoir ridge readout fits the labels but keeps a large readout gradient: on the same features as \textsc{LSM-cone}, the ridge readout does not reach a stationary point while the convex readout does. The convex solve brings the readout to a stationary optimum, independent of how the features are generated.

\begin{table}[t]
\centering
\small
\caption{Gradient and curvature at the trained centre, per method (per-branch width $64$, $K=50$, spike readout unless noted). The convex readouts (\textsc{Gauss-cone}, \textsc{LSM-cone}) are near-stationary at a full fit. \textsc{SG-plain} is a saddle and does not fit; \textsc{LSM-ridge} fits but keeps a large readout gradient.}
\label{tab:method_grad}
\setlength{\tabcolsep}{6pt}
\renewcommand{\arraystretch}{1.15}
\begin{tabular}{@{}lccccc@{}}
\toprule
Method & $\|\nabla_{\mathrm{int}}\|$ & $\|\nabla_W\|$ & negc & CE & acc \\
\midrule
\textsc{SG-plain}   & 0.481 & 0.106 & 1.00 & 2.272 & 0.18 \\
\textsc{Gauss-cone} & 0.003 & 0.012 & 0.00 & 0.004 & 1.00 \\
\textsc{LSM-cone}   & 0.012 & 0.009 & 0.00 & 0.005 & 1.00 \\
\textsc{LSM-ridge}  & 0.134 & 1.183 & 0.00 & 0.855 & 1.00 \\
\bottomrule
\end{tabular}
\end{table}

\textbf{The convex readout rescues the surrogate-gradient saddle.}
Table~\ref{tab:readout_rescue} holds the surrogate-gradient features fixed and swaps only the last layer. The surrogate-gradient readout sits at a saddle of the training loss ($\lambda_{\min}<0$, $\mathrm{negc}\ge 0.8$); the convex cone readout over the same features has no negative-curvature direction ($\lambda_{\min}>0$, $\mathrm{negc}=0$) and lowers the CE by an order of magnitude. The readout path from one to the other, at frozen features, is monotone down (Figure~\ref{fig:rescue}). This is the geometric form of the training-accuracy gap in Appendix~\ref{app:train_tables}.

\section{Arithmetic Addition Experimental Details}
\label{app:addition_repro}

For two $n$-digit base-$b$ operands, the input space grows exponentially as $b^{2n}$, so long-sequence OOD performance cannot be explained by memorizing short training examples.

\paragraph{Compute Resources} All pure convex reconstruction runs were trained primarily on a local MacBook M4 Pro with 24 GB of RAM. Surrogate-gradient training and finetuning runs were performed on A100 GPU partitions, including A100x4 and A100x8 configurations.

\paragraph{Why addition and autoregressive rollout?}
We use arithmetic addition as a controlled probe of temporal algorithmic generalization. Prior work on sequence models and Transformers uses arithmetic in the same spirit: models can fit short training lengths while failing to extrapolate to longer digit sequences, and successful length generalization often requires architectural or supervision choices that expose the task structure \citep{jelassi2023length,cho2024position,cho2024arithmetic}. Addition is therefore useful not because the input-output map is intrinsically difficult, but because correct extrapolation requires a reusable carry-transition rule.

Our main protocol uses a carry-augmented recurrent formulation. At each timestep, the model predicts both the sum digit and the next carry. This teacher-forced supervision gives direct access to the latent state that mediates the recurrent computation. At evaluation time, we use autoregressive rollout, where the model must reuse its own predicted carry over longer horizons. Thus, the benchmark separates short-horizon fitting from stable recurrent extrapolation. This is also the setting most aligned with our spiked dictionary theory: a useful LIF feature is not merely a static output pattern, but a recurrent trajectory generated by hidden membrane and spike dynamics.

We also evaluate a direct operand-to-output variant without explicit latent carry supervision. That setting serves as a protocol ablation: it asks whether convex reconstruction improves local digit prediction even when the model is not forced to learn an interpretable carry state. In our results, the direct variant improves token-level prediction and delays first errors, but does not replace the recurrent rollout benchmark as the main test of length generalization.

\paragraph{Task and splits.}
We train on addition in bases $b\in\{2,3,5\}$ with $n_{\mathrm{digits}}=5$, so the training horizon is $T=n_{\mathrm{digits}}+1=6$. Each timestep input is the normalized tuple $(a_t,b_t,c_t^{\mathrm{in}})$, where $a_t$ and $b_t$ are the operand digits and $c_t^{\mathrm{in}}$ is the carry into the current digit column. The model predicts the sum digit $y_t^{\mathrm{sum}}$ and the next carry $y_t^{\mathrm{carry}}$ at every timestep. All bases are averaged over seeds $\{0,1,2\}$. Each run uses separate pretraining and finetuning splits: $n_{\mathrm{train,pre}}=2304$, $n_{\mathrm{val,pre}}=512$, $n_{\mathrm{train,ft}}=2304$, and $n_{\mathrm{val,ft}}=512$. The held-out ID test set has $n_{\mathrm{test}}=1024$, and each OOD split has $n_{\mathrm{test,ood}}=1024$. For base $2$, we evaluate $n_{\mathrm{digits}}\in\{10,20,50\}$; for bases $3$ and $5$, we evaluate $n_{\mathrm{digits}}\in\{10,25,50\}$.

\paragraph{Architecture and training variants.}
All addition runs use a $K$-parallel LIF-SNN with
\[
L=3,\qquad P_{\mathrm{rec}}=256,\qquad P_{\mathrm{last}}=512,\qquad K=2.
\]
Both SG and CVX use final-layer spike readout, so the convex dictionary is built from binary spike features. This matches the witnessed spike-dictionary formulation in the theory. For bases $b>2$, the hidden dictionary remains binary, while the output head is multiclass and maps spike features to one of $b$ sum digits. The five training variants are exactly those defined in Appendix~\ref{app:training_algorithms}: \textsc{SG}, \textsc{SG-CVX}, \textsc{SG-SG}, \textsc{CVX}, and \textsc{CVX-SG}.

\paragraph{Losses and hyperparameter sweeps.}
The carry-augmented objective is
\[
\mathcal L_{\mathrm{joint}}
=
\lambda_{\mathrm{sum}}\mathcal L_{\mathrm{sum}}
+
\lambda_{\mathrm{carry}}\mathcal L_{\mathrm{carry}},
\]
with $\lambda_{\mathrm{sum}}=1$. We sweep
\[
\lambda_{\mathrm{carry}}
\in
\{0.125,0.25,0.5,0.75,1.0,1.25,1.5,2.0,4.0,6.0,8.0,10.0\}.
\]
The surrogate-gradient stages use Adam for $100$ pretraining epochs and $100$ finetuning epochs. The SG learning-rate grid is
\[
\{10^{-3},5\cdot 10^{-3},10^{-2},10^{-1}\},
\]
and the SG and CVX regularization grids are both
\[
\{10^{-2},10^{-1},0.5,1.0,5.0,10.0\}.
\]
Because spike readout is used, the CVX bias grid is fixed to $\{0.0\}$. Both SG and CVX use ramped timestep weighting and the joint teacher-forcing objective during training.

\paragraph{Evaluation metrics.}
Every trained stage is evaluated under both teacher forcing and autoregressive rollout. Unless otherwise stated, the main paper reports autoregressive evaluation. The primary metric is joint-token accuracy:
\[
\frac{1}{nT}\sum_{i=1}^n\sum_{t=1}^T
\mathbf 1\{\hat y_{i,t}^{\mathrm{sum}}=y_{i,t}^{\mathrm{sum}}\}
\mathbf 1\{\hat y_{i,t}^{\mathrm{carry}}=y_{i,t}^{\mathrm{carry}}\}.
\]
A timestep is counted as correct only when both the sum digit and carry prediction are correct. We also track sum-token accuracy, carry-token accuracy, joint-sequence accuracy, and first-error timestep statistics.

\paragraph{Carry-loss sensitivity.}
The appendix figures report autoregressive joint-token accuracy as a function of $\lambda_{\mathrm{carry}}$ for ID and OOD splits. These plots show the tradeoff between enforcing the latent carry variable and maintaining stable rollout behavior at longer horizons. The main tables report validation-selected configurations aggregated over seeds.

\begin{figure*}[t]
    \centering
    \includegraphics[width=\textwidth]{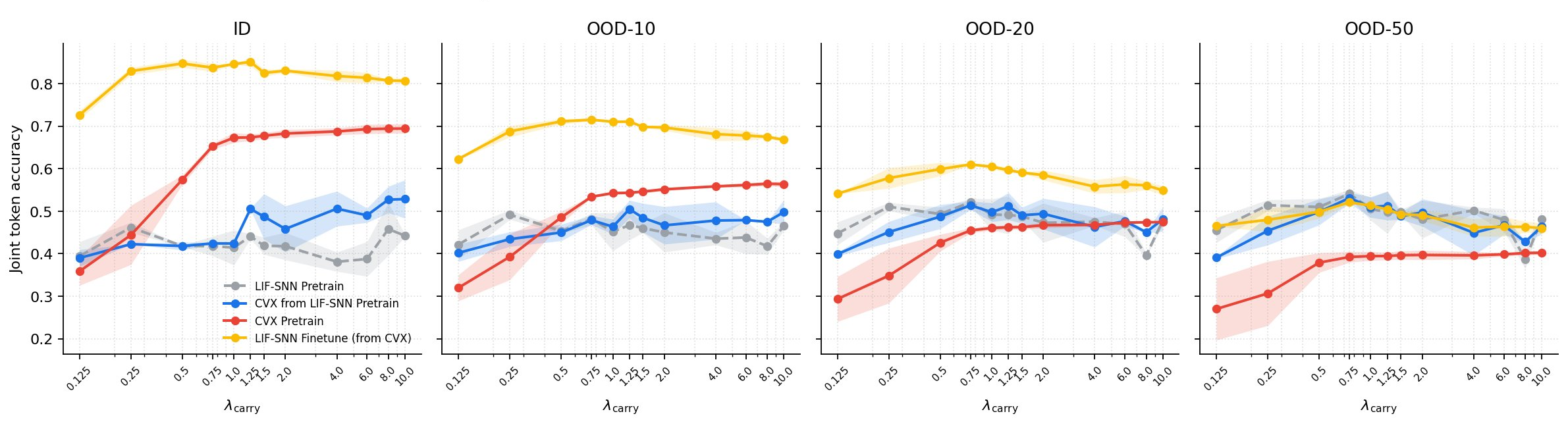}
    \caption{
    Base-$2$ addition: effect of $\lambda_{\mathrm{carry}}$ on autoregressive joint-token accuracy for ID and OOD splits. Results are averaged over three seeds. The architecture is $L=3$, $P_{\mathrm{rec}}=256$, $P_{\mathrm{last}}=512$, and $K=2$, with final-layer spike readout for both SG and CVX. OOD lengths are $n_{\mathrm{digits}}\in\{10,20,50\}$.
    }
    \label{fig:addition_lambda_base2}
\end{figure*}

\begin{figure*}[t]
    \centering
    \includegraphics[width=\textwidth]{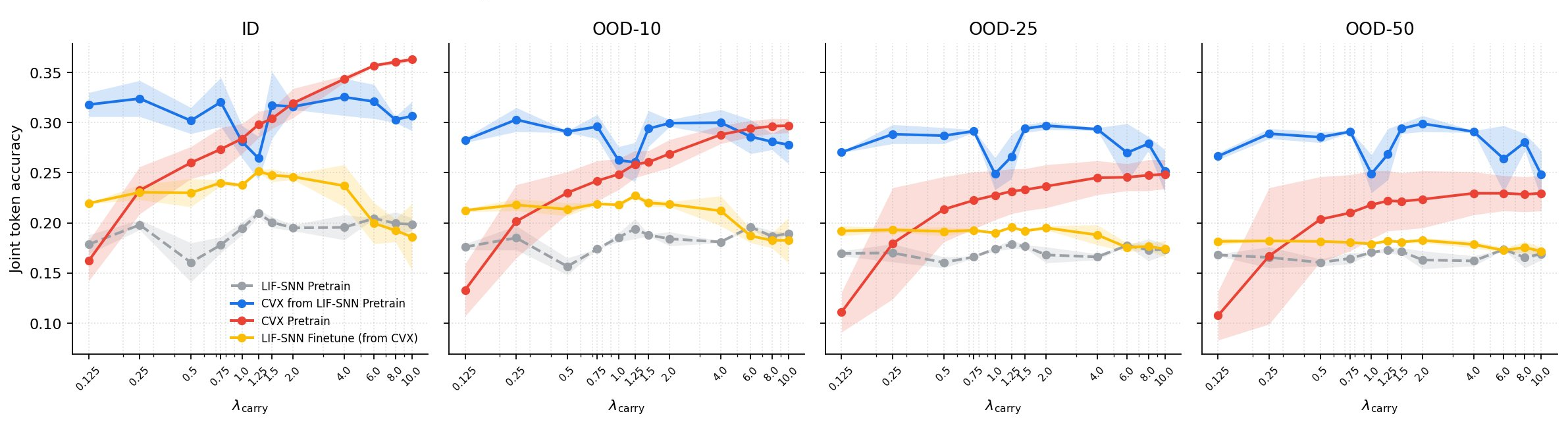}
    \caption{
    Base-$3$ addition: effect of $\lambda_{\mathrm{carry}}$ on autoregressive joint-token accuracy for ID and OOD splits. Results are averaged over three seeds. The architecture and readout match Figure~\ref{fig:addition_lambda_base2}; OOD lengths are $n_{\mathrm{digits}}\in\{10,25,50\}$.
    }
    \label{fig:addition_lambda_base3}
\end{figure*}

\begin{figure*}[t]
    \centering
    \includegraphics[width=\textwidth]{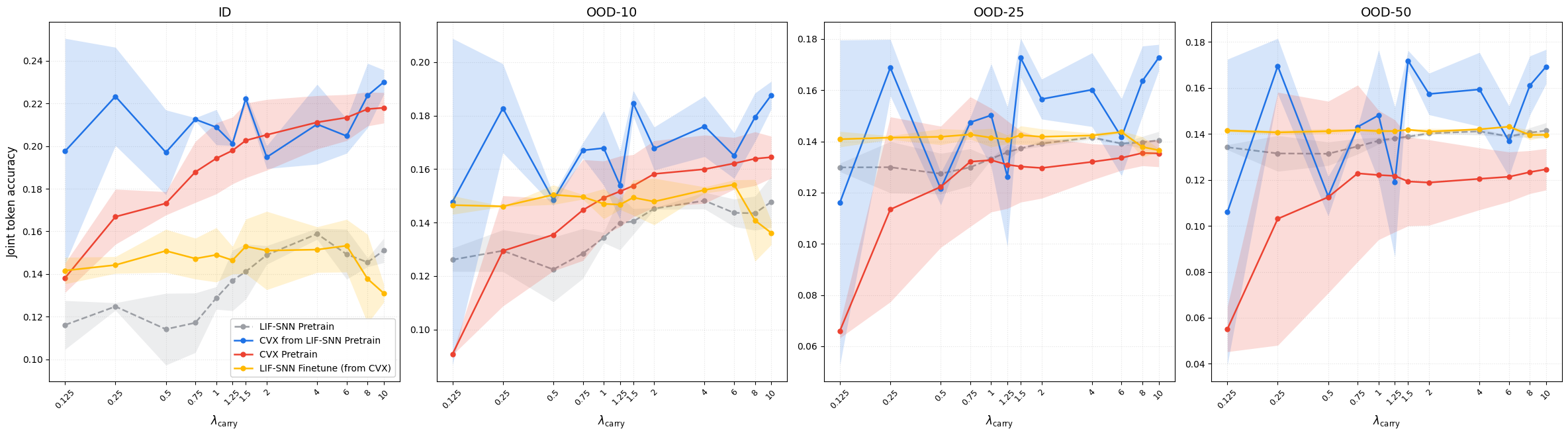}
    \caption{
    Base-$5$ addition: effect of $\lambda_{\mathrm{carry}}$ on autoregressive joint-token accuracy for ID and OOD splits. Results are averaged over three seeds. The architecture is the same as the base-$2$ and base-$3$ experiments: $L=3$, $P_{\mathrm{rec}}=256$, $P_{\mathrm{last}}=512$, and $K=2$, with final-layer spike readout for both SG and CVX. OOD lengths are $n_{\mathrm{digits}}\in\{10,25,50\}$.
    }
    \label{fig:addition_lambda_base5}
\end{figure*}

\paragraph{Aggregation.}
The main tables aggregate results over seeds after validation selection over the corresponding hyperparameter grid and $\lambda_{\mathrm{carry}}$ sweep. The sensitivity plots in this appendix show the full $\lambda_{\mathrm{carry}}$ dependence before selecting the final reported operating point.

\section{Limitations and Broader Impact}
\label{app:limitations}

\paragraph{Limitations.} The sampling guarantee covers shallow networks, namely two-layer threshold and LIF networks, criticality-tuned reservoirs without lateral weights, and parallel LIF subnetworks with a one-neuron readout layer. At the measured sizes, $C_0$ is $1.6$--$1.8\sqrt n$, and the certified width exceeds the width at which concentration is observed by a factor of $60$ to $140$. Deeper networks, lateral and shared recurrent weights, the $\sqrt n$ factor from the $\ell_2$ step, a lower bound on $\lambda_{\min}$ for the subnetwork kernel, and pretrained witnesses remain open. Completeness of the LIF dictionary is established at the second hidden layer only. The implementation uses a general-purpose convex solver without accelerator support, and the evaluation does not include speech or large-scale vision tasks.

\paragraph{Broader impact.} This research may serve as a stepping stone toward training more energy-efficient and biologically plausible AI models at scale. Beyond SNNs, the reconstruction viewpoint is also naturally aligned with reservoir computing, where internal recurrent dynamics are fixed and only the final readout layer is trained, similar in spirit to our \textsc{CVX} and \textsc{SG-CVX} algorithms. More broadly, globally optimized readout reconstruction can provide a principled way to reuse fixed or pretrained dynamical systems without relying entirely on end-to-end gradient-based training. If developed further, this could reduce the computational and energy costs of training temporal models, while also offering a clearer diagnostic framework for understanding when failures arise from optimization, representation, or the choice of internal dynamics.

\end{document}